\documentclass[11pt,twoside]{article}

\usepackage{fullpage}

\RequirePackage{amsthm,amsmath}
\RequirePackage[colorlinks,citecolor=blue,urlcolor=blue]{hyperref}

\usepackage{epsf}
\usepackage{graphics}
\usepackage{subcaption}
\usepackage{psfrag}

\usepackage{color}

\usepackage{mathtools}
\usepackage{amsfonts}
\usepackage{amssymb}

\DeclareSymbolFont{extraup}{U}{zavm}{m}{n}
\DeclareMathSymbol{\varheart}{\mathalpha}{extraup}{86}
\DeclareMathSymbol{\vardiamond}{\mathalpha}{extraup}{87}

\numberwithin{equation}{section}

\theoremstyle{plain}
\newtheorem{theorem}{Theorem}[section]
\newtheorem{proposition}[theorem]{Proposition}
\newtheorem{lemma}[theorem]{Lemma}
\newtheorem{corollary}{Corollary}[section]
\newtheorem{definition}{Definition}[section]
\newtheorem{remark}{Remark}[section]

\long\def\comment#1{}

\newcommand{\defn}{:\,=}

\newcommand{\id}{\mathsf{id}}

\newcommand{\dimone}{n_1}
\newcommand{\dimtwo}{n_2}
\newcommand{\maxdimonetwo}{\left(\dimone \lor \dimtwo \right)}

\newcommand{\csst}{\mathbb{C}_{\mathsf{SST}}}

\newcommand{\Cbiso}{\mathbb{C}_{\mathsf{BISO}}}

\newcommand{\bl}{\mathsf{bl}}

\newcommand{\blref}{\mathsf{bl_{ref}}}

\newcommand{\Pzero}{\mathbb{P}_{0}}
\newcommand{\Pone}{\mathbb{P}_{1}}
\newcommand{\Pzerotil}{\widetilde{\mathbb{P}}_{0}}
\newcommand{\Ponetil}{\widetilde{\mathbb{P}}_{1}}

\newcommand{\taulem}{\chi}

\newcommand{\set}{\ensuremath{\mathcal{S}}}

\newcommand{\minradius}{\ensuremath{\mathbf{r}}}

\newcommand{\E}{\ensuremath{{\mathbb{E}}}}

\newcommand{\1}{\ensuremath{{\sf (i)}}}
\newcommand{\2}{\ensuremath{{\sf (ii)}}}
\newcommand{\3}{\ensuremath{{\sf (iii)}}}

\DeclareMathOperator{\modd}{mod}

\newcommand{\NORMAL}{\ensuremath{\mathcal{N}}}
\newcommand{\BER}{\ensuremath{\mbox{\sf Ber}}}
\newcommand{\BIN}{\ensuremath{\mbox{\sf Bin}}}

\newcommand{\thetastar}{\ensuremath{\theta^*}}
\newcommand{\thetahat}{\ensuremath{\widehat{\theta}}}

\newcommand{\Mhat}{\ensuremath{\widehat{M}}}
\newcommand{\Mstar}{\ensuremath{M^*}}
\newcommand{\Mtilde}{\ensuremath{\widetilde{M}}}

\newcommand{\CRL}{\ensuremath{\mathsf{CRL}}}

\newcommand{\Mhatls}{\ensuremath{\widehat{M}}_{{\sf LS}}}

\newcommand{\real}{\ensuremath{\mathbb{R}}}

\newcommand{\pihat}{\ensuremath{\widehat{\pi}}}

\newcommand{\pihatref}{\ensuremath{\widehat{\pi}_{\sf ref}}}
\newcommand{\pihattds}{\ensuremath{\widehat{\pi}_{\sf tds}}}

\newcommand{\pihatftds}{\ensuremath{\widehat{\pi}_{{\sf tds}}}}
\newcommand{\sigmahat}{\ensuremath{\widehat{\sigma}}}

\newcommand{\sigmahattds}{\ensuremath{\widehat{\sigma}_{\sf tds}}}
\newcommand{\sigmahatftds}{\ensuremath{\widehat{\sigma}_{\sf tds}}}
\newcommand{\sigmahatpre}{\ensuremath{\widehat{\sigma}_{\sf pre}}}

\newcommand{\EE}{\ensuremath{\mathbb{E}}}

\newcommand{\symgp}{\mathfrak{S}}
\newcommand{\vars}{\zeta}
\newcommand{\maxvarone}{(\vars^2 \lor 1)}
\newcommand{\varplusone}{(\vars + 1)}

\newcommand{\bfone}{\mathbf{1}}
\DeclareMathOperator{\trace}{\mathsf{tr}}
\DeclareMathOperator{\Poi}{\mathsf{Poi}}
\DeclareMathOperator{\var}{\mathsf{var}}

\newcommand{\Cperm}{\mathbb{C}_{\mathsf{Perm}}^{\mathsf{r},\mathsf{c}}}
\newcommand{\Cpermr}{\mathbb{C}_{\mathsf{Perm}}^{\mathsf{r}}}

\newcommand{\Cskew}{\mathbb{C}_{\mathsf{skew}}}

\newcommand{\BL}{\mathsf{BL}}

\newcommand{\TV}{\mathsf{TV}}

\newcommand{\cE}{\mathcal{E}}

\newcommand{\bllarge}{\mathsf{BL}^{\mathbb{L}}}
\newcommand{\blsmall}{\mathsf{BL}^{\mathbb{S}}}
\newcommand{\blone}{\BL^{(1)}}
\newcommand{\bltwo}{\BL^{(2)}}
\newcommand{\blt}{\BL^{(t)}}

\newcommand{\blocksize}{\dimtwo \sqrt{\frac{\dimone}{N} \log (\dimone \dimtwo) }}

\newcommand{\order}{\ensuremath{\mathcal{O}}}
\newcommand{\ordertil}{\ensuremath{\widetilde{\order}}}

\newcommand{\Gdist}{\mathbb{G}}
\newcommand{\Bdist}{\mathbb{V}}

\newcommand{\cond}{\, | \,}

\newcommand{\cC}{\mathcal{C}}
\newcommand{\cB}{\mathcal{B}}
\newcommand{\cQ}{\mathcal{Q}}

\newcommand{\diam}{\operatorname*{\mathsf{diam}}}

\begin{document}

\begin{center}

{\bf{\LARGE{Towards Optimal Estimation of Bivariate Isotonic Matrices
      with Unknown Permutations}}}



\vspace*{.2in}

{\large{
\begin{tabular}{ccc}
Cheng Mao$^\star$ & Ashwin Pananjady$^\dagger$ & Martin
J. Wainwright$^{\dagger, \ddagger}$
\end{tabular}
}}
\vspace*{.2in}

\begin{tabular}{c}
Department of Statistics and Data Science, Yale University$^\star$ \\
Department of Electrical Engineering and Computer Sciences, UC Berkeley$^\dagger$ \\
Department of Statistics, UC Berkeley$^\ddagger$
\end{tabular}

\vspace*{.2in}

\today

\end{center}
\vspace*{.2in}

\begin{abstract}

 Many applications, including rank aggregation, crowd-labeling, and
  graphon estimation, can be modeled in terms of a bivariate isotonic
  matrix with unknown permutations acting on its rows and/or
  columns. We consider the problem of estimating an unknown matrix in
  this class, based on noisy observations of (possibly, a subset of)
  its entries. We design and analyze polynomial-time algorithms that
  improve upon the state of the art in two distinct metrics, showing,
  in particular, that minimax optimal, computationally efficient
  estimation is achievable in certain settings.
Along the way, we prove matching upper and lower bounds on the minimax
radii of certain cone testing problems, which may be of independent
interest. 
\end{abstract}


\section{Introduction}

Structured matrices with unknown permutations acting on their rows and
columns arise in multiple applications, including estimation from
pairwise comparisons~\cite{BraTer52,ShaBalGunWai17} and
crowd-labeling~\cite{DawSke79,ShaBalWai16}. Traditional parametric
models (e.g.,~\cite{BraTer52,Luc59,Thu27,DawSke79}) assume that these
matrices are obtained from rank-one or rank-two matrices via a known link
function. Aided by tools such as maximum likelihood estimation and
spectral methods, researchers have made significant progress in
studying both statistical and computational aspects of these
parametric
models~\cite{HajOhXu14,RajAga14,Shaetal16,NegOhSha16,ZhaCheDenJor16,
  GaoZho13, GaoLuZho16, KarOhSha11-b, LiuPenIhl12, DasDasRas13,
  GhoKalMcA11} and their low-rank
generalizations~\cite{RajAga16,NegOhTheXu17,
  KarOhSha11}.~\nocite{LeeSha17, CheSuh15, ParNeeZhaSanDhi15}

On the other hand, evidence from empirical studies suggests that
real-world data is not always well-described by such parametric
models~\cite{McLLuc65,BalWil97}. With the goal of increasing model
flexibility, a recent line of work has studied the class of
\emph{permutation-based} models~\cite{Cha15,
  ShaBalGunWai17,ShaBalWai16}.  Rather than imposing parametric
conditions on the matrix entries, these models impose only shape
constraints on the matrix, such as monotonicity, before unknown
permutations act on the rows and columns of the matrix.  On one hand, this more
flexible class reduces modeling bias compared to its parametric
counterparts while, perhaps surprisingly, producing models that can be
estimated at rates that differ only by logarithmic factors from the
classical parametric models. On the other hand, these advantages of
permutation-based models are accompanied by significant computational
challenges. The unknown permutations make the parameter space highly
non-convex, so that efficient maximum likelihood estimation is
unlikely. Moreover, spectral methods are often sub-optimal in
approximating shape-constrained sets of
matrices~\cite{Cha15,ShaBalGunWai17}. Consequently, results from many
recent papers show a non-trivial statistical-computational gap in
estimation rates for models with latent
permutations~\cite{ShaBalGunWai17,ChaMuk16,ShaBalWai16,FlaMaoRig16,PanWaiCou17}.

\paragraph{Related work}
While the primary motivation of our work comes from non-parametric
methods for aggregating pairwise comparisons, we begin by discussing a
few other lines of related work. The current paper lies at the
intersection of shape-constrained estimation and latent permutation
learning. Shape-constrained estimation has long been a major topic in
non-parametric statistics, and of particular relevance to our work is
the estimation of a bivariate isotonic matrix without latent
permutations~\cite{ChaGunSen18}. There, it was shown that the minimax
rate of estimating an $\dimone \times \dimtwo$ matrix from noisy
observations of all its entries is $\widetilde \Theta((\dimone
\dimtwo)^{-1/2})$. The upper bound is achieved by the least squares
estimator, which is efficiently computable due to the convexity of the
parameter space.

Shape-constrained matrices with permuted rows or columns also arise in
applications such as seriation~\cite{FogJenBacdAs13,FlaMaoRig16},
feature matching~\cite{ColDal16}, and graphon
estimation~\cite{BicCheLev11, ChaAir14, GaoLuZho15, KloTsyVer17}. In
particular, the monotone subclass of the statistical seriation
model~\cite{FlaMaoRig16} contains $n \times n$ matrices that have
increasing columns, and an unknown row permutation. Flammarion et
al.~\cite{FlaMaoRig16} established the minimax rate $\widetilde
\Theta(n^{-2/3})$ for estimating matrices in this class and proposed a
computationally efficient algorithm with rate $\ordertil
(n^{-1/2})$. For the subclass of such matrices where in addition, the
rows are also monotone, the results of the current paper improve the
two rates to $\widetilde \Theta (n^{-1})$ and $\ordertil (n^{-3/4})$
respectively.

Graphon estimation has seen its own extensive literature, and we only
list those papers that are most relevant to our setting. In essence,
these problems involve non-parametric estimation of a bivariate
function $f$ from noisy observations of $f(\xi_i, \xi_j)$ with the
design points $\{\xi_i\}_{i = 1}^n$ drawn i.i.d. from some
distribution supported on the interval $[0, 1]$. In contrast to
non-parametric estimation, however, the design points in graphon
estimation are unobserved, which gives rise to the underlying latent
permutation. Modeling the function $f$ as monotone recovers the model
studied in this paper, but other settings have been studied by various
authors: notably those where the function $f$ is
Lipschitz~\cite{ChaAir14}, block-wise constant~\cite{BicCheLev11,
  GaoLuZho15, KloTsyVer17} (also known as the stochastic block
model~\cite{Abb17}), or with $f$ satisfying other smoothness
assumptions~\cite{WolOlh13, GaoLuZho15, BorChaCohGan15}. There are
many interesting statistical-computational gaps also known to exist in
many of these problems.






Another related model in the pairwise comparison literature is that of
noisy sorting~\cite{BraMos08}, which involves a latent permutation
but no shape-constraint. In this prototype of a permutation-based
ranking model, we have an unknown, $n \times n$ matrix with constant
upper and lower triangular portions whose rows and columns are acted
upon by an unknown permutation. The hardness of recovering any such
matrix in noise lies in estimating the unknown permutation. As it
turns out, this class of matrices can be estimated efficiently at
minimax optimal rate $\widetilde \Theta(n^{-1})$ by multiple
procedures: the original work by Braverman and Mossel~\cite{BraMos08}
proposed an algorithm with time complexity $\order(n^c)$ for some
unknown and large constant $c$, and recently, an $\ordertil(n^2)$-time
algorithm was proposed by~Mao et al.~\cite{MaoWeeRig17}.  These
algorithms, however, do not generalize beyond the noisy sorting class,
which constitutes a small subclass of an interesting class of matrices
that we describe next.

The most relevant body of work to the current paper is that on
estimating matrices satisfying the \emph{strong stochastic
  transitivity} condition, or SST for short. This class of matrices
contains all $n \times n$ bivariate isotonic matrices with unknown
permutations acting on their rows and columns, with an additional
skew-symmetry constraint. The first theoretical study of these
matrices was carried out by Chatterjee~\cite{Cha15}, who showed that
a spectral algorithm achieved the rate $\ordertil(n^{-1/4})$ in the
normalized, squared Frobenius norm. Shah et al.~\cite{ShaBalGunWai17} then
showed that the minimax rate of estimation is given by $\widetilde
\Theta (n^{-1})$, and also improved the analysis of the spectral
estimator of Chatterjee to obtain the computationally
efficient rate $\ordertil (n^{-1/2})$. In follow-up
work~\cite{ShaBalWai16-2}, they also showed a second $\CRL$ estimator
based on the Borda count that achieved the same rate, but in
near-linear time. In related work, Chatterjee and
Mukherjee~\cite{ChaMuk16} analyzed a variant of the $\CRL$ estimator,
showing that for sub-classes of SST matrices, it achieved rates that
were faster than $\order(n^{-1/2})$. In a complementary direction, a
superset of the current authors~\cite{PanMaoMutWaiCou17} analyzed the
estimation problem under an observation model with structured missing
data, and showed that for many observation patterns, a variant of the
$\CRL$ estimator was minimax optimal.

Shah et al.~\cite{ShaBalWai16-2} also showed that conditioned on the
planted clique conjecture, it is impossible to improve upon a certain
notion of adaptivity of the $\CRL$ estimator in polynomial time.  Such
results have prompted various authors~\cite{FlaMaoRig16,
  ShaBalWai16-2} to conjecture that a similar
statistical-computational gap also exists when estimating SST matrices
in the Frobenius norm.

In our own preliminary work~\cite{mao18breakingcolt}, we announced
progress on the aforementioned statistical-computational gap. In
particular, we claimed a computationally efficient algorithm that for
$n \times n$ matrices, attained the rate $\ordertil(n^{-3/4})$ in the
squared Frobenius error in the full observation setting. This result
was stated in the extended abstract~\cite{mao18breakingcolt}, but not
proved there.  The current manuscript contains a superset of results
presented in the abstract~\cite{mao18breakingcolt}; in particular,
Theorem~\ref{thm:fast-tds} of the current manuscript significantly
extends Theorem 1 of the abstract~\cite{mao18breakingcolt}, and the
proof of~\cite[Theorem 1]{mao18breakingcolt} appears for the first
time in the current manuscript as a corollary of the proof of
Theorem~\ref{thm:fast-tds}.

It is also worth mentioning that estimation of such matrices has also
been studied in other metrics motivated by ranking applications.  One
such metric is the \emph{max-row-norm}, to be defined precisely in
Section~\ref{sec:err-met}; it has been studied in the paper
papers~\cite{ChaMuk16, ShaBalWai16-2} and quantifies a notion of
(matrix-weighted) distance between permutations. This metric has also
been used more recently to learn mixtures of such rankings, and a
characterization of the fundamental limits of estimation in this
metric has been acknowledged to be an important problem for ranking
applications~\cite{shah2018learning}.

\paragraph{Contributions}
In this paper, we study the problem of
estimating a bivariate isotonic matrix with unknown permutations
acting on its rows and columns, given noisy, (possibly) partial observations of its entries; this matrix class strictly contains the SST
model~\cite{Cha15,ShaBalGunWai17} for ranking from pairwise
comparisons.
We also study a sub-class of such matrices motivated by applications in crowd-labeling, which consists of bivariate isotonic matrices with one unknown permutation acting on its rows.

We begin by characterizing, in both the Frobenius and max-row-norm metrics,
the fundamental limits of estimation of both classes of matrices; the former result significantly generalizes those obtained by Shah et al.~\cite{ShaBalGunWai17}. In particular, our results hold for arbitrary matrix dimensions and sample sizes, and also extend results of Chatterjee, Guntuboyina and Sen~\cite{ChaGunSen18} for estimating the sub-class of bivariate isotonic matrices without unknown permutations. We then present computationally efficient algorithms for estimating both classes of matrices; these algorithms are novel in the sense that they are neither spectral in nature, nor simple variations of the Borda count estimator that was previously employed. 
They are also tractable in practice and show significant improvements over state-of-the-art estimators; Figure~\ref{fig:plot} presents such a comparison for our algorithm specialized to SST matrices with (roughly) one observation per entry.

These algorithms are analyzed in both the Frobenius and max-row-norm error metrics. In particular, we show that our algorithms attain minimax-optimal rates of estimation for the class $\Cperm$ in the max-row-norm, and our characterization additionally sheds light on the limitations of existing analyses in the literature. 
The rates attained by the algorithm in the Frobenius error match those announced in the abstract~\cite{mao18breakingcolt} for square matrices with partial observations. Additionally, we also show that these algorithms are minimax-optimal when the number of observations grows to be sufficiently large; notably, this stands in stark contrast to existing computationally efficient algorithms, which are not minimax-optimal in \emph{any} regime of the problem. Also notable is that the proof of our results provided in the supplementary material---which simultaneously covers both the partial observation and large sample settings---contains (as a corollary) the first proof of~\cite[Theorem 1]{mao18breakingcolt}.
\begin{figure}[ht]
\centering
\begin{minipage}[c]{.5\linewidth}
\includegraphics[clip, trim=11cm 6.5cm 6.2cm 6.3cm, width=\linewidth]{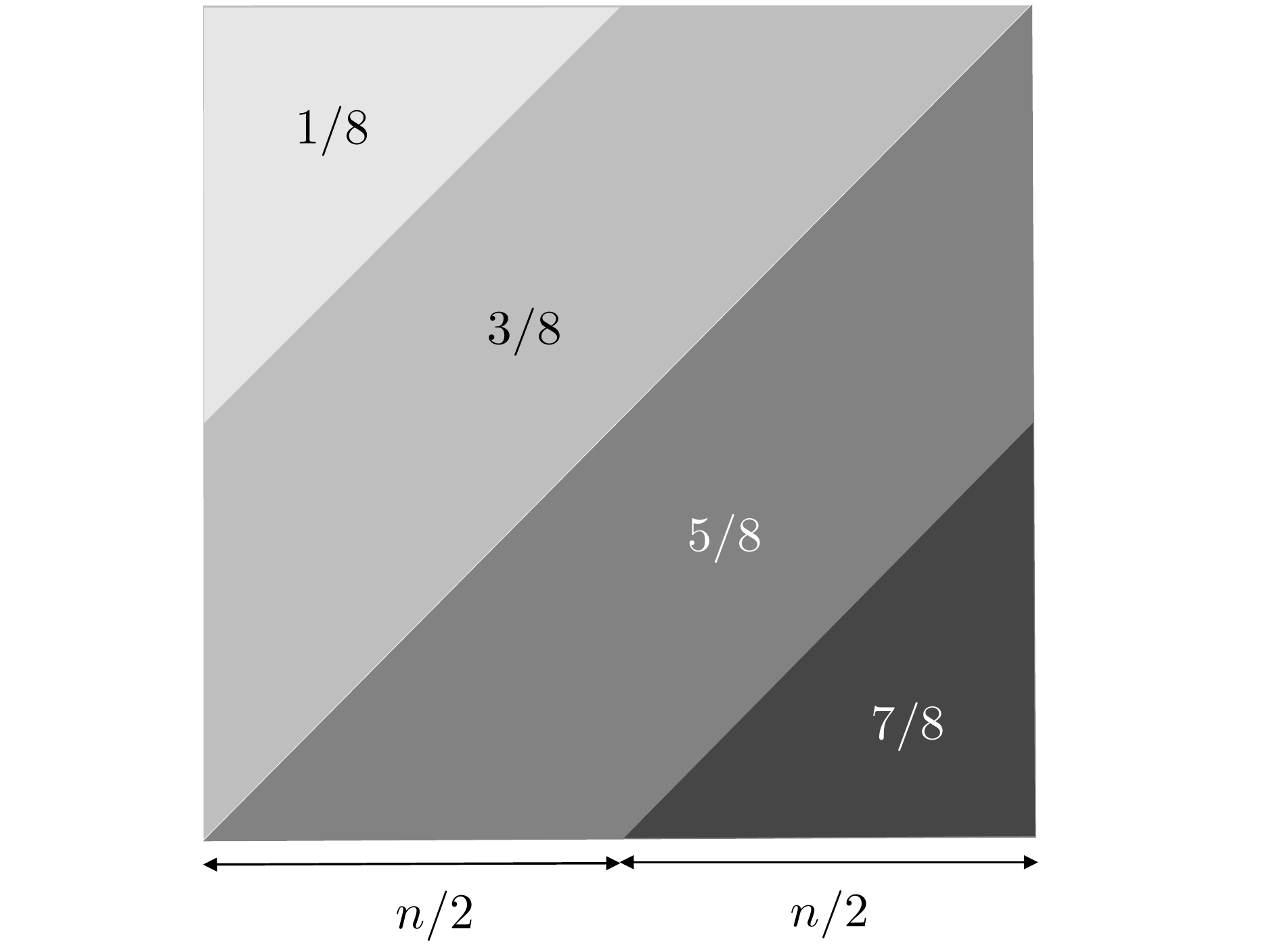}
\end{minipage}
\begin{minipage}[c]{.48\linewidth}
\includegraphics[clip, trim=1.4cm 6.45cm 2.4cm 6.7cm, width=\linewidth]{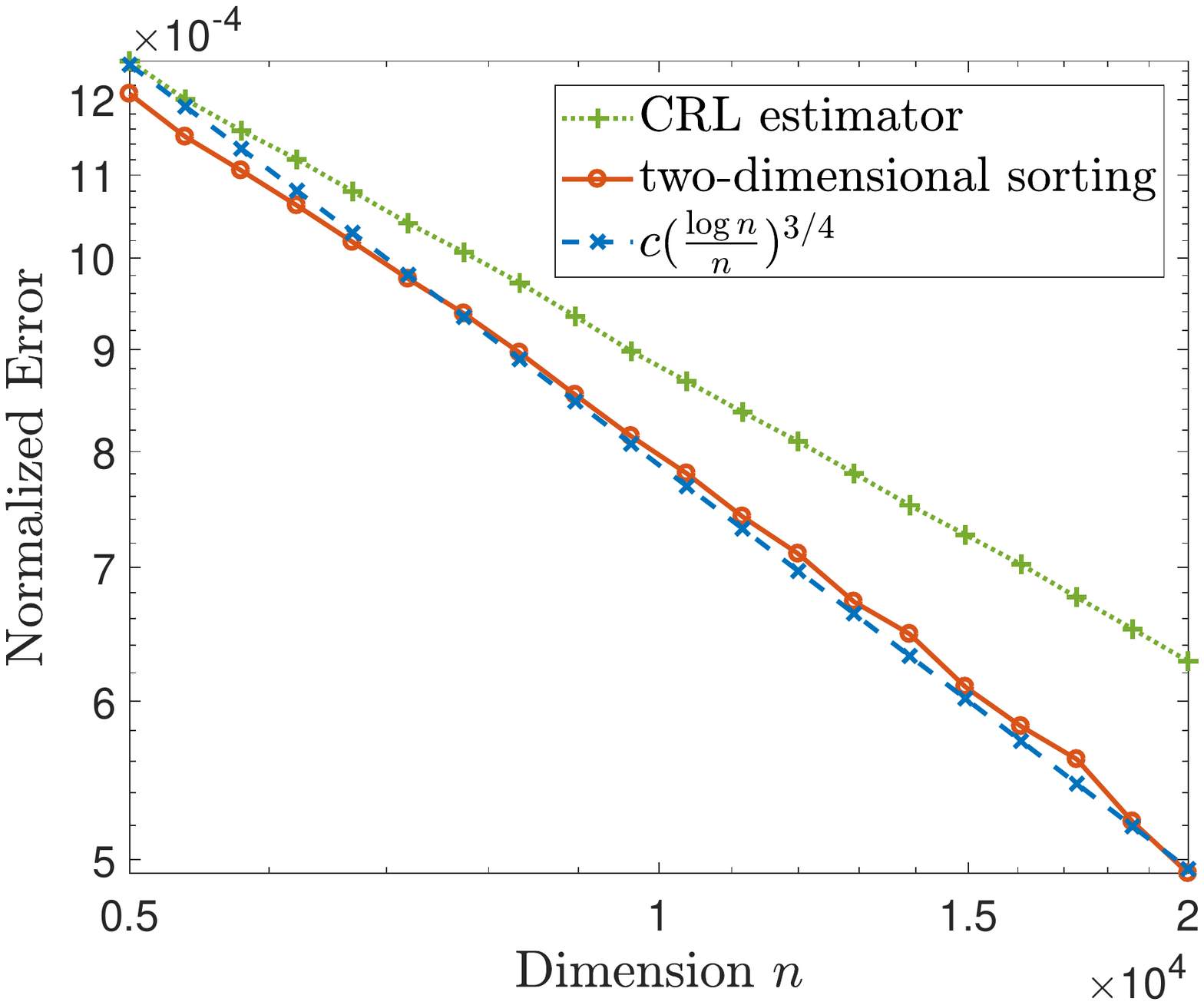}
\end{minipage} 
\caption{\textbf{Left:} A bivariate isotonic matrix; the ground truth
  $M^* \in [0, 1]^{n \times n}$ is a row and column permuted version
  of such a matrix.  \textbf{Right:} A log-log plot of the rescaled
  squared Frobenius error \mbox{$\frac{1}{n^2} \|\Mhat - M^*\|_F^2$}
  versus the matrix dimension $n$.  For each value of the dimension,
  the error is averaged over $10$ experiments each using $n^2$
  Bernoulli observations, and the estimator $\Mhat$ is either the
  two-dimensional sorting estimator that we introduce in
  Section~\ref{sec:tds}, or the $\CRL$ estimator from past
  work~\cite{ShaBalWai16-2}.
}
\label{fig:plot}
\end{figure}




\paragraph{Organization}

In Section~\ref{sec:setup}, we formally introduce our estimation
problem, and describe in detail how it is connected to applications in crowd-labeling and ranking from pairwise
comparisons. Section~\ref{sec:mainresults} contains precise statements and
discussions of our main results, and we provide proofs of our main results in Appendix~\ref{sec:proofs} in the supplementary material. 

\paragraph{Notation}
For a positive integer $n$, let $[n] \defn \{1, 2, \ldots, n\}$. For a
finite set $S$, we use $|S|$ to denote its cardinality. For two
sequences $\{a_n\}_{n=1}^\infty$ and $\{b_n\}_{n=1}^\infty$, we write
$a_n \lesssim b_n$ if there is a universal constant $C$ such that $a_n
\leq C b_n$ for all $n \geq 1$. The relation $a_n \gtrsim b_n$ is
defined analogously.  We use $c, C, c_1, c_2, \dots$ to denote
universal constants that may change from line to line.  We use
$\BER(p)$ to denote the Bernoulli distribution with success
probability $p$, the notation $\BIN(n,p)$ to denote the binomial
distribution with $n$ trials and success probability $p$, and the
notation $\Poi(\lambda)$ to denote the Poisson distribution with
mean $\lambda > 0$.  Given a matrix $M \in \real^{\dimone \times
  \dimtwo}$, its $i$-th row is denoted by $M_i$. 
Let $\symgp_n$ denote
the set of all permutations $\pi: [n] \to [n]$. Let $\id$ denote the
identity permutation, where the dimension can be inferred from
context.


\section{Background and problem setup} \label{sec:setup}

In this section, we present the relevant technical background and
notation on permutation-based models, and introduce the observation
model and error metrics of interest. We also elaborate on how exactly these
models arise in practice.

\subsection{Matrix models}

Our main focus is on designing efficient algorithms for estimating a
bivariate isotonic matrix with unknown permutations acting on its rows
and columns. Formally, we define $\Cbiso$ to be the class of matrices
in $[0,1]^{\dimone \times \dimtwo}$ with non-decreasing rows and
non-decreasing columns. For readability and without loss of generality, we assume frequently (in particular, everywhere except for Proposition~\ref{prop:meta} and Section~\ref{sec:ordered-col}) that
$\dimone \geq \dimtwo$; our results can be
straightforwardly extended to the other case. Given a matrix $M \in
\real^{\dimone \times \dimtwo}$ and permutations $\pi \in
\symgp_{\dimone}$ and $\sigma \in \symgp_{\dimtwo}$, we define the
matrix $M(\pi,\sigma) \in \real^{\dimone \times \dimtwo}$ by
specifying its entries as
\begin{align*}
\left[ M(\pi,\sigma)\right]_{i, j} = M_{\pi(i), \sigma(j)} \text{ for
} i \in [\dimone], j \in [\dimtwo].
\end{align*}
Also define the class  $\Cbiso(\pi, \sigma) \defn \{M(\pi, \sigma): M \in \Cbiso\}$
as the set of matrices that are bivariate isotonic when
viewed along the row permutation $\pi$ and column permutation
$\sigma$, respectively.

The classes of matrices that we are interested in estimating are given
by
\begin{align*}
\Cperm \defn \bigcup_{\substack{\pi \in \symgp_{\dimone} \\ \sigma \in
    \symgp_{\dimtwo}}} \Cbiso(\pi, \sigma), \quad \text{ and its
  subclass } \quad \Cpermr \defn \bigcup_{\pi \in \symgp_{\dimone}}
\Cbiso(\pi, \id).
\end{align*}
The former class contains bivariate isotonic matrices with
both rows and columns permuted, and the latter contains those with
only rows permuted.

\subsection{Observation models} \label{sec:obs}

In order to study estimation from noisy observations of a matrix $M^*$
in either of the classes $\Cperm$ or $\Cpermr$, we suppose that $N$
noisy entries are sampled independently and uniformly with replacement
from all entries of $M^*$.  This sampling model is popular in the
matrix completion literature, and is a special case of the \emph{trace
  regression model}~\cite{NegWai10b,KolLouTsy11}. It has also been
used in the context of permutation-based models by Mao et
al.~\cite{MaoWeeRig17} to study the noisy sorting class.

More precisely, let $E^{(i,j)}$ denote the $\dimone \times \dimtwo$
matrix with $1$ in the $(i,j)$-th entry and $0$ elsewhere, and suppose
that $X_\ell$ is a random matrix sampled independently and uniformly
from the set \mbox{$\{E^{(i,j)}: i \in [\dimone], \, j \in [\dimtwo]\}$.} We
observe $N$ independent pairs $\{(X_\ell,
y_\ell)\}_{\ell=1}^N$ from the model
\begin{align} \label{eq:model}
y_\ell = \trace(X_\ell^\top M^*) + z_\ell,
\end{align}
where the observations are contaminated by independent, zero-mean,
sub-exponential noise $z_\ell$ with parameter $\vars$, that is, 
\begin{align}
\E \exp ( s z_\ell ) \le \exp ( \vars^2 s^2 ) \quad \text{ for all } s \text{ such that } |s| \le 1/\vars . 
\label{eq:sub-exp-cond}
\end{align}
Note that if \eqref{eq:sub-exp-cond} holds for all $s \in \real$, then $z_\ell$ is called sub-Gaussian, which is a stronger condition. 
We assume for convenience that an upper bound on the parameter
$\vars$ is known to our estimators; this assumption is mild since for many
standard noise distributions, such an upper bound is either immediate (in the case of any bounded distribution) or the
standard deviation of the noise is a proxy, up to a universal constant factor, for the parameter $\vars$ (in the case of the Gaussian or Poisson noise models) and can be estimated very accurately by a variety of methods\footnote{For instance, one could implement one of many consistent estimators for $M^*$ to obtain a matrix $\Mhat$, and use the quantity $ \{ \frac{1}{N} \sum_{\ell=1}^N [ y_\ell - \trace (X_\ell^\top \Mhat) ]^2 \}^{1/2}$ as an estimate of the standard deviation.}.

It is important to note at this juncture that although the observation model~\eqref{eq:model} is motivated by the matrix completion literature, we make no assumptions of partial observability in our paper. In particular, our results hold for all tuples $(N, \dimone, \dimtwo)$, with the sample size $N$ allowed to grow larger than the effective dimension $\dimone \dimtwo$.

Besides the standard Gaussian observation model, in which 
\begin{subequations}
\begin{align}
z_{\ell} \stackrel{\mathrm{i.i.d.}}{\sim} \NORMAL(0, 1), \label{eq:Gaussnoise}
\end{align}
another noise model of interest is one which arises in
applications such as crowd-labeling and ranking from pairwise
comparisons.  Here, for every \mbox{$x \in \{E^{(i,j)}: i \in [\dimone], \, j \in [\dimtwo]\}$} and conditioned on $X_{\ell} = x$, 
our observations take the form
\begin{align}
y_\ell \sim \BER\big(\trace(x^\top M^*)\big), \label{eq:Bernoise}
\end{align}
\end{subequations}
and consequently, the sub-exponential parameter $\vars$ is bounded. For a discussion of other regimes of noise in a related matrix model, see Gao~\cite{Gao17}.

For analytical convenience, we employ the standard trick of Poissonization, whereby we assume throughout the paper that $N' = \Poi(N)$ random observations are drawn according to the trace regression model~\eqref{eq:model}, with the Poisson random variable drawn independently of everything else. Upper and lower bounds derived under this model carry over with loss of constant factors to the model with exactly $N$ observations; for a detailed discussion, see Appendix~\ref{app:poi}.


Now given $N' = \Poi(N)$ observations $\{(X_\ell, y_\ell)\}_{\ell=1}^{N'}$, let us define the matrix of observations $Y = Y \left(\{(X_\ell, y_\ell)\}_{\ell=1}^{N'}\right)$, with entry $(i, j)$ given by
\begin{align} \label{eq:obs-Y}
Y_{i,j} = \frac{\dimone \dimtwo}{ N } \sum_{\ell = 1}^{N'} y_\ell \, \bfone \{ X_\ell = E^{(i,j)} \} . 
\end{align}
In other words, we simply average the observations at each entry by the expected number of observations, 
so that
$
\E[Y] = M^*.
$
Moreover, we may write the model in the linearized form $Y = M^* + W$,
where $W$ is a matrix of additive noise having independent, zero-mean entries thanks to Poissonization.\footnote{See, e.g, Shah et al.~\cite{ShaBalGunWai17} for a justification of such a decomposition in the fully observed setting.}
To be more precise, we can decompose the noise at each entry as 
\begin{align*}
&W_{i,j} = Y_{i,j} - M^*_{i,j} \\
&= \frac{\dimone \dimtwo}{ N } \sum_{\ell = 1 }^{N'} z_\ell \cdot \bfone \{ X_\ell = E^{(i,j)} \}
+ M^*_{i,j} \frac{\dimone \dimtwo}{ N } \Big( \sum_{\ell = 1}^{N'} \bfone \{ X_\ell = E^{(i,j)} \} - \frac{N}{\dimone \dimtwo} \Big) . 
\end{align*} 
By Poissonization, the quantities $\sum_{\ell = 1}^{N'} \bfone \{ X_\ell = E^{(i,j)} \}$ for $(i, j) \in [\dimone] \times [\dimtwo]$ are i.i.d. $\Poi ( \frac{N}{\dimone \dimtwo} )$ random variables, so the second term above is simply the deviation of a Poisson variable from its mean. 
On the other hand, the first term is a normalized sum of independent sub-exponential noise. 
Therefore, this linearized and decomposed form of noise provides an amenable starting point for our analysis.

\subsection{Error metrics} \label{sec:err-met}

We analyze estimation of the matrix $M^*$ and the permutations $(\pi^*, \sigma^*)$ in two metrics. For a tuple of ``proper" estimates $(\Mhat, \pihat, \sigmahat)$, in that $\Mhat(\pihat, \sigmahat) \in \Cbiso(\pihat, \sigmahat)$ (and $\sigmahat = \id$ if we are estimating over the class $\Cpermr$), the normalized squared Frobenius error is given by the random variable 
\begin{align*}
\mathcal{F} \big( M^*, \Mhat(\pihat, \sigmahat) \big) = \frac{1}{\dimone \dimtwo} \big\| \Mhat(\pihat, \sigmahat) - M^* \big\|_F^2. 
\end{align*}

The max-row-norm approximation error of the estimate $\pihat$, on the other hand, is given by the random variable
\begin{align*}
\mathcal{R} (M^*, \pihat) &= \max_{i \in [\dimone]} \, \mathcal{R}_i (M^*, \pihat), \text{ where} \\
\mathcal{R}_i (M^*, \pihat) &= \frac{1}{\dimtwo} \big\| [M^*(\pi^*, \sigma^*)]_i - [M^*(\pihat, \sigma^*)]_i \big\|_2^2.
\end{align*}
As will be clear from the sequel, the quantity $\mathcal{R}_i (M^*, \pihat)$ arises as a natural consequence of our development; it represents the approximation error of the permutation estimate $\pihat$ on row $i$ of the matrix $M^*(\pi^*, \sigma^*)$.
%

When estimating over the class $\Cperm$, the max-column-norm error $\mathcal{C} (M^*, \sigmahat)$ is defined analogously as
$
\mathcal{C} (M^*, \sigmahat) = \max_{i \in [\dimtwo]} \, \mathcal{C}_i (M^*, \sigmahat)$,  where 
$
\mathcal{C}_i (M^*, \sigmahat) = \frac{1}{\dimone} \big\| [M^*(\pi^*, \sigma^*)]^i - [M^*(\pi^*, \sigmahat)]^i \big\|_2^2,
$
and we have used $M^i$ to denote the $i$th column of a matrix $M$.
However, since the error $\mathcal{C}$ can be shown to exhibit similar behavior to the error $\mathcal{R}$, it suffices to study the max-row-norm error $\mathcal{R}$ defined above.
The relation between the error metrics for a natural class of algorithms is shown in more rigorous terms by Proposition~\ref{prop:meta}. 

\subsection{Applications}
\label{sec:app}

The matrix models studied in this paper
arise in crowd-labeling and estimation from pairwise comparisons,
and can be viewed as generalizations of low-rank matrices of 
a particular type.  

Let us first describe their relevance to the crowd-labeling problem~\cite{ShaBalWai16}. Here, there is a set of
$\dimtwo$ questions of a binary nature; the true answers to these
questions can be represented by a vector $x^* \in \{0, 1\}^{\dimtwo}$,
and our goal is to estimate this vector by asking these questions to
$\dimone$ \emph{workers} on a crowdsourcing platform. Since workers have varying levels of expertise, it is important to \emph{calibrate} them, i.e., to obtain a good estimate of which workers are reliable and which are not. This is typically done by asking them a set of gold standard questions, which are expensive to generate, and sample efficiency is an extremely important consideration. Indeed, gold standard questions are carefully chosen to control for the level of difficulty and diversity~\cite{le2010ensuring, oleson2011programmatic}. Worker calibration is seen as a key step towards improving the quality of samples collected in crowdsourcing applications~\cite{rzeszotarski2011instrumenting, oyama2013accurate}.

Mathematically, we may model worker abilities
via the probabilities with which they
answer questions correctly, and collect these
probabilities within a matrix $M^* \in [0,1]^{\dimone \times
  \dimtwo}$. The entries of this matrix are latent, and must
  be learned from observing workers' answers to questions.
%
In the calibration problem, we know the answers to the questions; from these, we can estimate worker abilities and question difficulties, or more generally, the entries of
the matrix $M^*$. In many applications, we also have additional knowledge about gold standard questions; for instance, in addition to the true answers, we may also know the relative
difficulties of the questions themselves. 



%
Imposing sensible constraints on the matrix $M^*$ in these
applications goes back to classical work on the subject, with the
majority of models of a \emph{parametric} nature; for instance, the
Dawid-Skene model~\cite{DawSke79} is widely used in crowd-labeling
applications, and its variants have been analyzed by many authors
(e.g.,~\cite{KarOhSha11-b, LiuPenIhl12, DasDasRas13,
  GhoKalMcA11}). However, in a parallel line of work, generalizations
of the parametric Dawid-Skene model have been empirically evaluated on
a variety of crowd-labeling tasks~\cite{welinder2010multidimensional,
  whitehill2009whose}, and shown to achieve performance superior to
the Dawid-Skene model for many such tasks.  The permutation-based
model of Shah et al.~\cite{ShaBalWai16} is one such generalization,
and was proven to alleviate some important pitfalls of parametric
models from both the statistical and algorithmic standpoints.
Operationally, such a model assumes that workers have a total ordering
$\pi$ of their abilities, and that questions have a total ordering
$\sigma$ of their difficulties. The matrix $M^*$ is thus bivariate
isotonic when the rows are ordered in increasing order of worker
ability, and columns are ordered in decreasing order of question
difficulty.
However, since worker abilities and question difficulties are unknown
\emph{a priori}, the matrix of probabilities obeys the inclusion $M^*
\in \Cperm$.  In the particular case where we also know the relative
difficulties of the questions themselves, we may assume that the
column permutation is known, so that our estimation problem is now
over the class $\Cpermr$.

Let us now discuss the application to estimation from pairwise
comparisons. An interesting subclass of $\Cperm$ are those matrices
that are square ($\dimone = \dimtwo = n$), and also skew symmetric.
More precisely, let us define $\Cbiso'$ analogously to the class
$\Cbiso$, except with matrices having columns that are non-increasing
instead of non-decreasing. Also define the class
\begin{subequations}
\begin{align}
  \Cskew(n) & \defn \{M \in [0, 1]^{n \times n}: M + M^\top = 11^\top
  \},
\end{align}
as well as   the \emph{strong stochastic transitivity} class
\begin{align}
\csst(n) \defn \left( \bigcup_{\pi \in \symgp_{n}} \Cbiso'(\pi, \pi)
\right) \bigcap \Cskew(n).
\end{align}
\end{subequations}

The class $\csst(n)$ is useful as a model for estimation from pairwise
comparisons~\cite{Cha15, ShaBalGunWai17}, and was proposed as a
strict generalization of parametric models for this
problem~\cite{BraTer52, NegOhSha16, RajAga14}. In particular, given
$n$ items obeying some unknown underlying ranking $\pi$, entry $(i,
j)$ of a matrix $M^* \in \csst(n)$ represent the probability $\Pr(i
\succ j)$ with which item $i$ beats item $j$ in a comparison. The
shape constraint encodes the transitivity condition that for all
triples $(i, j, k)$ obeying $\pi(i) < \pi(j) < \pi(k)$, we must have
\begin{align*}
\Pr(i \succ k) \geq \max\{\Pr(i \succ j), \Pr(j \succ k)\}.
\end{align*}
For a more classical introduction to these models, see the
papers~\cite{Fis73, McLLuc65, BalWil97} and the references
therein. Our task is to estimate the underlying ranking from results
of passively chosen pairwise comparisons\footnote{Such a passive,
  simultaneous setting should be contrasted with the \emph{active}
  case (e.g.,~\cite{HecShaRamWai16, FalOrlPicSur17, AgaAgaAssKha17}),
  where we may sequentially choose pairs of items to compare depending
  on the results of previous comparisons.}  between the $n$ items, or
more generally, to estimate the underlying probabilities $M^*$ that
govern these comparisons\footnote{Accurate, proper estimates of $M^*$
  in the Frobenius error metric translate to accurate estimates of the
  ranking $\pi$ (see Shah et al.~\cite{ShaBalGunWai17}).}. In
particular, the underlying probabilities could be estimated globally,
as reflected in the Frobenius error $\mathcal{F}$, or locally, as
reflected in the max-row-norm error $\mathcal{R}$. In the latter case,
we require that for each $k$, an estimate of the $k$-th ranked item
must be ``close" to the $k$-th ranked item in ground truth. Here,
items $i$ and $j$ are said to be close if the vector of probabilities
with which item $i$ beats other items is similar to the vector of
probabilities with which item $j$ beats other items.  All results in
this paper stated for the more general matrix model $\Cperm$ apply to
the class $\csst(n)$ with minimal modifications.

The error metric $\mathcal{R}$ has also been used more recently to
learn mixtures of rankings from pairwise comparisons, and guarantees
in this norm have been established in this context for the singular
value thresholding estimator~\cite{shah2018learning}.  Thus, a
concrete theoretical study of the fundamental limits of estimation in
this metric is an important problem for these ranking applications,
and our work provides such an analysis for classes of
permutation-based ranking models. From a technical standpoint, prior
work has often bounded the Frobenius error $\mathcal{F}$ with the
metric $\mathcal{R}$ (see
equations~\eqref{eq:max-norms-r}--\eqref{eq:max-norms-c}), so studying
the error metric $\mathcal{R}$ allows us to better understand the
limitations of existing estimators in the $\mathcal{F}$ metric.


\section{Main results} \label{sec:mainresults}

In this section, we present precise statements of our main results. We
assume throughout this section (unless otherwise stated) that as per
the setup, we have $\dimone \geq \dimtwo$, and provide a summary of
our results---in Table~\ref{tab:results}---for the special case of
square, $n \times n$ matrices, with $N = n^2$. We emphasize that all
our results are significantly more general---for instance, our results
in the Frobenius error hold for all tuples $(N, \dimone,
\dimtwo)$---and Table~\ref{tab:results} captures only a small subset
of them. For example, it does not capture the aforementioned minimax
optimality of our efficient estimators in the Frobenius error when $N$
is large.  Let us first revisit the fundamental limits of estimation
in the Frobenius error, and prove lower bounds in the max-row-norm
error. We then introduce our algorithms in Section~\ref{sec:ex-algo}.

\begin{table}[ht]
\centering
\begin{tabular}{|c|c|c|c|c|}
\hline
\multicolumn{1}{|r|}{\begin{tabular}[c]{@{}r@{}}Model Class \\ $\rightarrow$\end{tabular}} & \multicolumn{2}{c|}{Class $\Cpermr$}                                                                                                                                                & \multicolumn{2}{c|}{Class $\Cperm$}                                                                                                                                                                               \\ \cline{2-5} 
\multicolumn{1}{|l|}{Metric $\downarrow$}                                                  & Lower bounds                                                                                          & Efficient alg.                                                      & Minimax risk                                                                     & Efficient alg.                                                                                                         \\ \hline
\begin{tabular}[c]{@{}c@{}} Frobenius \\ estimation \\ error\\ $\EE\left[ \mathcal{F} \right]$\end{tabular}                    & \begin{tabular}[c]{@{}c@{}}  $\Omega(n^{-1})$ \\  Theorem~\ref{thm:funlim} \\ (Minimax) \end{tabular} & \begin{tabular}[c]{@{}c@{}}$\widetilde{O}(n^{-3/4})$\\ Theorem~\ref{thm:ordered-col} \end{tabular} & \begin{tabular}[c]{@{}c@{}}$\widetilde{\Theta}(n^{-1})$\\ Theorem~\ref{thm:funlim}; \\ \cite{ShaBalGunWai17} \end{tabular}   & \begin{tabular}[c]{@{}c@{}}$\widetilde{O}(n^{-3/4})$\\ Theorem~\ref{thm:fast-tds}; \\ \cite{mao18breakingcolt} \end{tabular}
 \\ \hline
\begin{tabular}[c]{@{}c@{}}Max-row-norm\\ approximation\\ error\\ $\EE\left[ \mathcal{R} \right]$\end{tabular}   & \begin{tabular}[c]{@{}c@{}}$\Omega(n^{-3/4})$\\ Theorem~\ref{thm:rowlb} \\ (Local alg.) \end{tabular}                      & \begin{tabular}[c]{@{}c@{}}$\widetilde{O}(n^{-3/4})$\\ Theorem~\ref{thm:ordered-col} \end{tabular} & \begin{tabular}[c]{@{}c@{}}$\widetilde{\Theta}(n^{-1/2})$\\ Theorem~\ref{thm:rowlb} \end{tabular} & \begin{tabular}[c]{@{}c@{}}$\widetilde{O}(n^{-1/2})$\\ Theorem~\ref{thm:fast-tds}; \\ \cite{ShaBalWai16-2, ChaMuk16} \end{tabular}
                                                    \\ \hline
\end{tabular}
\caption{Estimation rates for each model class and metric for $n
  \times n$ matrices with Bernoulli or standard Gaussian observations
  of $n^2$ entries. The lower bound for the class $\Cpermr$ in the
  metric $\mathcal{R}$ holds for a class of natural ``local"
  algorithms introduced in Definition~\ref{def:pdd}; many algorithms
  that we are aware of (including our own) are covered by this
  class. \\ (Theorem numbers reference the present
  paper.) \vspace{-5mm}}
\label{tab:results}
\end{table}

\subsection{Statistical limits of estimation} \label{sec:funlim}
We begin by characterizing the fundamental limits of estimation under the trace regression observation model~\eqref{eq:model} with $N' = \Poi(N)$ observations. We define the least squares estimator over a closed set~$\mathbb{C}$ of $\dimone \times \dimtwo$ matrices as the projection
\begin{align*}
\Mhatls(\mathbb{C}, Y) \in \arg\min_{M \in \mathbb{C}} \|Y - M \|_F^2 .
\end{align*}
The projection is a non-convex problem when the class $\mathbb{C}$ is given by either the class $\Cperm$ or $\Cpermr$, and is unlikely to be computable exactly in polynomial time. However, studying this estimator allows us to establish a baseline that characterizes the best achievable statistical rate. 
In the following theorem, we characterizes the Frobenius risk of the least squares estimator, and also provide a minimax lower bound. These results hold for any sample size $N$ of the problem.
%
Also recall the shorthand $Y = Y \left(\{X_{\ell}, y_{\ell} \}_{\ell = 1}^{N'}\right)$, and let $\bar{\vars} \defn \vars \lor 1$ denote the proxy for the noise that accounts for missing data.
\begin{theorem} \label{thm:funlim}
(a) Suppose that $\dimtwo \leq \dimone$. 
There is an absolute constant $c_1 > 0$ such that for any matrix $M^* \in \Cperm$, we have 
\begin{align} \label{eq:ls-upper}
&\mathcal{F} \big( M^*, \Mhatls(\Cperm, Y) \big) \leq 
c_1 \bigg\{ 
\bar{\vars}^2 \frac{\dimone }{N} \log \dimone + \bar{\vars} \frac{\dimtwo}{N} (\log \dimone)^2   \\ 
& \qquad + \Big( \bar{\vars} \frac{ 1 }{\sqrt{N}} (\log \dimone )^2 \Big) \land
\Big( \bar{\vars}^2 \frac{ \dimtwo }{ N } \log \dimone \Big)^{2/3}  \land
\Big( \bar{\vars}^2 \frac{ \dimone \dimtwo }{ N } \log N \Big)  \bigg\} \land 1  \notag 
\end{align}
with probability at least $1 - \dimone^{-\dimone}$.

\noindent (b) Suppose that $\dimtwo \le \dimone$, and that $N' \sim \Poi(N)$ independent samples are drawn under the standard Gaussian observation
model~\eqref{eq:Gaussnoise} or the Bernoulli observation
model~\eqref{eq:Bernoise}. Then there exists an absolute constant $c_2 > 0$ such that
any
estimator $\Mhat$ satisfies
\begin{align}
  \label{eq:lower}
\sup_{M^* \in \Cpermr} \EE \left[ \mathcal{F}(M^*, \Mhat) \right]
\geq c_2 \left\{ \left[ \frac{\dimone}{N} + \left( \frac{1}{\sqrt{N} } \wedge \left(\frac{\dimtwo}{N} \right)^{2/3} \wedge \frac{\dimone \dimtwo}{N} \right) \right] \land 1 \right\}.
\end{align}
\end{theorem}
When interpreted in the context of square matrices under partial
observations, our result should be viewed as paralleling that of Shah
et al.~\cite{ShaBalGunWai17}. In addition, however, the result also
provides a generalization in several directions.  First, the upper
bound holds under the general sub-exponential noise model.  Second,
the lower bound holds for the class $\Cpermr$, which is strictly
smaller than the class $\Cperm$.  Third, and more importantly, we
study the problem for arbitrary tuples $(N, \dimone, \dimtwo)$, and
this allows us to uncover interesting phase transitions in the rates
that were previously unobserved\footnote{The regime $N \geq n_1 n_2$  
is interesting for the problems of ranking and crowd-labeling that 
motivate our work, since it is pertinent to compare items or ask workers 
to answer questions multiple times in order to reduce the noisiness of the gathered data. From a theoretical standpoint, the large $N$ regime isolates multiple non-asymptotic behaviors on the way to asymptotic consistency as $N \to \infty$, and is related in spirit to recent work on studying the high signal-to-noise ratio regime in ranking models~\cite{Gao17}, where phase transitions were also observed.}; these are discussed in more detail
below.

Via the inclusion $\Cpermr \subseteq \Cperm$, we observe that for both classes $\Cpermr$ and $\Cperm$, the upper and lower bounds match up to logarithmic factors in all regimes of $(N, \dimone, \dimtwo)$ under the standard Gaussian or Bernoulli noise model. Such a poly-logarithmic gap between the upper and lower bounds is related to corresponding gaps that exist in upper and lower bounds on the metric entropy of bounded bivariate isotonic matrices~\cite{ChaGunSen18}. Closing this gap is known to be an important problem in the study of these shape-constrained objects (see, e.g.~\cite{gao2007entropy}).

Let us now interpret the theorem in more detail. 
First, note that the risk of any proper estimator 
is bounded by $1$ since the entries of $M^*$ are bounded in the interval $[0, 1]$. We thus focus on the remaining terms of the bound.
Up to a poly-logarithmic factor, the upper bound can be simplified to 
\mbox{$
\bar{\vars}^2 \frac{\dimone }{N} + \big( \bar{\vars} \frac{ 1 }{\sqrt{N}} \big) \land
\big( \bar{\vars}^2 \frac{ \dimtwo }{ N } \big)^{2/3}  \land
\big( \bar{\vars}^2 \frac{ \dimone \dimtwo }{ N } \big) . 
$}
The first term $\bar{\vars}^2 \frac{\dimone }{N}$ is due to the unknown permutation on the rows (which dominates the unknown column permutation since $\dimone \geq \dimtwo$). 
The remaining terms, corresponding to the minimum of three rates, stem from estimating the underlying bivariate isotonic matrix, so we make a short digression to state a corollary specialized to this setting. Recall that $\Cbiso$ was used to denote the class of $\dimone \times \dimtwo$ bivariate isotonic matrices. 
Owing to the convexity of the set $\Cbiso$, the least squares estimator $\Mhatls(\Cbiso, Y)$ is computable efficiently~\cite{BriDykPilRob84, KynRaoSac15}. 
Moreover, we use the shorthand notation 
\begin{align}
\label{eq:shorthand}
&\vartheta(N, \dimone, \dimtwo, \vars) \defn \bar{\vars} \frac{\dimtwo}{N} (\log \dimone)^2 + \\
&\qquad \qquad \Big( \bar{\vars} \frac{ 1 }{\sqrt{N}} (\log \dimone )^2 \Big) \land
\Big( \bar{\vars}^2 \frac{ \dimtwo }{ N } \log \dimone \Big)^{2/3}  \land
\Big( \bar{\vars}^2 \frac{ \dimone \dimtwo }{ N } \log N \Big) . \notag
\end{align}

\begin{corollary} \label{cor:biso}
(a) Suppose that $\dimtwo \leq \dimone$. Then there is a universal constant $c_1 > 0$ such that for any matrix $M^* \in \Cbiso$, we have
\begin{align*}
\mathcal{F} \big( M^*, \Mhatls(\Cbiso, Y) \big) 
\leq c_1 \vartheta(N, \dimone, \dimtwo, \vars) \wedge 1 
\end{align*}
with probability at least $1 - \dimone^{- 3 \dimone}$. 

\noindent (b) Suppose that $\dimtwo \le \dimone$, and that $N' \sim \Poi(N)$ independent samples are drawn under the standard Gaussian observation
model~\eqref{eq:Gaussnoise} or the Bernoulli observation
model~\eqref{eq:Bernoise}. Then there exists an absolute constant $c_2 > 0$ such that
any
estimator $\Mhat$ satisfies
\begin{align*}
\sup_{M^* \in \Cbiso} \EE \left[ \mathcal{F}(M^*, \Mhat) \right] \geq c_2 \left\{ \frac{ 1 }{\sqrt{N}} \land
\Big( \frac{ \dimtwo }{ N } \Big)^{2/3}  \land
\frac{ \dimone \dimtwo }{ N }  \wedge 1 \right\}  .
\end{align*}
\end{corollary}
 
%
%
Corollary~\ref{cor:biso} should be viewed as paralleling the results
of Chatterjee et al.~\cite{ChaGunSen18} (see, also, Han et al.~\cite{han2019}) 
under a slightly different
noise model, while providing some notable extensions once
again. Firstly, we handle sub-exponential noise; secondly, the bounds
hold for all tuples $(N, \dimone, \dimtwo)$ and are optimal up to a
logarithmic factor provided the sample size $N$ is sufficiently
large. In more detail, the non-parametric rate $\bar{\vars} \frac{ 1
}{\sqrt{N}}$ was also observed in Theorems~2.1 and 2.2 of the
paper~\cite{ChaGunSen18} provided in the fully observed setting ($N =
\dimone \dimtwo$), with the lower bound additionally requiring that
the matrix was not extremely skewed. In addition to this rate, we also
isolate two other interesting regimes when $N \geq \dimtwo^2$ and
$1/\sqrt{N}$ is no longer the minimizer of the three terms above. The
first of these regimes is the rate $( \bar{\vars}^2 \frac{ \dimtwo }{
  N } )^{2/3}$, which is also non-parametric; notably, it corresponds
to the rate achieved by decoupling the structure across columns and
treating the problem as $\dimtwo$ separate isotonic regression
problems~\cite{NemPolTsy85, Zha02}. This suggests that if the matrix
is extremely skewed or if $N$ grows very large, monotonicity along the
smaller dimension is no longer as helpful; the canonical example of
this is when $\dimtwo = 1$, in which case we are left with the
(univariate) isotonic regression problem\footnote{Indeed, in a similar regime, Chatterjee et al.~\cite{ChaGunSen18} show in their Theorem 2.3 that the upper bound is achieved by an estimator that performs univariate regression along each row, followed by a projection onto the set of BISO matrices. On the other hand, our results establish near-optimal minimax rates in a unified manner through metric entropy estimates, so that the upper bounds are sharp in all regimes simultaneously, and hold for the least squares estimator under sub-exponential noise.}.  The final rate
$\bar{\vars}^2 \frac{ \dimone \dimtwo }{ N }$ is parametric and
comparatively trivial, as it can be achieved by an estimator that
simply averages observations at each entry. This suggests that when
the number of samples grows extremely large, we can ignore all
structure in the problem and still be optimal at least in a minimax
sense.
Let us now return to a discussion of Theorem~\ref{thm:funlim}. To further clarify the rates and transitions between them, we simplify the discussion by focusing on two regimes of matrix dimensions.

\paragraph{Example 1: $\dimtwo \leq \dimone \leq \dimtwo^2$} Here, by treating $\vars$ as a constant, we may simplify the minimax rate (up to logarithmic factors) as 
\begin{align} \label{eq:lower-simple-1}
\inf_{\Mhat} \sup_{M^* \in \Cperm} \EE \left[ \mathcal{F}(M^*, \Mhat) \right] \asymp
\begin{cases}
1 & \text{ for } N \le \dimone \\
\frac{\dimone}{N} & \text{ for } \dimone \le N \le \dimone^2 \\
\frac{1}{\sqrt{N} } & \text{ for } \dimone^2 \le N \le \dimtwo^4 \\
(\frac{\dimtwo}{N})^{2/3} & \text{ for } \dimtwo^4 \le N \le \dimone^3 \dimtwo \\
\frac{\dimone \dimtwo}{N} & \text{ for } N \ge \dimone^3 \dimtwo
\end{cases} 
\end{align}
which delineates five distinct regimes depending on the sample size $N$.
The first regime is the trivial rate. 
The second regime $\dimone \le N \le \dimone^2$ is when the error due to the latent permutation dominates, while the third regime $\dimone^2 \le N \le \dimtwo^4$ corresponds to when the hardness of the problem is dominated by the structure inherent to bivariate isotonic matrices. For $N$ larger than $\dimtwo^4$,  the effect of \emph{bivariate} isotonicity disappears, at least in a minimax sense. 
Namely, in the fourth regime $\dimtwo^4 \le N \le \dimone^3 \dimtwo$, the rate $(\dimtwo / N)^{2/3}$ is the same as if we treat the problem as $\dimtwo$ separate $\dimone$-dimensional isotonic regression problems with an unknown permutation~\cite{FlaMaoRig16}. 
For even larger sample size $N \ge \dimone^3 \dimtwo$, in the fifth regime, the minimax-optimal rate $\dimone \dimtwo / N$ is trivially achieved by ignoring all structure and outputting the matrix $Y$ alone. \hfill $\clubsuit$ 

\paragraph{Example 2: $\dimtwo \leq \dimone \leq C \dimtwo$} In this near-square regime, we may once again simplify the bound and obtain (up to logarithmic factors) that 
\begin{align} \label{eq:lower-simple-2}
\inf_{\Mhat} \sup_{M^* \in \Cperm} \EE \left[ \mathcal{F}(M^*, \Mhat) \right] \asymp
\begin{cases}
1 & \text{ for } N \le \dimone \\
\frac{\dimone}{N} & \text{ for } \dimone \le N \le \dimone^2 \\
\frac{1}{\sqrt{N} } & \text{ for } \dimone^2 \le N \le \dimone^2 \dimtwo^2 \\
\frac{\dimone \dimtwo}{N} & \text{ for } N \ge \dimone^2 \dimtwo^2
\end{cases} 
\end{align}
so that two of the cases from before now collapse into one. Ignoring the trivial constant rate, we thus observe a transition from a parametric rate to a non-parametric rate, and back to the trivial parametric rate. \hfill $\clubsuit$ 
\smallskip


Having discussed our minimax risk bounds in the Frobenius error,
we now turn to establishing lower bounds in the max-row-norm metric. We show two results in this context: our first result is a minimax lower bound for the class $\Cperm$, and the second result
is a lower bound for the class $\Cpermr$ that holds for a class of natural estimators defined below.

\begin{definition}[Pairwise Distinguishability via Differences (PDD)] \label{def:pdd}
A row permutation estimator $\pihat$ is said to obey the PDD property if it is given by the following procedure: 
\begin{itemize}
\item Step~1: 
Create a directed graph $G$ with vertex set $[\dimone]$, where for each pair of distinct indices $i, j \in [\dimone]$, the existence of an edge between $i$ and $j$, as well as its direction (if the edge exists), depends only on the row difference $Y_i - Y_j$. 
\item Step~2: 
Return a uniformly random topological sort $\pihat$ of the graph $G$ if there exists one; otherwise, return a uniformly random permutation.  
\end{itemize}
Denote the class of (row-)PDD estimators by $\mathcal{P}^{\mathsf{r}}_{\mathsf{pdd}}$. The class of column-PDD estimators $\mathcal{P}^{\mathsf{c}}_{\mathsf{pdd}}$ is defined analogously.
\end{definition}

Recall that a permutation $\pi$ is called a \emph{topological sort} of
a directed acyclic graph $G$ if $\pi(u) < \pi(v)$ for every directed
edge $u \to v$.  The technical terms of a graph and its topological
sort are adopted mainly for convenience and should not obscure the
intuition behind the PDD property.  Intuitively, in a row-PDD
estimator, the decision of whether to rank row $i$ above row $j$
(represented by a directed edge in Step~1) depends on the matrix $Y$
only through the difference of the rows $Y_i - Y_j$, and the final permutation is obtained by aggregating such pairwise decisions (via the topological sort). Thus, any PDD
estimator is ``local" in this specific sense, and defines the class of
estimators alluded to in Table~\ref{tab:results}.
In the context of SST matrix estimation, the row and column PDD
properties are equivalent, and many existing permutation estimators
for this class based on the Borda and Copeland counts~\cite{ChaMuk16,
  shah2017simple, PanMaoMutWaiCou17} can be verified\footnote{To be
  clear, these procedures return any permutation that sorts the
  row sums of the matrix $Y$. Any such estimate can be made PDD by
  choosing a uniform (sub-)permutation for all rows with the same
  sum.} to lie in the class $\mathcal{P}^{\mathsf{r}}_{\mathsf{pdd}}$.

We now state lower bounds for both classes of matrices in the $\mathcal{R}$ metric.
\begin{theorem}
\label{thm:rowlb}
Suppose that $\dimtwo \leq \dimone$ and $N \leq \dimone
\dimtwo$. \\ (a) Suppose that $N' \sim \Poi(N)$ independent samples
are drawn from either the standard Gaussian observation
model~\eqref{eq:Gaussnoise} or the Bernoulli observation
model~\eqref{eq:Bernoise}.  Then there exists an absolute constant $c
> 0$ such that when estimating over the class $\Cperm$, any
permutation estimate $\pihat$ has worst-case error at least
\begin{align}
  \label{eq:rowlb-lower2}
\sup_{M^* \in \Cperm} \EE \left[ \mathcal{R} (M^*, \pihat) \right]
\geq c \left[ 1 \wedge \left( \frac{\dimone}{N} \right)^{1/2} \right].
\end{align}
(b) Suppose that $N' \sim \Poi(N)$ independent samples are drawn from
the standard Gaussian observation model~\eqref{eq:Gaussnoise}.
Then, there exists an absolute constant $c > 0$ such that when
estimating over the class $\Cpermr$, any PDD row permutation estimate
$\pihat \in \mathcal{P}^{\mathsf{r}}_{\mathsf{pdd}}$ hast worst-case
error at least
\begin{align} \label{eq:rowlb-lower1}
\sup_{M^* \in \Cpermr} \EE \left[ \mathcal{R} (M^*, \pihat) \right]
\geq c \left[ 1 \wedge \left( \frac{\dimone}{N} \right)^{3/4} \right].
\end{align}
\end{theorem}
The bound~\eqref{eq:rowlb-lower2} is optimal up to a constant factor,
and attained by an estimator that is computable in polynomial
time~\cite{PanMaoMutWaiCou17}. Indeed, we also establish a version of
this fact (up to a logarithmic factor) in
Theorem~\ref{thm:ordered-col} to follow, and other existing
algorithms~\cite{ShaBalWai16-2, ChaMuk16} are also able to match the
lower bound~\eqref{eq:rowlb-lower2} up to a logarithmic factor.
As will be clarified shortly, the minimax lower bound~\eqref{eq:rowlb-lower2} has another important consequence, and shows why prior work on this problem was unable to surpass what was perceived as a fundamental gap in estimation in the Frobenius error. 

On the other hand, the bound~\eqref{eq:rowlb-lower1} is optimal (up to logarithmic factors) for the class of estimators $\mathcal{P}^{\mathsf{r}}_{\mathsf{pdd}}$, and we demonstrate a matching upper bound in Theorem~\ref{thm:ordered-col}. We also conjecture\footnote{It is worth noting that proceeding analogously to part (a) of the theorem and generalizing existing results~\cite{WeiGunWai17} on testing the monotone cone against the zero vector yields the (weaker) minimax lower bound 
\[
\inf_{\pihat} \sup_{M^* \in \Cpermr} \EE \left[ \mathcal{R} (M^*, \pihat) \right]
\geq c \left[ 1 \wedge \frac{\dimone \log \dimtwo}{N} \right].
\]}
that the minimax lower bound
\begin{align} \label{eq:conj-row}
\inf_{\pihat} \sup_{M^* \in \Cpermr} \EE \left[ \mathcal{R} (M^*, \pihat) \right]
\geq c \left[ 1 \wedge \left( \frac{\dimone}{N} \right)^{3/4} \right] 
\end{align}
holds for an absolute constant $c > 0$, with the infimum taken over all measurable functions from the observations to the set of permutations $\symgp_{\dimone}$.

Theorem~\ref{thm:rowlb} is proved via reductions to particular
hypothesis testing problems on cones. As a consequence of proving part
(a) of the theorem, we extend existing lower bounds on the minimax
radius of testing the positive orthant cone against the zero
vector~\cite{WeiGunWai17}, also accommodating Bernoulli noise and
missing data in our observations. This result, provided in
Proposition~\ref{prop:orthcone}, may be of independent interest.  In
order to prove part (b), we show a lower bound on the minimax radius
of testing, given noisy observations of two vectors in the positive,
monotone cone, whether or not one of the vectors is entry-wise larger
than the other. This result, collected in
Proposition~\ref{prop:moncone}, may also be of independent interest
since it is not covered by existing theory on cone testing
problems~\cite{WeiGunWai17}.

Having completed our discussion of the fundamental limits of
estimation, let us now turn to a discussion of computationally
efficient algorithms.


\subsection{Efficient algorithms} 
\label{sec:ex-algo}

Our algorithms belong to a broader family of algorithms that rely on
two distinct steps: first, estimate the unknown permutation(s)
defining the problem; then project onto the class of matrices that are
bivariate isotonic when viewed along the estimated
permutations. Formally, any such algorithm is described by the
meta-algorithm below.

\paragraph{Algorithm 1 (meta-algorithm)}
\begin{itemize}
\item Step 1: Use any algorithm to obtain permutation estimates $(\pihat, \sigmahat)$, setting $\sigmahat = \id$ if estimating over class $\Cpermr$. 
\item Step 2: Return the matrix estimate \\
$
\Mhat(\pihat, \sigmahat) \defn \arg \min_{M \in \Cbiso(\pihat, \sigmahat)} \| Y - M \|_F^2 .
$
\end{itemize}
Owing to the convexity of the set $\Cbiso(\pihat, \sigmahat)$, the projection operation in Step 2 of the algorithm can be computed in near linear time~\cite{BriDykPilRob84, KynRaoSac15}. 
The following result, a variant of Proposition~4.2 of Chatterjee and Mukherjee~\cite{ChaMuk16}, allows us to characterize the error rate of any such meta-algorithm as a function of the permutation estimates $(\pihat, \sigmahat)$.

Recall the definition of the set $ \Cbiso(\pi, \sigma) \defn \{M(\pi, \sigma): M
\in \Cbiso\} $ as the set of matrices that are bivariate isotonic when
viewed along the row permutation $\pi$ and column permutation
$\sigma$, respectively. In particular, we have the inclusion $M^* \in \Cbiso(\pi^*, \sigma^*)$ where $\pi^*$ and $\sigma^*$
are unknown permutations in $\symgp_{\dimone}$ and $\symgp_{\dimtwo}$,
respectively. In the following proposition, we also do not make the assumption $\dimtwo \leq \dimone$; recall our shorthand notation $\vartheta(N, \dimone, \dimtwo, \vars)$ defined in equation~\eqref{eq:shorthand}.

\begin{proposition}
\label{prop:meta}
There exists an absolute constant $C > 0$ such that for all $M^* \in \Cbiso(\pi^*, \sigma^*)$, the estimator $\Mhat(\pihat, \sigmahat)$ obtained by running the meta-algorithm satisfies
\begin{align} 
& \mathcal{F} \big( M^*, \Mhat(\pihat, \sigmahat) \big) \label{eq:oracle} 
\leq C \Big\{  \vartheta(N, \dimone \vee \dimtwo, \dimone \land \dimtwo, \vars) \\
& + \frac{1}{\dimone \dimtwo} \big\|M^*(\pi^*, \sigma^*) - M^*(\pihat, \sigma^*)
\big\|_F^2 + \frac{1}{\dimone \dimtwo} \big\|M^*(\pi^*, \sigma^*) - M^*(\pi^*, \sigmahat)
\big\|_F^2 \Big\} \notag
\end{align}
with probability exceeding $1 - \dimone^{-\dimone}$.
\end{proposition}
A few comments are in order. The term $\vartheta(N, \dimone \vee \dimtwo, \dimone \land \dimtwo, \vars)$ on the upper line of the RHS of
the bound~\eqref{eq:oracle} corresponds to an estimation error, if the
true permutations $\pi$ and $\sigma$ were known a priori (see Corollary~\ref{cor:biso}), and the
latter terms on the lower line correspond to an approximation error that we incur as
a result of having to estimate these permutations from data.
Comparing the bound~\eqref{eq:oracle} to the minimax lower
bound~\eqref{eq:lower}, we see that up to a poly-logarithmic factor, the
estimation error terms of the bound~\eqref{eq:oracle} are unavoidable, and so we can restrict our attention to obtaining good permutation estimates
$(\pihat, \sigmahat)$.

Past work~\cite{ShaBalWai16-2, ChaMuk16} has typically proceeded from
equation~\eqref{eq:oracle} by using the inequalities
\begin{subequations}
\begin{align} 
\frac{1}{\dimone \dimtwo} \big\|M^*(\pi^*, \sigma^*) - M^*(\pihat,
\sigma^*) \big\|_F^2 &\leq \mathcal{R} (M^*, \pihat), \quad \text{
  and} \label{eq:max-norms-r} \\ 
\frac{1}{\dimone \dimtwo} \big\|M^*(\pi^*, \sigma^*) -
M^*(\pi^*, \sigmahat) \big\|_F^2 &\leq \mathcal{C} (M^*, \sigmahat) \label{eq:max-norms-c}
\end{align}
\end{subequations}
to then reduce the problem to bounding the sum of max-row-norm and max-column-norm errors. However, irrespective of how good
the permutation estimates $(\pihat, \sigmahat)$ really are, such an
analysis approach necessarily produces sub-optimal rates in the Frobenius error
for the class $\Cperm$, owing to the minimax
lower bound~\eqref{eq:rowlb-lower2}. In
particular, given $N = n^2$ noisy observations of an $n \times n$ matrix, any such analysis cannot improve upon the Frobenius error
rate $n^{-1/2}$ for matrices in the class $\Cperm$. 
Consequently, our algorithm for the class $\Cperm$ (defined in Section~\ref{sec:tds} to follow) exploits a finer analysis
technique for the approximation error terms so as to guarantee faster rates.

We now present two permutation estimation procedures that can be
plugged into Step~1 of the meta-algorithm.


\subsubsection{Matrices with ordered columns}
\label{sec:ordered-col}

As a stepping stone to our main algorithm, which estimates over the
class $\Cperm$, we first consider the estimation problem when the
permutation along one of the dimensions is known. This corresponds to
estimation over the subclass $\Cpermr$, and following the
meta-algorithm above, it suffices to provide a permutation estimate
$\pihat$.  The result of this section holds without the assumption
$\dimone \ge \dimtwo$.

We need more notation to facilitate the description of the
algorithm. 
We say that $\bl = \{B_k\}_{k=1}^{|\bl|}$ is a \emph{partition} of $[\dimtwo]$, if
  $[\dimtwo] = \bigcup_{k=1}^{|\bl|} B_k$ and $B_j \cap B_k = \varnothing$
  for $j \neq k$.
Moreover, we group the columns of a matrix $Y \in
\real^{\dimone \times \dimtwo}$ into $|\bl|$ blocks according to their
indices in $\bl$, and refer to $\bl$ as a partition or \emph{blocking}
of the columns of $Y$. In the algorithm, partial row sums of $Y$ are computed on indices contained
in each block.

\paragraph{Algorithm 2 (sorting partial sums)}
\begin{itemize}
\item Step 1: Choose a partition $\blref$ of the set $[\dimtwo]$ consisting of contiguous blocks, such that each block $B$ in $\blref$ has size
\begin{align*}
\frac 12 \blocksize \le |B| \le \blocksize .
\end{align*}
\item Step 2: Given the observation matrix $Y$, compute the row sums
\begin{align*}
S(i) = \sum_{j \in [\dimtwo]} Y_{i,j} \quad \text{ for each $i \in
  [\dimone]$,}
\end{align*} 
and the partial row sums within each block 
\begin{align*}
S_{B}(i) = \sum_{j \in B} Y_{i,j} \quad \text{ for each $i \in
  [\dimone]$ and $B \in \blref$.}
\end{align*}
Create a directed graph $G$ with vertex set $[\dimone]$, where an edge
$u \to v$ is present if either
\begin{align*}
S(v) - S(u) & > 16 \varplusone \bigg( \sqrt{\frac{\dimone
    \dimtwo^2}{N} \log(\dimone \dimtwo) } + \frac{\dimone \dimtwo}{N}
\log(\dimone \dimtwo) \bigg), \text{ or} 
\\ S_B(v) - S_B(u) & > 16 \varplusone \bigg(
\sqrt{\frac{\dimone \dimtwo}{N} |B| \log(\dimone \dimtwo) } +
\frac{\dimone \dimtwo}{N} \log(\dimone \dimtwo) \bigg) \\
& \qquad \qquad \qquad \qquad \qquad \qquad \qquad \text{ for some
} B \in \blref. 
\end{align*}

\item Step 3: Return any topological sort $\pihatref$ of the graph $G$; if none exists, return a uniformly random permutation $\pihatref$. 
%
%
\end{itemize}

Note that if we return a uniformly random topological sort instead of an arbitrary one in Step~3 above, then Algorithm~2 yields a PDD estimator $\pihatref'$ by definition~\eqref{def:pdd}. 
There is no loss of generality in studying $\pihatref$ instead of $\pihatref'$ in this section, because we established the lower bound (Theorem~\ref{thm:rowlb}(b)) for the average-case topological sort while we will prove the upper bound (Theorem~\ref{thm:ordered-col}) for the worst-case topological sort.

We now turn to a detailed discussion of the running time of Algorithm~2.
A topological sort of a generic graph $G(V, E)$ can be found via Kahn's algorithm~\cite{kahn1962topological} in time $\order(|V|+|E|)$. In our context, the topological sort operation translates to a running time of $\order(\dimone^2)$. 
In Step 2,
constructing the graph $G$ takes time $\order(\dimone^2
\dimtwo^{1/2})$, since there are at most $\order(\dimtwo^{1/2})$
blocks. 
This leads to a total complexity of the order $\order(\dimone^2 \dimtwo^{1/2})$.

Let us now give an intuitive explanation for the algorithm. While algorithms in past work~\cite{ShaBalWai16-2, ChaMuk16, PanMaoMutWaiCou17} sort the rows of the matrix according to the full Borda counts $S(i)$ defined in Step 2, they are limited by the high standard deviation in these estimates. Our key observation is that when the columns are perfectly ordered, judiciously chosen partial row sums (which are less noisy than full row sums) also contain information that can help estimate the underlying row permutation. The thresholds on the score differences in Step 2 are chosen to be comparable to the standard deviations of the respective estimates, with additional logarithmic factors that allow for high-probability statements via application of Bernstein's bounds.


\begin{theorem} \label{thm:ordered-col}
There exists an absolute constant $c_1 > 0$ such that for any matrix $M^* \in \Cpermr$, we have
\begin{subequations}
\begin{align} \label{eq:r-upper}
\mathcal{R}(M^*, \pihatref) \leq 1 \wedge c_1 \bar{\vars}^2 \left( \frac{\dimone
  \log \dimone}{N} \right)^{3/4}
\end{align}
with probability at least $1- 2(\dimone \dimtwo)^{-2}$.  Consequently, there exists an absolute constant $c_2 > 0$ such that 
\begin{align} \label{eq:r-upper2}
\mathcal{F}(M^*, \Mhat(\pihatref, \id)) \leq 1 \wedge c_2 \Big\{ \bar{\vars}^2
  \left( \frac{\dimone \log \dimone}{N} \right)^{3/4} + \vartheta(N, \dimone \vee \dimtwo, \dimone \land \dimtwo, \vars) \Big\}
\end{align}
\end{subequations}
with probability at least $1- 3(\dimone \dimtwo)^{-2}$.
\end{theorem}
As we have discussed, the above guarantees for the worst-case topological sort $\pihatref$ is also valid for the average-case topological sort $\pihatref'$. 
Comparing the bounds~\eqref{eq:rowlb-lower1} and~\eqref{eq:r-upper},
we see that the estimator $\Mhat(\pihatref' , \id)$
is the optimal row-PDD estimator for the class
$\Cpermr$ (up to poly-logarithmic factors) in the metric $\mathcal{R}(M^*,
\pihat)$. If conjecture~\eqref{eq:conj-row} holds, then this would also be true
unconditionally.

In order to evaluate the Frobenius error guarantee, it is helpful to specialize to the regime $\dimtwo \leq \dimone \leq C \dimtwo$.

\paragraph{Example: $\dimtwo \leq \dimone \leq C\dimtwo$}
In this case, the Frobenius error guarantee simplifies to
\begin{align} \label{eq:r-upper-simple}
\frac{1}{(\log \dimone)^2} \mathcal{F}(M^*, \Mhat(\pihatref, \id)) \lesssim
\begin{cases}
1 & \text{ for } N \le \dimone \\
\left( \frac{\dimone}{N} \right)^{3/4} & \text{ for } \dimone \le N \le \dimone^3 \\
\frac{1}{\sqrt{N} } & \text{ for } \dimone^3 \le N \le \dimone^4 \\
\frac{\dimone \dimtwo}{N} & \text{ for } \dimone^4 \leq N \le \dimone^5 \\
\left( \frac{\dimone}{N} \right)^{3/4} & \text{ for } N \ge \dimone^5.
\end{cases} 
\end{align}
We may compare the bounds~\eqref{eq:lower-simple-2} and~\eqref{eq:r-upper-simple}; note that when $\dimone^3 \leq N \leq \dimone^5$, our estimator achieves the minimax lower bound given by $\frac{1}{ \sqrt{N} } \land \frac{ \dimone \dimtwo }{ N }$ (up to poly-logarithmic factors).
As alluded to before, when $N \geq \dimone^5$, we may switch to trivially outputting the matrix $Y$, and so the sub-optimality in the large $N$ regime can be completely avoided.
%
On the other hand, no such modification can be made in the small sample regime $\dimone \le N \le \dimone^3$, and the estimator falls short of
being optimal in the Frobenius error in this case.  \hfill $\clubsuit$ 
\smallskip

Closing the aforementioned gap in the
Frobenius error in the small sample regime is an interesting
open problem. Having established guarantees for our algorithm, we now turn to using the intuition gained from these guarantees to provide estimators for matrices in the larger class $\Cperm$. 

\subsubsection{Two-dimensional sorting for class $\Cperm$} \label{sec:tds}

We reinstate the assumption $\dimtwo \leq \dimone$ in this section. The algorithm in the previous section cannot be immediately extended to the class $\Cperm$, since it assumes that the matrix is perfectly sorted along one of the dimensions. However, it suggests a plug-in procedure that can be described informally as follows.
\begin{enumerate}
\item Sort the columns of the matrix $Y$ according to its column sums.
\item Apply Algorithm 2 to the column-sorted matrix to obtain a row permutation estimate. 
\item Repeat Steps 1 and 2 with $Y$ transposed to obtain a column permutation estimate.
\end{enumerate}
Although the columns of $Y$ are only approximately sorted in the first
step, the hope is that the finer row-wise control given by Algorithm~2
is able to improve the row permutation estimate. The actual algorithm, provided
below, essentially implements this intuition, but with a careful data-dependent blocking procedure that we describe next.
Given a data matrix $Y \in \real^{\dimone \times \dimtwo}$, the
following blocking subroutine returns a column partition $\BL(Y)$. 


\paragraph{Subroutine 1 (blocking)}
\begin{itemize}
\item Step 1: Compute the column sums $\{C(j)\}_{j = 1}^{\dimtwo}$ of
  the matrix $Y$ as
\begin{align*}
C(j) = \sum_{i=1}^{\dimone} Y_{i,j}.
\end{align*} 
Let $\sigmahatpre$ be a permutation along which the sequence
$\{C(\sigmahatpre(j))\}_{j=1}^{\dimtwo}$ is non-decreasing.

\item Step 2: Set $\tau = 16 \varplusone \Big( \sqrt{\frac{\dimone^2
    \dimtwo}{N} \log(\dimone \dimtwo) } + \frac{\dimone \dimtwo}{N}
  \log(\dimone \dimtwo) \Big)$ and $K = \lceil \dimtwo/\tau
  \rceil$. Partition the columns of $Y$ into $K$ blocks by defining
\begin{align*}
\bl_1 &= \{j \in [\dimtwo]: C(j) \in (-\infty, \tau) \}, \\ \bl_k &=
\left\{j \in [\dimtwo] : C(j) \in \big[ (k - 1) \tau , k \tau \big)
  \right\} \text{ for } 1<k<K, \text{ and} \\ \bl_K &= \{j \in
           [\dimtwo]: C(j) \in [(K-1)\tau, \infty)\}.
\end{align*}
Note that each block is contiguous when the columns are permuted by
$\sigmahatpre$.

\item Step 3 (aggregation): Set $\beta = \blocksize$. Call a block
  $\bl_k$ ``large'' if $|\bl_k| \geq \beta$ and ``small"
  otherwise. Aggregate small blocks in $\bl$ while leaving the large
  blocks as they are, to obtain the final partition $\BL$.

More precisely, consider the matrix $Y' = Y(\id, \sigmahatpre)$ having
non-decreasing column sums and contiguous blocks. Call two small blocks
``adjacent'' if there is no other small block between them.  Take
unions of adjacent small blocks to ensure that the size of each
resulting block is in the range $[ \frac{1}{2} \beta, 2 \beta]$. If
the union of all small blocks is smaller than $\frac{1}{2} \beta$,
aggregate them all.

Return the resulting partition $\BL(Y) = \BL$.
\end{itemize}

Ignoring Step~3 for the moment, we see that the blocking $\bl$ is analogous to the blocking $\blref$ of Algorithm~2, along which
partial row sums may be computed. While the blocking $\blref$ was chosen in a data-independent manner due to the columns being sorted exactly, the blocking $\bl$ is chosen based on approximate estimation of the column permutation. 
However, some of these $K$ blocks may be too small,
resulting in noisy partial sums; in order to mitigate this issue,
Step~3 aggregates small blocks into large enough ones.  We are now in
a position to describe the two-dimensional sorting algorithm.

\paragraph{Algorithm 3 (two-dimensional sorting)}
\begin{itemize}
\item Step 0: Split the observations into two independent sub-samples
  of equal size, and form the corresponding matrices $Y^{(1)}$ and
  $Y^{(2)}$ according to equation~\eqref{eq:obs-Y}.

\item Step 1: Apply Subroutine 1 to the matrix $Y^{(1)}$ to obtain a
  partition $\BL = \BL(Y^{(1)})$ of the columns. Let $K$ be the number
  of blocks in $\BL$.

\item Step 2: Using the second sample $Y^{(2)}$, compute the row sums
\begin{align*}
S(i) = \sum_{j \in [\dimtwo]} Y^{(2)}_{i,j} \text{ for each }i \in
[\dimone],
\end{align*} 
and the partial row sums within each block 
\begin{align*}
S_{\BL_k}(i) = \sum_{j \in \BL_k} Y^{(2)}_{i,j} \text{ for each }i \in
[\dimone], k \in [K].
\end{align*}
Create a directed graph $G$ with vertex set $[\dimone]$, where an edge
$u \to v$ is present if either
\begin{subequations}
\begin{align}
S(v) - S(u) & > 16 \varplusone \bigg( \sqrt{\frac{\dimone
    \dimtwo^2}{N} \log(\dimone \dimtwo) } + \frac{\dimone \dimtwo}{N}
\log(\dimone \dimtwo) \bigg), \text{ or} \label{eq:full-sum}
\\ S_{\BL_k}(v) - S_{\BL_k}(u) & > 16 \varplusone \bigg(
\sqrt{\frac{\dimone \dimtwo}{N} |\BL_k| \log(\dimone \dimtwo) } +
\frac{\dimone \dimtwo}{N} \log(\dimone \dimtwo) \bigg)  \label{eq:block-sum} \\
& \qquad \qquad \qquad \qquad \qquad \qquad \qquad \text{ for some
} k \in [K].  \notag 
\end{align}
\end{subequations}

\item Step 3: Compute a topological sort $\pihattds$ of the graph $G$;
  if none exists, set $\pihattds = \id$.

\item Step 4: Repeat Steps 1--3 with $(Y^{(i)})^\top$ replacing
  $Y^{(i)}$ for $i=1,2$, the roles of $\dimone$ and $\dimtwo$
  switched, and the roles of $\pi$ and $\sigma$ switched, to compute
  the permutation estimate $\sigmahattds$.

\item Step 5: Return the permutation estimates $(\pihattds,
  \sigmahatftds)$.
\end{itemize}


The topological sorting step once again takes time $\order(\dimone^2)$ and reading the matrix takes time $\dimone \dimtwo$. Consequently, since $\dimone \geq \dimtwo$, the construction of the graph $G$ in Step~2 dominates the
computational complexity, and takes time $\order(\dimone^2 \dimtwo /
\beta) = \order(\dimone^2 \dimtwo^{1/2})$. Computing judiciously chosen partial row sums once again captures much more of the signal in the problem than entire row sums alone, and we obtain the following
guarantee.

\begin{theorem}
  \label{thm:fast-tds}
Suppose that $\dimtwo \leq \dimone$. Then there exists an absolute constant $c_1 > 0$ such that for any matrix $M^* \in \Cperm$, we have
\begin{subequations}
\begin{align}
\mathcal{R}(M^*, \pihattds) \leq 1 \wedge c_1 \bar{\vars}^2 \Big(\frac{\dimone
  \log \dimone}{N} \Big)^{1/2} \label{eq:rowtds}
\end{align}
with probability exceeding $1- 9(\dimone \dimtwo)^{-3}$. 
Moreover, there exists an absolute constant $c_2 > 0$ such that 
\begin{align} \label{eq:frobtds}
\mathcal{F}(M^*, \Mhat(\pihattds, \sigmahattds)) \leq 1 \wedge c_2 \Big\{ \bar{\vars}^2 \Big(\frac{\dimone \log \dimone}{N} \Big)^{3/4} + \vartheta(N, \dimone, \dimtwo, \vars) \Big\}
\end{align}
\end{subequations}
with probability exceeding $1- 10(\dimone \dimtwo)^{-3}$.
\end{theorem}
Comparing the bounds~\eqref{eq:rowtds} and~\eqref{eq:r-upper2}, we see that our polynomial-time estimator is minimax-optimal (up to a poly-logarithmic factor) in the max-row-norm metric, but this was
already achieved by other estimators in the
literature~\cite{ShaBalWai16-2, ChaMuk16, PanMaoMutWaiCou17}. 

To interpret our rate in the Frobenius error, 
it is once again helpful to specialize to the case $\dimtwo \le \dimone \le C \dimtwo$. 

%
\paragraph{Example: $\dimtwo \leq \dimone \leq C\dimtwo$}
In this case, the Frobenius error guarantee simplifies exactly as before to
\begin{align} \label{eq:r-upper-simple-2}
\frac{1}{(\log \dimone)^2} \mathcal{F}(M^*, \Mhat(\pihattds, \sigmahattds)) \lesssim
\begin{cases}
1 & \text{ for } N \le \dimone \\
\left( \frac{\dimone}{N} \right)^{3/4} & \text{ for } \dimone \le N \le \dimone^3 \\
\frac{1}{\sqrt{N} } & \text{ for } \dimone^3 \le N \le \dimone^4 \\
\frac{\dimone \dimtwo}{N} & \text{ for } \dimone^4 \leq N \le \dimone^5 \\
\left( \frac{\dimone}{N} \right)^{3/4} & \text{ for } N \ge \dimone^5.
\end{cases} 
\end{align}
Once again, comparing the bounds~\eqref{eq:lower-simple-2} and~\eqref{eq:r-upper-simple-2}, we see that when $N \geq \dimone^3$, our estimator (when combined with outputting the $Y$ matrix when \mbox{$N \geq \dimone^5$}) achieves the minimax lower bound up to poly-logarithmic factors.

This optimality is particularly notable because existing estimators~\cite{ChaMuk16, ShaBalGunWai17, ShaBalWai16-2, Cha15, PanMaoMutWaiCou17} for the class $\Cperm$ are only able to attain, in the regime $\dimone = \dimtwo$, the rate
\begin{align} \label{eq:r-upper-simple-3}
\frac{1}{(\log \dimone)^2} \mathcal{F}(M^*, \Mhat_{\mathsf{prior}}) \lesssim
\begin{cases}
1 & \text{ for } N \le \dimone \\
\left( \frac{\dimone}{N} \right)^{1/2} & \text{ for } N \ge \dimone ,
\end{cases} 
\end{align}
where we have used $\Mhat_{\mathsf{prior}}$ to indicator any such estimator from prior work. Thus, the non-parametric rates observed when $N$ is large are completely washed out by 
the rate $\left( \frac{\dimone}{N} \right)^{1/2}$, and so this prevents existing estimators from achieving minimax optimality
in \emph{any} regime of $N$.
%


Returning to our estimator, we see that in the regime $\dimone \leq N \leq \dimone^2$, it falls short of being minimax-optimal, but breaks the conjectured
statistical-computational barrier alluded to in the introduction. \hfill $\clubsuit$ 

A corollary of Theorem~\ref{thm:fast-tds} in the regime $\dimone \leq N \leq \dimone \dimtwo$ was announced for a slight variant of our estimator in Theorem 1 of the abstract~\cite{mao18breakingcolt}; the proof of Theorem~\ref{thm:fast-tds} hence contains the first proof of~\cite[Theorem 1]{mao18breakingcolt}.
%


Note that Theorem~\ref{thm:fast-tds} extends to estimation of
matrices in the class $\csst(n)$. In particular, we have $\dimone =
\dimtwo = n$, and either of the two estimates $\pihattds$ or
$\sigmahattds$ may be returned as an estimate of the permutation
$\pi$ while preserving the same guarantees. 

We conclude by noting that the sub-optimality of our estimator in the small-sample regime is not due to a weakness in the analysis.
In particular, our analysis in this regime is also optimal up to poly-logarithmic factors; the rate $(\frac{\dimone}{N} )^{3/4}$ is indeed the rate attained by the two-dimensional sorting algorithm for the noisy sorting subclass of $\Cperm$. In fact, a variant of this algorithm was used in a recursive fashion to successively improve the rate for noisy sorting matrices~\cite{MaoWeeRig17}; the first step of this algorithm generates an estimate with rate exactly $(\frac{\dimone}{N} )^{3/4}$.

This concludes our discussion of the main results and their consequences.
Proofs of these results can be found in Appendix~\ref{sec:proofs} of the supplementary material. We conclude the main paper in the next section with 
a short discussion.






\section{Discussion}

We have studied the class of permutation-based models in two distinct
metrics. A notable consequence of our results is that our
polynomial-time algorithms are able to achieve the minimax lower bound
in the Frobenius error up to a poly-logarithmic factor provided the
sample size grows to be large. Moreover, we have overcome a crucial
bottleneck in previous analyses that underlay a
statistical-computational gap. Several intriguing questions related to
estimating such matrices remain: \\

(1) What is the fastest Frobenius error rate achievable by
computationally efficient estimators in the partially observed setting
when $N$ is small? \\ (2) Can the techniques from here be used to
narrow statistical-computational gaps in other permutation-based
models~\cite{ShaBalWai16,FlaMaoRig16,PanWaiCou17}?

As a partial answer to the first question, it can be shown that when
the informal algorithm described at the beginning of
Section~\ref{sec:tds} is recursed in the natural way and applied to
the noisy sorting subclass of the SST model, it yields another
minimax-optimal estimator for noisy sorting, similar to the multistage
algorithm of Mao et al.~\cite{MaoWeeRig17}.  However, this same
guarantee is preserved for neither the larger class of matrices
$\Cperm$, nor for its sub-class $\Cpermr$.  Improving the rate will
likely require techniques that are beyond the reach of those
introduced in this paper.

It is also worth noting that model~\eqref{eq:model} allowed us to
perform sample-splitting in Algorithm 3 to preserve independence
across observations, so that Step 2 is carried out on a sample that is
independent of the blocking generated in Step 1.  Thus, our proofs
also hold for the observation model where we have \emph{exactly} $2$
independent samples per entry of the matrix.

It is natural to wonder if just one independent sample per entry
suffices, and whether sample splitting is required at all. Reasoning
heuristically, one way to handle the dependence between the two steps
is to prove a union bound over exponentially many possible
realizations of the blocking; unfortunately, this fails since the
desired concentration of partial row sums fails with polynomially
small probability.  Thus, addressing the original sampling
model~\cite{Cha15, ShaBalGunWai17} (with one sample per entry) presents an interesting technical
challenge that may also involve its own statistical-computational
trade-offs~\cite{Mon15}.






\section*{Acknowledgments}

CM thanks Philippe Rigollet for helpful discussions. The work of CM
was supported in part by grants NSF CAREER DMS-1541099, NSF DMS-1541100 and
ONR N00014-16-S-BA10, and the work of AP and MJW was supported in part
by grants NSF-DMS-1612948 and DOD ONR-N00014. We thank Jingyan Wang for pointing out an error in an earlier version of the paper, and thank anonymous reviewers for their helpful comments. 


\appendix




\section{Proofs of main results} \label{sec:proofs}

In this appendix, we present all of our proofs. We begin
by stating and proving technical lemmas that are used repeatedly
in the proofs. We then prove our main results in the order in which they
were stated.

Throughout the proofs, we
assume without loss of generality that $M^* \in \Cbiso(\id,\id) =
\Cbiso$.  Because we are interested in rates of estimation up to
universal constants, we assume that each independent sub-sample
contains $N' = \Poi(N)$ observations (instead of $\Poi(N)/2$). We use the shorthand $Y = Y \big(\{(X_\ell,
y_\ell)\}_{\ell=1}^{N'} \big)$, throughout.

\subsection{Preliminary lemmas}

Before proceeding to the proofs of our main results, we
provide two lemmas that underlie many of our arguments.  
The proofs of these lemmas can be found in Sections \ref{sec:lem1} and \ref{sec:lem2}, respectively. 
Let us denote the $\ell_\infty$ norm of $A$ by $\|A\|_\infty = \max_{i \in [\dimone] , \, j \in [\dimtwo] } |A_{i, j}|$. 

The first lemma establishes concentration of a linear form of observations $Y_{i, j}$ 
around its mean.
\begin{lemma}
\label{lem:par-sum}
For any fixed matrix $A \in \real^{\dimone \times \dimtwo}$ and scalar $u \geq 0$, and under our observation model~\eqref{eq:model}, we have 
\begin{align*}
\left| \langle Y - M^*, A \rangle \right| 
\le 2 \varplusone \Big( \sqrt{e-1} \, \|A\|_F \sqrt{ \frac{ \dimone \dimtwo }{ N } u } 
+ \|A\|_\infty \frac{ \dimone \dimtwo }{ N } u \Big) 
\end{align*}
with probability at least $1 - 4 e^{-u}$. 

Consequently, for any nonempty subset $\mathcal{S} \subset [\dimone] \times
[\dimtwo]$, it holds that
\begin{align*}
\bigg| \sum_{(i,j) \in \mathcal{S}} (Y_{i,j} - M^*_{i,j})
\bigg| \le 8 \varplusone \bigg( \sqrt{\frac{|\mathcal{S}| \dimone
\dimtwo}{N} \log(\dimone \dimtwo) } + \frac{\dimone \dimtwo}{N}
\log(\dimone \dimtwo) \bigg) 
\end{align*}
with probability at least $1 - 4 (\dimone \dimtwo)^{-4}$. 
\end{lemma}

The next lemma generalizes Theorem~5 of Shah et
al.~\cite{ShaBalGunWai17} to any model in which the noise satisfies a ``mixed tail" assumption. More precisely, a random matrix $W \in \real^{\dimone \times \dimtwo}$ is said to have an $(\alpha, \beta)$-mixed tail if 
there exist (possibly $(\dimone, \dimtwo)$-dependent) positive scalars $\alpha$ and $\beta \le \dimone^2$ such that for any fixed matrix $A \in \real^{\dimone \times \dimtwo}$ and $u \ge 0$, we have 
\begin{align} \label{eq:mixed-tail}
\Pr \Big\{ \big| \langle W, A \rangle \big| 
\ge \big( \alpha \|A\|_F \sqrt{ u } 
+ \beta \|A\|_\infty u \big) \Big\} \le 2 e^{-u}. 
\end{align}
It is worth
mentioning that similar (but less general) lemmas characterizing the estimation error for 
a bivariate isotonic matrix were also proved
in prior work~\cite{ChaGunSen18,ChaMuk16}.

\begin{lemma}
\label{lem:shah}
Let $\dimtwo \leq \dimone$, and consider the
observation model $Y = M^* + W$. Suppose that the noise matrix
$W$ satisfies the $(\alpha, \beta)$-mixed tail condition~\eqref{eq:mixed-tail}. \\
(a) There is an absolute constant $c$ such that for all $M^* \in \Cperm$, the least squares estimator $\Mhatls(\Cperm, Y)$ satisfies 
\begin{align*}
&\big\| \Mhatls(\Cperm, Y) - M^* \big\|_F^2 \leq c \Bigg\{ \alpha^2 \dimone \log \dimone + \beta \dimtwo (\log \dimone)^2 +  \beta \dimone \log \dimone \\
& \quad + \Big[ \alpha \sqrt{ \dimone \dimtwo } (\log \dimone )^2 \Big] \land  
\Big[ \alpha^2 \dimone \dimtwo \log (\dimone/\alpha + e)\Big] 
\land \Big[ \alpha^{4/3} \dimone^{1/3} \dimtwo (\log \dimone)^{2/3} \Big] \Bigg\} 
\end{align*}
with probability at least $1 - \dimone^{- 3 \dimone}$. 

\noindent (b) There is an absolute constant $c$ such that for all $M^* \in \Cbiso$, the least squares estimator $\Mhatls(\Cbiso, Y)$ satisfies  
\begin{align*}
&\big\| \Mhatls(\Cbiso, Y) - M^* \big\|_F^2 \leq c \Bigg\{ \beta \dimtwo (\log \dimone)^2 \\
& \quad + \Big[ \alpha \sqrt{ \dimone \dimtwo } (\log \dimone )^2 \Big] \land  
\Big[ \alpha^2 \dimone \dimtwo \log (\dimone/\alpha + e)\Big] 
\land \Big[ \alpha^{4/3} \dimone^{1/3} \dimtwo (\log \dimone)^{2/3} \Big] \Bigg\}
\end{align*}
with probability at least $1 - \dimone^{- 3 \dimone}$.
\end{lemma}


\subsubsection{Proof of Lemma \ref{lem:par-sum}}
\label{sec:lem1}

Recall that the observation matrix $Y$ is defined by 
\begin{align*}
Y_{i,j} = \frac{\dimone \dimtwo}{ N } \sum_{\ell = 1}^{N'} y_\ell \, \bfone \{ X_\ell = E^{(i,j)} \} . 
\end{align*}
Let $\{ I_{i, j} \}_{i \in [\dimone], \, j \in [\dimtwo]}$ be the partition of $[N']$ defined so that $\ell \in I_{i, j}$ if and only if $X_\ell = E^{(i, j)}$. 
In other words, the observation $y_\ell$ is a noisy version of entry $M^*_{i, j}$ for each $\ell \in I_{i, j}$. 
Then, we have 
\begin{align*}
Y_{i, j} = \frac{\dimone \dimtwo}{ N } \sum_{\ell \in I_{i, j}} y_\ell . 
\end{align*} 
By Poissonization $N' \sim \Poi (N)$, we know that the random variables $|I_{i, j}|$ are i.i.d. $\Poi ( \frac{N}{\dimone \dimtwo} )$, and that the quantities $Y_{i, j}$ are independent for $(i, j) \in [\dimone] \times [\dimtwo]$. 
Moreover, we may write 
\begin{align*}
Y_{i,j} - M^*_{i,j} &= \frac{\dimone \dimtwo}{ N } \sum_{\ell \in I_{i, j} } (y_\ell - M^*_{i,j}) 
+ M^*_{i,j} \frac{\dimone \dimtwo}{ N } \Big( |I_{i, j} | - \frac{N}{\dimone \dimtwo} \Big) . 
\end{align*} 

As a result, it holds that 
\[
\langle Y - M^*, A \rangle = Z_1 + Z_2,
\]
where we define 
\begin{align*}
Z_1 &= \sum_{ \substack{ i \in [\dimone] \\ j \in [\dimtwo] } }  A_{i, j} \frac{\dimone \dimtwo}{ N } \sum_{\ell \in I_{i, j} } (y_\ell - M^*_{i,j})   \quad \text{ and} \\
Z_2 &= \sum_{ \substack{ i \in [\dimone] \\ j \in [\dimtwo] } }  A_{i, j} M^*_{i,j} \frac{\dimone \dimtwo}{ N } \Big( |I_{i, j} | - \frac{N}{\dimone \dimtwo} \Big) . 
\end{align*}
We now control $Z_1$ and $Z_2$ separately. 

First, for $s \in \real$, we compute 
\begin{align}
\E \exp ( s Z_1 ) &= \E_{I_{i, j}} \bigg[ \E_{y_{\ell}} \bigg[ \exp \bigg( s \sum_{i, j} A_{i, j} \frac{\dimone \dimtwo}{ N } \sum_{\ell \in I_{i, j} } (y_\ell - M^*_{i,j}) \bigg) \, \Big| \, I_{i, j} \bigg] \bigg] \notag \\
&= \E_{I_{i, j}} \bigg[ \prod_{i, j} \prod_{\ell \in I_{i, j} }  \E_{y_{\ell}} \bigg[ \exp \bigg( s  \frac{\dimone \dimtwo}{ N } A_{i, j} (y_\ell - M^*_{i,j}) \bigg) \, \Big| \, I_{i, j} \bigg] \bigg] \notag \\
&\le \E_{I_{i, j}} \bigg[ \prod_{i, j} \prod_{\ell \in I_{i, j} }  \exp \bigg( \Big( \vars s  \frac{\dimone \dimtwo}{ N } A_{i, j} \Big)^2  \bigg) \bigg] \; \; \text{ if } \;\; |s| \le \frac{ N }{ \dimone \dimtwo \|A\|_\infty \vars}, \label{eq:sub-exp-1}
\end{align}
where inequality\eqref{eq:sub-exp-1} holds 
since the quantity $y_\ell - M^*_{i, j}$ is sub-exponential with parameter $\zeta$. 
It then follows that for $|s| \le \frac{ N }{ \dimone \dimtwo \|A\|_\infty \vars}$, we have
\begin{align*}
\E \exp ( s Z_1 ) &\le \E \bigg[ \prod_{i, j} \exp \bigg( \Big( \vars s  \frac{\dimone \dimtwo}{ N } A_{i, j} \Big)^2 |I_{i, j}| \bigg) \bigg] \\
&= \prod_{i, j} \E \bigg[ \exp \bigg( \Big( \vars s  \frac{\dimone \dimtwo}{ N } A_{i, j} \Big)^2 |I_{i, j}| \bigg) \bigg] \\
&\stackrel{\1}{=} \prod_{i, j} \exp \bigg\{ \frac{N}{\dimone \dimtwo} \bigg[ \exp \bigg( \Big( \vars s  \frac{\dimone \dimtwo}{ N } A_{i, j} \Big)^2 \bigg) - 1 \bigg] \bigg\},
\end{align*}
where step $\1$ follows by explicit computation since $|I_{i, j}|$ are i.i.d. $\Poi ( \frac{N}{\dimone \dimtwo} )$ random variables. 
We may now apply the inequality $e^x - 1 \le (e-1) x$ valid for $x \in [0, 1]$, with the substitution 
\[
x = \Big( \vars s  \frac{\dimone \dimtwo}{ N } A_{i, j} \Big)^2;
\]
note that this lies in the range $[0, 1]$ for all $|s| \le \frac{ N }{ \dimone \dimtwo \|A\|_\infty \vars} $. Thus, for $s$ in this range, we have
\begin{align*}
\E \exp ( s Z_1 ) &\le \prod_{i, j} \exp \Big\{ (e-1) \frac{\dimone \dimtwo}{ N } \big( \vars s  A_{i, j} \big)^2 \Big\} \\
&= \exp \Big\{ (e-1) \frac{\dimone \dimtwo}{ N } \vars^2 s^2 \|A\|_F^2 \Big\} . 
\end{align*}
Using the Chernoff bound, we obtain for $|s| \le \frac{ N }{ \dimone \dimtwo \|A\|_\infty \vars} $ and all $t \geq 0$ the tail bound 
\begin{align*}
\Pr \{ Z_1 \ge t \} \le e^{-st} \, \E \exp ( s Z_1 ) \le \exp \Big\{ (e-1) \frac{\dimone \dimtwo}{ N } \vars^2 s^2 \|A\|_F^2 - s t \Big\} .
\end{align*}
The optimal choice $s = \frac{ N t }{ 2 (e - 1) \dimone \dimtwo \vars^2 \|A\|_F^2 } \land \frac{ N }{ \dimone \dimtwo \|A\|_\infty \vars}$ yields the bound
\begin{align}
\Pr \{ Z_1 \ge t \} \le \exp \bigg( - \frac{ N }{ \dimone \dimtwo } \Big( \frac{ t^2 }{ 4 (e - 1) \vars^2 \|A\|_F^2 } \land \frac{ t }{ 2 \vars \|A\|_\infty } \Big) \bigg) . 
\label{eq:z1}
\end{align}
The lower tail bound is obtained analogously. 

Let us now turn to the noise term $Z_2$.
For $s \in \real$, we have 
\begin{align*}
\E \exp ( s Z_2 ) &= \E \bigg[ \exp \bigg( s \sum_{i, j}  A_{i, j} M^*_{i,j} \frac{\dimone \dimtwo}{ N } \Big( |I_{i, j} | - \frac{N}{\dimone \dimtwo} \Big) \bigg) \bigg] \\
&= \prod_{i, j}  \E \bigg[ \exp \bigg( s  A_{i, j} M^*_{i,j} \frac{\dimone \dimtwo}{ N } \Big( |I_{i, j} | - \frac{N}{\dimone \dimtwo} \Big) \bigg) \bigg] \\
&= \prod_{i, j}  \exp \bigg\{ \frac{N}{\dimone \dimtwo} \bigg[ \exp \Big( s  A_{i, j} M^*_{i,j} \frac{\dimone \dimtwo}{ N } \Big) - s  A_{i, j} M^*_{i,j} \frac{\dimone \dimtwo}{ N } - 1 \bigg] \bigg\},
\end{align*}
where the last step uses the explicit MGF of $|I_{i, j}|$.
We may now apply the inequality $e^x - x - 1 \le (e-2) x^2$ valid for $x \in [0, 1]$, with the substitution 
\[
x = s  A_{i, j} M^*_{i,j} \frac{\dimone \dimtwo}{ N };
\]
note that this quantity is in the range $[0, 1]$
%
provided $|s| \le \frac{ N }{ \dimone \dimtwo \|A\|_\infty }$. 

Thus, for all $s$ in this range, we have 
\begin{align*}
\E \exp ( s Z_2 ) &\le \prod_{i, j}  \exp \Big\{ (e-2) \frac{\dimone \dimtwo}{ N } \big( s  A_{i, j} M^*_{i,j} \big)^2 \Big\} \\
&\le \exp \Big\{ (e-2) \frac{\dimone \dimtwo}{ N } s^2 \|A\|_F^2 \Big\} . 
\end{align*}
A similar Chernoff bound argument then yields, for each $t \geq 0$, the bound
\begin{align}
\Pr \{ Z_2 \ge t \} \le \exp \bigg( - \frac{ N }{ \dimone \dimtwo } \Big( \frac{ t^2 }{ 4 (e - 2) \|A\|_F^2 } \land \frac{ t }{ 2 \|A\|_\infty } \Big) \bigg),
\label{eq:z2}
\end{align}
and the lower tail bound holds analogously. 

Combining the tail bounds \eqref{eq:z1} on $Z_1$ and \eqref{eq:z2} on $Z_2$, we obtain
\begin{align*}
\big| \langle Y - M^*, A \rangle \big| \leq |Z_1| + |Z_2| 
\le 2 \varplusone \Big( \sqrt{e-1} \|A\|_F \sqrt{ \frac{ \dimone \dimtwo }{ N } u } 
+ \|A\|_\infty \frac{ \dimone \dimtwo }{ N } u \Big) 
\end{align*}
with probability at least $1 - 4 e^{-u}$, for each $u \geq 0$. 

The second consequence of the lemma follows immediately from the first assertion by taking $A$ to be the indicator of the subset $\mathcal{S}$, and some algebraic manipulation. \qed

\subsubsection{Proof of Lemma \ref{lem:shah}}
\label{sec:lem2}

We first state several lemmas that will be used in the proof. 
The following variational formula due to Chatterjee~\cite{Cha14} is convenient for controlling the performance of a least squares estimator. 

\begin{lemma}[Chatterjee's variational formula] \label{lem:variational}
Let $\cC$ be a closed subset of ${\real}^{\dimone \times \dimtwo}$. 
Suppose that $Y = M^* + W$ where $M^* \in \cC$ and  $W \in {\real}^{\dimone \times \dimtwo}$. 
Let $\Mhatls (\cC, Y)$ denote the least squares estimator, that is, the projection of $Y$ onto $\cC$.  
Define a function $f_{M^*}: {\real}_+ \to {\real}$ by
$$
f_{M^*}(t) = \sup_{ \substack{ M \in \cC \\ \|M - M^*\|_2 \le t } } \langle W, M - M^* \rangle - \frac{t^2}2 .
$$
If there exists $t^* > 0$ such that $f_{M*}(t) < 0$ for all $t \ge t^*$, then we have
$\|\Mhatls (\cC, Y) - M^*\|_2 \le t^*.$
\end{lemma}
%
This deterministic form is proved in \cite[Lemma 6.1]{FlaMaoRig16}. 
Here we simply state the result in matrix form for convenient application. 

The following chaining tail bound due to Dirksen~\cite[Theorem 3.5]{Dir15} is tailored for bounding the supremum of an empirical process with a mixed tail. We state a version specialized to our setup. For each positive scalar $c$, let $c \|\cdot\|$ denote the norm $\|\cdot\|$ scaled by $c$.
Let $\gamma_p (\cC, \|\cdot\|)$ and $\diam (\cC, \|\cdot\|)$ denote  Talagrand's $\gamma_p$ functional (see, e.g.,~\cite{Dir15}) and the diameter of the set $\cC$ in the distance induced by the norm $\|\cdot\|$, respectively.
\begin{lemma}[Generic chaining tail bounds]
\label{lem:mixed-tail}
Let $W$ be a random matrix in $\real^{\dimone \times \dimtwo}$ satisfying the $(\alpha, \beta)$-mixed tail condition \eqref{eq:mixed-tail}. 
Let $\cC$ be a subset of $\real^{\dimone \times \dimtwo}$ and let $M^* \in \cC$. 
Then there exists a universal positive constant $c$ such that for any $u \ge 1$, we have
\begin{align*}
&\Pr \Big\{ \sup_{ M \in \cC } \big| \langle W, M - M^* \rangle \big| 
\ge c \Big( \gamma_2 (\cC, \alpha \| \cdot \|_F ) + \gamma_1 (\cC, \beta \| \cdot \|_\infty) \\
& \qquad \qquad \qquad \qquad 
+ \sqrt{u} \diam (\cC, \alpha \| \cdot \|_F )  + u \diam (\cC, \beta \| \cdot \|_\infty) \Big) \Big\} \le e^{-u}.
\end{align*}
\end{lemma}
%

The final lemma bounds the metric entropy of the set $\Cperm$ in the $\ell_2$ or $\ell_\infty$ norm. The $\ell_2$ entropy bound is known \cite{ShaBalGunWai17, ChaGunSen18}. For $\epsilon > 0$ and a set $\cC$ equipped with a norm $\| \cdot \|$, let $N( \epsilon, \cC, \|\cdot \| )$ denote the $\epsilon$-metric entropy of $\cC$ in the norm $\|\cdot\|$. 

\begin{lemma}
\label{lem:entropy}
There is an absolute positive constant $c$ such that for any $\epsilon \leq \sqrt{\dimone \dimtwo}$, we have
\begin{subequations}
\begin{align}
&\log N ( \epsilon, \Cperm , \| \cdot \|_F ) \le 2 \dimone \log \dimone +  \label{eq:l2-me}  \\
& \Big[ c \frac{\dimone \dimtwo}{\epsilon^2} \Big( \log \frac{ \sqrt{ \dimone \dimtwo} }{\epsilon} \Big)^2 \Big] \land \Big( c \, \dimone \dimtwo \log \frac{ \sqrt{ \dimone \dimtwo}  }{\epsilon} \Big) \land \Big( c \frac{ \sqrt{ \dimone \dimtwo } }{ \epsilon } \dimtwo \log \dimone \Big); \notag \\
&\log N ( \epsilon, \Cbiso , \| \cdot \|_F ) \le \label{eq:l2-me-biso} \\
& \Big[ c \frac{\dimone \dimtwo}{\epsilon^2} \Big( \log \frac{ \sqrt{ \dimone \dimtwo} }{\epsilon} \Big)^2 \Big] \land \Big( c \, \dimone \dimtwo \log \frac{ \sqrt{ \dimone \dimtwo}  }{\epsilon} \Big) \land \Big( c \frac{ \sqrt{ \dimone \dimtwo } }{ \epsilon } \dimtwo \log \dimone \Big),  \notag 
\end{align}
\end{subequations}
and the metric entropies are zero if $\epsilon > \sqrt{ \dimone \dimtwo }$. 
We also have, for each $\epsilon \leq 1$, the bounds 
\begin{subequations}
\begin{align}
\log N ( \epsilon, \Cperm , \| \cdot \|_\infty ) &\le \Big[ \frac{\dimtwo }{\epsilon} \log (e \dimone) \Big] \land \Big( \dimone \dimtwo \log{ \frac{e}{\epsilon} } \Big) + 2 \dimone \log \dimone; 
\label{eq:linf-me} \\
\log N ( \epsilon, \Cbiso , \| \cdot \|_\infty ) &\le \Big[ \frac{\dimtwo }{\epsilon} \log (e \dimone) \Big] \land \Big( \dimone \dimtwo \log{ \frac{e}{\epsilon} } \Big), \label{eq:linf-me-biso}
\end{align}
\end{subequations}
and the metric entropies are zero if $\epsilon > 1$. 
\end{lemma}
The proof of this lemma is provided at the end of the section.

Taking these lemmas as given, we are ready to prove Lemma \ref{lem:shah}. We only
provide the proof for the class $\Cperm$; the proof for the class $\Cbiso$ is analogous.
Let 
$\cB_{M^*}(t)$ denote the ball of radius $t$ in the Frobenius norm centered at $M^*$. 
To apply Lemma \ref{lem:variational}, we define 
$$
g(t) = \sup_{ M \in \Cperm \cap \cB_{M^*}(t) } \langle W, M - M^* \rangle 
\quad \text{ and } \quad 
f(t) = g(t) - \frac{t^2}2 . 
$$
The key is to bound this supremum $g(t)$. 
By the assumption on the noise matrix $W$, Lemma \ref{lem:mixed-tail} immediately implies that with probability $1 - e^{-u}$, we have
\begin{align}
g(t) &\lesssim \gamma_2 ( \Cperm \cap \cB_{M^*} (t), \alpha \| \cdot \|_F ) + \gamma_1 ( \Cperm \cap \cB_{M^*} (t), \beta \| \cdot \|_\infty) \label{eq:sup-emp} \\
& \quad + \sqrt{u} \diam ( \Cperm \cap \cB_{M^*} (t), \alpha \| \cdot \|_F )  + u \diam ( \Cperm \cap \cB_{M^*} (t), \beta \| \cdot \|_\infty)  .  \notag 
\end{align}

The diameters are, in turn, bounded as
\begin{align*}
& \diam ( \Cperm \cap \cB_{M^*} (t), \alpha \| \cdot \|_F ) \le \alpha \big( t \land \sqrt{\dimone \dimtwo} \, \big) 
\quad \text{ and} \\
& \diam ( \Cperm \cap \cB_{M^*} (t), \beta \| \cdot \|_\infty) \le \beta ( t \land 1 ) 
\end{align*}
since each entry of $M \in \Cperm$ is in $[0, 1]$. 

It remains to bound the $\gamma_1$ and $\gamma_2$ functionals of the set $\Cperm \cap \cB_{M^*} (t)$. 
These functionals can be bounded by entropy integral bound (see equation (2.3) of \cite{Dir15}) 
\[
\gamma_p ( \Cperm \cap \cB_{M^*} (t), \| \cdot \| ) \le c_p \int_0^\infty \big[ \log N \big( \epsilon , \Cperm \cap \cB_{M^*} (t) , \| \cdot \| \big) \big]^{1/p} d \epsilon,
\]
valid for a constant $c_p > 0$ and any norm $\| \cdot \|$. We use this bound for $p = 1$ and $p = 2$, 
and the metric entropy bounds established in Lemma~\ref{lem:entropy}. 

Let us begin by establishing a bound on the $\gamma_2$ functional by writing
\begin{align*}
&\gamma_2 ( \Cperm \cap \cB_{M^*} (t), \alpha \| \cdot \|_F ) 
\le c_2 \int_0^{\alpha (t \land \sqrt{\dimone \dimtwo}) } \big[ \log N \big( \epsilon , \Cperm \cap \cB_{M^*} (t) , \alpha \| \cdot \|_F \big) \big]^{1/2} d \epsilon.
\end{align*}
We now make two observations about the metric entropy on the RHS.
First, note that scaling the Frobenius norm by a factor $\alpha$ amounts to replacing $\epsilon$ by $\epsilon/\alpha$ in~\eqref{eq:l2-me}.  Second, notice that the metric entropy is expressed as a minimum of three terms; we provide a bound for each of the terms separately, and then obtain the final bound as a minimum of the three cases.


For the first term, notice that the minimum is never attained when $\epsilon \leq \alpha \dimone^{-5}$, so that we have
\begin{align*}
\gamma_2 ( \Cperm \cap \cB_{M^*} (t), \alpha \| \cdot \|_F ) 
&\lesssim \int_{\alpha \dimone^{-5}}^{\alpha (t \land \sqrt{\dimone \dimtwo}) } \bigg\{ \frac{\alpha^2 \dimone \dimtwo}{\epsilon^2} \Big( \log \frac{ \alpha \sqrt{ \dimone \dimtwo}  }{\epsilon} \Big)^2 + 2 \dimone \log \dimone \bigg\}^{1/2} d \epsilon \\
&\lesssim 
\int_{\alpha \dimone^{-5}}^{\alpha \sqrt{ \dimone \dimtwo} } \frac{\alpha \sqrt{\dimone \dimtwo} }{\epsilon} \log \frac{ \alpha \sqrt{ \dimone \dimtwo}  }{\epsilon} d \epsilon 
+ \alpha t \sqrt{ \dimone \log \dimone } \\
&\lesssim \alpha \sqrt{\dimone \dimtwo} ( \log \dimone )^2 
+ \alpha t \sqrt{ \dimone \log \dimone } . 
\end{align*}
Now consider the second term of the bound~\eqref{eq:l2-me}; we have 
\begin{align*}
\gamma_2 ( \Cperm \cap \cB_{M^*} (t), \alpha \| \cdot \|_F ) 
&\lesssim \int_0^{\alpha ( t \land \sqrt{ \dimone \dimtwo } ) } \bigg\{ \Big(  \dimone \dimtwo \log \frac{ \alpha \sqrt{ \dimone \dimtwo}  }{\epsilon} \Big) + 2 \dimone \log \dimone \bigg\}^{1/2} d \epsilon \\
&\lesssim \alpha t \sqrt{\dimone \dimtwo \log ( \dimone/t + e ) } 
+ \alpha t \sqrt{ \dimone \log \dimone } . 
\end{align*}
Finally, the third term of bound~\eqref{eq:l2-me} can be used to obtain
\begin{align*}
\gamma_2 ( \Cperm \cap \cB_{M^*} (t), \alpha \| \cdot \|_F ) 
&\lesssim \int_0^{\alpha t } \bigg\{ \Big(  \frac{ \alpha \sqrt{ \dimone \dimtwo } }{ \epsilon } \dimtwo \log \dimone \Big)  + 2 \dimone \log \dimone \bigg\}^{1/2} d \epsilon \\
&\lesssim \alpha t^{1/2} \dimone^{1/4} \dimtwo^{3/4} \sqrt{ \log \dimone } 
+ \alpha t \sqrt{ \dimone \log \dimone } . 
\end{align*}
With the three bounds combined, we have 
\begin{align*}
&\gamma_2 ( \Cperm \cap \cB_{M^*} (t), \alpha \| \cdot \|_F ) 
\lesssim \alpha t \sqrt{ \dimone \log \dimone } + \\
&\Big[ \alpha \sqrt{\dimone \dimtwo} ( \log \dimone )^2 \Big] \land 
\Big[ \alpha t \sqrt{\dimone \dimtwo \log ( \dimone/t + e ) }  \Big]
\land \Big( \alpha t^{1/2} \dimone^{1/4} \dimtwo^{3/4} \sqrt{ \log \dimone }  \Big) . 
\end{align*}

Let us now turn to bounding the $\gamma_1$ functional in $\ell_\infty$ norm. Scaling the $\|\cdot\|_\infty$ norm by a factor $\beta$ amounts to replacing $\epsilon$ by $\epsilon/\beta$ in the metric entropy bound \eqref{eq:linf-me}, so we have 
\begin{align*}
&\gamma_1 ( \Cperm \cap \cB_{M^*} (t), \beta \| \cdot \|_\infty ) 
\le c_1 \int_0^{\beta (t \land 1) } \log N \big( \epsilon , \Cperm \cap \cB_{M^*} (t) , \beta \| \cdot \|_\infty \big) d \epsilon \\
& \le c_1 \int_0^{\beta} \Big\{ \Big[ \frac{ \beta \dimtwo }{\epsilon} \log (e \dimone) \Big] \land \Big( \dimone \dimtwo \log{ \frac{e \beta }{\epsilon} } \Big) + 2 \dimone \log \dimone  \Big\} d \epsilon \\
& \lesssim \int_0^{\beta \dimone^{-5}} \dimone \dimtwo \log{ \frac{\beta }{\epsilon} } d \epsilon 
+ \int_{\beta \dimone^{-5}}^{\beta } \frac{ \beta \dimtwo }{\epsilon} \log \dimone d \epsilon 
+ \beta \dimone \log \dimone \\
& \lesssim \beta \dimtwo (\log \dimone)^2 + \beta \dimone \log \dimone . 
\end{align*}

Putting together the pieces with the bound~\eqref{eq:sup-emp}, we obtain 
\begin{align*}
g(t) &\lesssim 
\Big[ \alpha \sqrt{\dimone \dimtwo} ( \log \dimone )^2 \Big] \land 
\Big[ \alpha t \sqrt{\dimone \dimtwo \log ( \dimone/t + e ) }  \Big]
\land \Big( \alpha t^{1/2} \dimone^{1/4} \dimtwo^{3/4} \sqrt{ \log \dimone }  \Big) \\
&\quad + \alpha t \sqrt{ \dimone \log \dimone } + \beta \dimtwo (\log \dimone)^2 + \beta \dimone \log \dimone 
+ \alpha t \sqrt{u}  + \beta u . 
\end{align*}
As a result, there is a universal positive constant $c$ such that choosing  
\begin{align}
\label{eq:t-star}
t \ge t^* &= c_3 \bigg\{ \Big[ \sqrt{\alpha} (\dimone \dimtwo)^{1/4} \log \dimone \Big] \land  
\Big[ \alpha \sqrt{\dimone \dimtwo \log (\dimone/\alpha + e) } \Big] 
\land \Big[ \alpha^{2/3} \dimone^{1/6} \dimtwo^{1/2} (\log \dimone)^{1/3} \Big] \\
&\quad + \alpha \sqrt{ \dimone \log \dimone } + \sqrt{\beta \dimtwo} \log \dimone + \sqrt{ \beta \dimone \log \dimone }
+ \alpha \sqrt{u} + \sqrt{ \beta u } \bigg\} \notag 
\end{align}
yields the bound $g(t) < t^2/8$. 
This holds for each \emph{individual} $t \ge t^*$ with probability $1 - e^{-u}$. 
We now prove that $f(t) < 0$ \emph{simultaneously} for all $t \ge t^*$ with high probability. Note that typically, the star-shaped property of the set suffices to provide such a bound, but we include the full proof for completeness.

To this end, we first note that by assumption, 
$$
\Pr \{ |W_{i, j}| \ge \alpha \sqrt{u} + \beta u \} \le 2 e^{-u} . 
$$
A union bound then implies that 
$$
\Pr \big\{ \| W \|_F \ge ( \alpha \sqrt{u} + \beta u ) \sqrt{ \dimone \dimtwo } \big\} \le 2 \dimone \dimtwo e^{-u} . 
$$
Therefore, we have with probability at least $1 - 2 \dimone \dimtwo e^{-u}$ that 
$$
g(t) 
\le t \| W \|_F \le t ( \alpha \sqrt{u} + \beta u ) \sqrt{ \dimone \dimtwo } 
$$
simultaneously for all $t \ge 0$. 
On this event, it holds that $f(t) < 0$ for all $t \ge t^\# = 3 ( \alpha \sqrt{u} + \beta u ) \sqrt{ \dimone \dimtwo }$. 

For $t \in [t^*, t^\#]$, we employ a discretization argument (clearly, we can assume $t^\# \ge t^*$ without loss of generality). 
Let $T = \{ t_1, \dots, t_k \}$ be a  discretization of the interval $[t^*, t^\#]$ such that
$
t^* = t_1 < \cdots < t_k = t^\# 
$
and
$ 2 t_1 \ge t_2 $. 
Note that $T$ can be chosen so that 
$$
|T| = k \le \log_2 \frac{t^\#}{t^*} + 1
\le \log_2 \Big( ( 3 + 3 \sqrt{\beta u} ) \sqrt{\dimone \dimtwo} \Big) + 1 
\le 7 \log (\dimone u) , 
$$ 
where we used the assumption that $\beta \le \dimone^2$. 
Using the high probability bound $g(t) < t^2/8$ for each individual $t \ge t^*$ and a union bound over $T$, we obtain that with probability at least $1 - 7 \log (\dimone u) e^{-u}$, 
$$
\max_{t \in T} g(t) - t^2/8 < 0 . 
$$
On this event, 
we use the fact that $g(t)$ is non-decreasing and that $t_i \ge t_{i+1}/2$ to conclude that for each $t \in [t_i, t_{i+1}]$ where $i \in [k-1]$,  we have
$$
f(t) = g(t) - t^2/2 \le g(t_{i+1}) - t_i^2/2 \le g(t_{i+1}) - t_{i+1}^2/8 \le \max_{t \in T} g(t) - t^2/8 < 0 . 
$$

In summary, we obtain that $f(t) < 0$ for all $t \ge t^*$ simultaneously with probability at least 
$
1 - 2 \dimone \dimtwo e^{-u} - 7 \log (\dimone u) e^{-u} . 
$
Choosing $u = 4 \dimone \log \dimone$, recalling the definition of $t^*$ in \eqref{eq:t-star} and applying Lemma \ref{lem:variational}, we conclude that with probability at least $1 - \dimone^{- 3 \dimone}$, 
\begin{align*}
& \big\| \Mhatls(\Cperm, Y) - M^* \big\|_F^2 
\le (t^*)^2 \lesssim \\
&\Big[ \alpha \sqrt{ \dimone \dimtwo } (\log \dimone )^2 \Big] \land  
\Big[ \alpha^2 \dimone \dimtwo \log (\dimone/\alpha + e)\Big] 
\land \Big[  \alpha^{4/3} \dimone^{1/3} \dimtwo (\log \dimone)^{2/3} \Big] \\
&+ \alpha^2 \dimone \log \dimone + \beta \dimtwo (\log \dimone)^2 +  \beta \dimone \log \dimone  . 
\end{align*}

The entire argument can be repeated for the class
$\Cbiso$, in which case terms of the order $\dimone \log \dimone$ disappear as there is no latent permutation. Since the argument is analogous, we omit the details.
This completes the proof of the lemma. \qed

\paragraph{Proof of Lemma~\ref{lem:entropy}} 
Note that $\sqrt{ \dimone \dimtwo }$ and $1$ are the diameters of $\Cperm$ in $\ell_2$ and $\ell_\infty$ norms respectively, so we can assume $\epsilon \le \sqrt{\dimone \dimtwo}$ or $1$ in the two cases. 

For the $\ell_2$ metric entropy, Lemma 3.4 of \cite{ChaGunSen18} yields
\[
\log N ( \epsilon, \Cbiso , \| \cdot \|_F ) \le c_1 \frac{\dimone \dimtwo}{\epsilon^2} \Big( \log \frac{ \sqrt{ \dimone \dimtwo} }{\epsilon} \Big)^2 ,
\]
which is the first term of \eqref{eq:l2-me}. 
In addition, since $\Cbiso$ is contained in the ball in $\real^{\dimone \times \dimtwo}$ of radius $\sqrt{ \dimone \dimtwo }$ centered at zero, we have the simple bound
\[
\log N ( \epsilon, \Cbiso , \| \cdot \|_F ) \le c_2 \dimone \dimtwo \log \frac{ \sqrt{ \dimone \dimtwo} }{\epsilon} ,
\]
which is the second term of \eqref{eq:l2-me}. 

Moreover, any matrix in $\Cbiso$ has Frobenius norm bounded by $\sqrt{\dimone \dimtwo}$ and has non-decreasing columns
%
 so $\Cperm$ is a subset of the class of matrices considered in Lemma~6.7 of the paper~\cite{FlaMaoRig16} with $A = 0$ and $t = \sqrt{\dimone \dimtwo}$ (see also equation~(6.9) of the paper for the notation). 
The aforementioned lemma yields the bound 
\[
\log N ( \epsilon, \Cperm , \| \cdot \|_F ) \le c_3 \frac{ \sqrt{ \dimone \dimtwo } }{ \epsilon } \dimtwo \log \dimone. 
\]
Taking the minimum of the three bounds above yields the $\ell_2$ metric entropy bound~\eqref{eq:l2-me-biso} on the class of bivariate isotonic matrices. 
Since $\Cperm$ is a union of $\dimone ! \dimtwo !$ permuted versions of $\Cbiso$, combining this bound with an additive term $2 \dimone \log \dimone$ provides the bound for the class $\Cperm$.

For the $\ell_\infty$ metric entropy, we again start with $\Cbiso$. 
Let us define a discretization $I_\epsilon = \{ 0, \epsilon, 2 \epsilon, \dots, \lfloor 1/\epsilon \rfloor \epsilon \} $ and a set of matrices 
$$
\cQ = \big\{ M \in \Cbiso : M_{i, j} \in I_\epsilon \big\} ,
$$
which is a discretized version of $\Cbiso$. 
We claim that $\cQ$ is an $\epsilon$-net of $\Cbiso$ in the $\ell_\infty$ norm. 
Indeed, for any $M \in \Cbiso$, we can define a matrix $M'$ by setting
$$
M'_{i, j} = \operatorname*{argmin}_{a \in I_\epsilon } | a - M_{i, j} | 
$$
with the convention that if $M_{i, j} = (k + 0.5) \epsilon$ for an integer $0 \le k < \lfloor 1/\epsilon \rfloor$, then we set $M'_{i, j} = (k+1) \epsilon$. 
It is not hard to see that $M' \in \cQ$ and moreover $\|M' - M\|_\infty \le \epsilon$. 
Therefore the claim is established.

It remains to bound the cardinality of $\cQ$. 
Since each column of $\cQ$ is non-decreasing and takes values in $I_\epsilon$ having cardinality $\lfloor 1/\epsilon \rfloor + 1$, it is well known (by a ``stars and bars" argument) that the number of possible choices for each column of a matrix in $\cQ$ can be bounded as 
$$
\binom{ \dimone + \lfloor 1/\epsilon \rfloor }{ \lfloor 1/\epsilon \rfloor } \le \Big( e \frac{ \dimone + \lfloor 1/\epsilon \rfloor }{ \lfloor 1/\epsilon \rfloor } \Big)^{\lfloor 1/\epsilon \rfloor } 
\land \Big( e \frac{ \dimone + \lfloor 1/\epsilon \rfloor }{ \dimone } \Big)^{\dimone} ,
$$
where we used the bound $\binom{n}{k} \le (\frac{ en }{k} )^k$ for any $0 \le k \le n$. 
Since a matrix in $\cQ$ has $\dimtwo$ columns, we obtain 
$$
\log |\cQ| \le \dimtwo \log \binom{ \dimone + \lfloor 1/\epsilon \rfloor }{ \lfloor 1/\epsilon \rfloor } 
\le \Big[ \frac{\dimtwo }{\epsilon} \log (e \dimone) \Big] \land \Big( \dimone \dimtwo \log{ \frac{e}{\epsilon} } \Big) . 
$$

This bounds $\log N ( \epsilon, \Cbiso , \| \cdot \|_\infty )$ and the same argument as before yields the bound for $\Cperm$. \qed



\subsection{Proof of Theorem~\ref{thm:funlim}}

We split the proof into three distinct parts. We first prove the upper bound in part (a) of the theorem, and then prove the lower bound in part (b) in two separate sections. The proof in each section encompasses more than one regime of the theorem, and allows us to precisely pinpoint the source of the minimax lower bound.

\subsubsection{Proof of part (a)}

As argued, the bound is trivially true when $N \le \dimone$ by the boundedness of the set $\Cperm$, so we assume  that $N \ge \dimone \ge \dimtwo$. 
Under our observation model \eqref{eq:obs-Y}, if we define $W = Y - M^*$, then Lemma~\ref{lem:par-sum} implies that the assumption of Lemma~\ref{lem:shah} is satisfied with $\alpha = c_1 \varplusone \sqrt{ \frac{\dimone \dimtwo}{N} }$ and $\beta = c_1 \varplusone \frac{\dimone \dimtwo}{N}$ for a universal constant $c_1 > 0$. 
Therefore, Lemma~\ref{lem:shah}(a) yields that with probability at least $1 - \dimone^{- 3 \dimone}$, we have
\begin{align*}
& \big\| \Mhatls(\Cperm, Y) - M^* \big\|_F^2 \lesssim \varplusone^2 \frac{ \dimone^2 \dimtwo}{N} \log \dimone + \varplusone \frac{\dimone \dimtwo^2}{ N} (\log \dimone)^2 
\\ 
&+ \Big[ \varplusone \frac{ \dimone \dimtwo }{\sqrt{N}} (\log \dimone )^2 \Big] \land
\Big[ \varplusone^2 \frac{ \dimone^2 \dimtwo^2 }{ N } \log N \Big] \land
\Big[ \varplusone^{4/3} \frac{\dimone \dimtwo^{5/3} }{ N^{2/3} } (\log \dimone)^{2/3}  \Big]   .
\end{align*}
Normalizing the bound by $1/(\dimone \dimtwo)$ completes the proof. 
\qed

\subsubsection{Proof of part (b): permutation error}


We start with the term $\frac{\dimone}{N} \land 1$ of the lower bound which stems from the unknown permutation on the rows. 

Its proof is an application of Fano's lemma. The technique is standard, and we briefly review it here.
Suppose we wish to estimate a parameter
$\theta$ over an indexed class of distributions $\mathcal{P} =
\{\mathbb{P}_\theta \, \mid\, \theta \in \Theta \}$ in the square of a
(pseudo-)metric $\rho$. We refer to a subset of parameters
$\{\theta^1, \theta^2, \ldots, \theta^K \}$ as a local $(\delta,
\epsilon)$-packing set if
\begin{align*}
\min_{i,j \in [K], \, i\neq j} \rho(\theta^i, \theta^j) \geq \delta
\qquad \text{ and } \qquad \frac{1}{K(K-1)} \sum_{i, j \in [K], \, i\neq j}
D(\mathbb{P}_{\theta^i} \| \mathbb{P}_{\theta^j}) \leq \epsilon.
\end{align*}
Note that this set is a $\delta$-packing in the metric $\rho$ with the
average Kullback-Leibler (KL) divergence bounded by $\epsilon$.  The following result is
a straightforward consequence of Fano's inequality (see \cite[Theorem~2.5]{Tsy09}):
\begin{lemma}[Local packing Fano lower bound]
\label{fanolbprime}
For any $(\delta, \epsilon)$-packing set of cardinality $K$, we have
\begin{align}
\inf_{\thetahat} \sup_{\thetastar \in \Theta} \EE \left[
\rho(\thetahat, \thetastar)^2\right] \geq \frac{\delta^2}{2} \left(1
- \frac{\epsilon + \log 2}{\log K} \right). \label{eq:fano}
\end{align}
\end{lemma}

In addition, the Gilbert-Varshamov bound~\cite{Gil52, Var57} guarantees the
existence of binary vectors $\{v^1, v^2, \ldots, v^K \} \subseteq \{0, 1\}^{\dimone}$
such that 
\begin{subequations}
\begin{align}
K &\geq 2^{c_1 \dimone} \text{ and} \\
\| v^i - v^j \|_2^2 &\geq c_2 \dimone \text{ for each } i \neq j,
\end{align}
\end{subequations}
for some fixed tuple of constants $(c_1, c_2)$. We use this guarantee to design a packing of matrices in the class $\Cpermr$. For each $i \in [K]$, fix some $\delta \in [0, 1/4]$ to be precisely set later, and define the matrix $M^i$ having identical columns, with entries given by
\begin{align}
M^i_{j, k} = 
\begin{cases}
1/2, \text{ if }v^i_j = 0 \\
1/2 + \delta, \text{ otherwise.}
\end{cases}
\end{align}
Clearly, each of these matrices $\{M^i\}_{i = 1}^{K}$ is a member of the class $\Cpermr$, and each distinct pair of matrices $(M^i, M^j)$ satisfies the inequality $\| M^i - M^j \|_F^2 \geq c_2 \dimone \dimtwo \delta^2$. 

Let $\mathbb{P}_{M}$ denote the probability distribution of the observations in the model \eqref{eq:model} with underlying matrix $M \in \Cpermr$. Our observations are independent across entries of the matrix, and so the KL divergence tensorizes to yield
\begin{align}
D(\mathbb{P}_{M^i} \| \mathbb{P}_{M^j}) = \sum_{\substack{k \in [\dimone] \\ \ell \in [\dimtwo]}} D(\mathbb{P}_{M^i_{k, \ell}} \| \mathbb{P}_{M^j_{k, \ell}}).
\label{eq:kl-sum}
\end{align}
Let us now examine one term of this sum. Note that we observe $\kappa_{k, \ell} \sim \Poi(\frac{N}{\dimone \dimtwo})$ samples of each entry $(k, \ell)$.  

Under the Bernoulli observation model~\eqref{eq:Bernoise}, conditioned on the event $\kappa_{k,\ell} = \kappa$, we have the distributions
\begin{align*}
\mathbb{P}_{M^i_{k,\ell}} = \BIN(\kappa, M^i_{k, \ell}), \quad \text{and} \quad
\mathbb{P}_{M^j_{k,\ell}} = \BIN(\kappa, M^j_{k, \ell}).
\end{align*}
Consequently, the KL divergence conditioned on $\kappa_{k, \ell} = \kappa$ is given by
\begin{align*}
D(\mathbb{P}_{M^i_{k, \ell}} \| \mathbb{P}_{M^j_{k, \ell}}) = \kappa D(M^i_{k, \ell} \| M^j_{k, \ell}),
\end{align*}
where we have used $D(p \| q) = p \log ( \frac{p}{q} ) + (1 - p) \log ( \frac{1 - p}{1 - q} )$ to denote the KL divergence between the Bernoulli random variables $\BER(p)$ and $\BER(q)$.

Note that for $p, q \in [1/2, 3/4]$, we have
\begin{align*}
D(p \| q) &= p \log \left( \frac{p}{q} \right) + (1 - p) \log \left( \frac{1 - p}{1 - q} \right) \\
&\stackrel{\1}{\leq} p \left( \frac{p - q}{q} \right) + (1 - p) \left( \frac{q - p}{1 - q} \right) \\
&= \frac{(p - q)^2}{q (1 - q)} \\
&\stackrel{\2}{\leq} \frac{16}{3} (p - q)^2. 
\end{align*}
Here, step $\1$ follows from the inequality $\log x \leq x - 1$, and step $\2$ from the assumption $q \in [\frac 12, \frac 34]$.
Taking the expectation over $\kappa$, we have
\begin{align}
D(\mathbb{P}_{M^i_{k, \ell}} \| \mathbb{P}_{M^j_{k, \ell}}) \le \frac{16}{3} \frac{N}{\dimone \dimtwo} (M^i_{k, \ell} - M^j_{k, \ell})^2 \le \frac{16}{3} \frac{N}{\dimone \dimtwo} \delta^2 ,
\label{eq:kl-entry}
\end{align}
Summing over $k \in [\dimone], \ell \in [\dimtwo]$ yields
$D(\mathbb{P}_{M^i} \| \mathbb{P}_{M^j}) \leq \frac{16}{3} N \delta^2.$

Under the standard Gaussian observation model~\eqref{eq:Gaussnoise}, a similar argument yields the bound $D(\mathbb{P}_{M^i} \| \mathbb{P}_{M^j}) \leq \frac{1}{2} N \delta^2$,
since we have $D( \mathcal N(p, 1) \| \mathcal N(q, 1) ) = (p-q)^2/2$.

Substituting into Fano's inequality~\eqref{eq:fano}, we have
\begin{align*}
\inf_{\Mhat} \sup_{\Mstar \in \Cpermr} \EE \left[
\| \Mhat - \Mstar\|_F^2\right] \geq \frac{c_2 \dimone \dimtwo \delta^2}{2} \left(1
- \frac{ \frac{16}{3} N \delta^2 + \log 2}{c_3 \dimone} \right).
\end{align*}
Finally, if $N \ge c_4 \dimone$ for a sufficiently large constant $c_4 > 0$, then we obtain the lower bound of order $\dimone / N$ by choosing $\delta^2 = c \frac{\dimone}{N}$ for some constant $c > 0$ and normalizing by
$1/(\dimone \dimtwo)$. 
If $N \le c_4 \dimone$, then we simply choose $\delta$ to be a sufficiently small constant and normalize to obtain the lower bound of constant order. 
\qed

\subsubsection{Proof of part (b): estimation error}
\label{sec:est-err}

We now turn to the term
$\frac{1}{\sqrt{N} } \wedge \left(\frac{\dimtwo}{N} \right)^{2/3} \wedge \frac{\dimone \dimtwo}{N}$ 
of the lower bound which stems from estimation of the underlying bivariate isotonic matrix even if the permutations are given. 
This lower bound is partly known for the model of one observation per entry under Gaussian noise \cite{ChaGunSen18}, 
and it suffices to slightly extend their proof to fit our model. 
The proof is based on Assouad's lemma. 

\begin{lemma}[Assouad's Lemma]\label{lem:assouad}
Consider a parameter space $\Theta$. Let $\mathbb{P}_\theta$ denote the distribution of the model given that the true parameter is $\theta \in \Theta$. Let $\E_\theta$ denote the corresponding expectation.
Suppose that for each $\tau\in\{-1, 1\}^d$, there is an associated $\theta^\tau \in \Theta$. 
Then it holds that
$$\underset{\tilde\theta}{\inf}\underset{\theta^* \in\Theta}{\sup} \E_{\theta^\ast} \ell^2(\tilde\theta, \theta^*) \geq \frac d8 \, \underset{\tau\neq\tau'}{\min}\frac{\ell^2(\theta^{\tau}, \theta^{\tau'})}{d_H(\tau, \tau')} \, \underset{d_H(\tau, \tau') = 1}{\min}(1 - \|\mathbb P_{\theta^\tau} - \mathbb P_{\theta^{\tau'}}\|_{TV}),$$
where $\ell$ denotes any distance function on $\Theta$, $d_H$ denotes the Hamming distance, $\|\cdot\|_{TV}$ denotes the total variation distance, and the infimum is taken over all estimators $\tilde \theta$ measurable with respect to the observation. 
\end{lemma}

To apply the lemma, we construct a mapping from the hypercube to $\Cbiso \subset \Cperm$. 
Consider integers $k_1 \in [\dimone]$ and $k_2 \in [\dimtwo]$, and let $m_1 = \dimone/k_1$ and $m_2 =  \dimtwo/k_2$.
Assume without loss of generality that \( m_1 \) and \( m_2 \) are integers.
For each $\tau\in\{-1,1\}^{k_1 \times k_2}$,  define $M^\tau\in\Cbiso$ in the following way. For $i\in[\dimone]$ and $j\in[\dimtwo]$, first  identify the unique $u\in[k_1],v\in[k_2]$ for which $(u-1)m_1<i\leq um_1$ and $(v-1)m_2<j\leq vm_2$, and then  take
$$M_{i, j}^\tau = \lambda \Big( \frac{u + v - 1}{k_1 + k_2} + \frac{\tau_{u,v}}{2 (k_1 + k_2) } \Big) $$
for $\lambda \in (0, 1]$ to be chosen. 
It is not hard to verify that $M^\tau\in\Cbiso$. 

We now proceed to show the lower bound using Assouad's lemma, which in our setup states that 
\begin{align}
&\underset{\tilde M}{\inf}\underset{M^* \in \Cbiso}{\sup} \E_{M^*} \|\tilde M - M^*\|_F^2 \geq 
\label{eq:a-lem-sp} \\
& \quad \frac{k_1 k_2}{8} \, \underset{\tau\neq\tau'}{\min}\frac{ \| M^{\tau} - M^{\tau'}\|_F^2 }{d_H(\tau, \tau')} \, \underset{d_H(\tau, \tau') = 1}{\min}(1 - \|\mathbb P_{M^\tau} - \mathbb P_{M^{\tau'}}\|_{TV}). \notag 
\end{align}

For any $\tau, \tau' \in \{-1, 1\}^{k_1 \times k_2}$, it holds that 
\begin{align}
\| M^\tau - M^{\tau'}\|_F^2 &=  \sum_{i,j}(M^\tau_{i, j}-M^{\tau'}_{i, j})^2  \notag \\
&= \sum_{u\in[k_1],v\in[k_2]} \sum_{(u-1)m_1<i\leq um_1}\sum_{(v-1)m_2<j\leq vm_2}(M^{\tau}_{i, j} - M^{\tau'}_{i,j})^2  \notag \\
&= \sum_{u\in[k_1],v\in[k_2]} m_1 m_2 \frac{ \lambda^2 (\tau_{u,v} - \tau'_{u,v})^2 }{4 (k_1 + k_2)^2} \notag \\
&= \frac{ \lambda^2 m_1 m_2}{(k_1 + k_2)^2} d_H(\tau, \tau') .
\label{eq:fro-ham}
\end{align}
As a result, we have
\begin{align}
\underset{\tau\neq\tau'}{\min}\frac{ \| M^{\tau} - M^{\tau'} \|_F^2 }{d_H(\tau, \tau')} = \frac{ \lambda^2 m_1 m_2}{(k_1 + k_2)^2} .
\label{eq:min-ratio}
\end{align}

To bound $\|\mathbb P_{M^\tau} - \mathbb P_{M^{\tau'}}\|_{TV}$, we use Pinsker's inequality 
$$
\|\mathbb P_{M^\tau} - \mathbb P_{M^{\tau'}}\|^2_{TV} \leq \frac12 D(\mathbb P_{M^\tau}\|\mathbb P_{M^{\tau'}}) . 
$$ 
Under either the Bernoulli or the standard Gaussian observation model, we have seen that by combining \eqref{eq:kl-sum} and \eqref{eq:kl-entry} (which hold even in this regime of $N$), the KL divergence can be bounded as 
\begin{align*}
D(\mathbb P_{M^\tau} \|\mathbb P_{M^{\tau'}}) &\le \sum_{i \in [\dimone], \, j \in [\dimtwo]} \frac{16}{3} \frac{N}{\dimone \dimtwo}  (M^{\tau}_{i,j} - M^{\tau'}_{i,j} )^2 \\ 
&= \frac{16}{3} \frac{N}{\dimone \dimtwo} \| M^{\tau} - M^{\tau'} \|_F^2 \\
&= \frac{16}{3} \frac{N}{\dimone \dimtwo} \frac{ \lambda^2 m_1 m_2}{(k_1 + k_2)^2} = \frac{16}{3} \frac{ \lambda^2 N}{k_1 k_2(k_1 + k_2)^2} , 
\end{align*}
where the second-to-last equality follows from \eqref{eq:fro-ham} for any $\tau$ and $\tau'$ such that $d_H(\tau, \tau') = 1$. 
Therefore, it holds that 
\begin{align}
\underset{d_H(\tau, \tau')=1}{\min}(1-\|\mathbb P_{M^\tau} - \mathbb P_{M^{\tau'}}\|_{TV})\geq\bigg(1- \sqrt{\frac{ 8 N }{ 3 k_1 k_2 }} \frac{\lambda}{k_1 + k_2} \bigg).
\label{eq:min-tv}
\end{align}

Plugging \eqref{eq:min-ratio} and \eqref{eq:min-tv} into Assouad's lemma \eqref{eq:a-lem-sp}, we obtain 
\begin{align} \label{eq:assouad-lb}
\underset{\tilde M}{\inf}\underset{M^* \in \Cbiso}{\sup} \E_{M^*} \|\tilde M - M^*\|_F^2 \geq 
\frac{ \lambda^2 n_1 n_2}{8 (k_1 + k_2)^2}  \bigg(1- \sqrt{\frac{ 8 N }{ 3 k_1 k_2 }} \frac{\lambda}{k_1 + k_2}  \bigg) . 
\end{align}
Note that the bound~\eqref{eq:assouad-lb} holds for all tuples $(N, \dimone, \dimtwo)$. In order to obtain the various regimes, we must set particular values of $k_1$ and $k_2$. 
If $N \le 1$, then setting $k_1 = k_2 = 1$ and $\lambda$ to be a sufficiently small constant clearly gives the trivial bound. 
If $1\le N \le \dimtwo^4$, then we take $k_1 = k_2 = \lfloor N^{1/4} \rfloor \le \dimtwo$ and $\lambda$ to be a sufficiently small positive constant so that 
$$
\underset{\tilde M}{\inf}\underset{M^* \in \Cbiso}{\sup} \E_{M^*} \|\tilde M - M^*\|_F^2 \geq 
c_1 \frac{ \dimone \dimtwo }{ \sqrt{N} } , 
$$
for a constant $c_1 > 0$. 
If $\dimtwo^4 \le N \le \dimone^3 \dimtwo$, we take $k_1 = \lfloor (N/\dimtwo)^{1/3} \rfloor \le \dimone$, $k_2 = n_2$ and $\lambda$ to be a sufficiently small positive constant so that 
$$
\underset{\tilde M}{\inf}\underset{M^* \in \Cbiso}{\sup} \E_{M^*} \|\tilde M - M^*\|_F^2 \geq 
c_2 \Big( \frac{ \dimtwo }{ N } \Big)^{2/3} , 
$$
for a constant $c_2 > 0$. 
Finally, if $N \ge \dimone^3 \dimtwo$, we choose $k_1 = \dimone$, $k_2 = \dimtwo$ and $\lambda = \sqrt{ \frac{ 3 \dimone \dimtwo }{ 32 N } } (\dimone + \dimtwo)$ to conclude that
$$
\underset{\tilde M}{\inf}\underset{M^* \in \Cbiso}{\sup} \E_{M^*} \|\tilde M - M^*\|_F^2 \geq 
\frac{ 3 (n_1 n_2)^2 }{256 N } .
$$
Normalizing the above bounds by $1/(\dimone \dimtwo)$ yields the theorem. 
\qed

\subsection{Proof of Corollary~\ref{cor:biso}}

The proof of this corollary follows from the steps used to establish Theorem~\ref{thm:funlim}. In particular, for the upper bound, applying Lemma~\ref{lem:shah}(b) with the same parameter choices as above yields that with probability at least $1 - \dimone^{- 3 \dimone}$, we have
\begin{align*}
& \big\| \Mhatls(\Cbiso, Y) - M^* \big\|_F^2 \lesssim \varplusone \frac{\dimone \dimtwo^2}{ N} (\log \dimone)^2 
\\ 
&+ \Big[ \varplusone \frac{ \dimone \dimtwo }{\sqrt{N}} (\log \dimone )^2 \Big] \land
\Big[ \varplusone^2 \frac{ \dimone^2 \dimtwo^2 }{ N } \log N \Big] \land
\Big[ \varplusone^{4/3} \frac{\dimone \dimtwo^{5/3} }{ N^{2/3} } (\log \dimone)^{2/3}  \Big]   .
\end{align*}
Normalizing the bound by $1/(\dimone \dimtwo)$ proves the upper bound.

The lower bound established in Section~\ref{sec:est-err} is valid for the class $\Cbiso$, so the proof is complete. 
\qed


\subsection{Proof of Theorem~\ref{thm:rowlb}}
\label{sec:pf-rowlb}



At a high level, our strategy is to reduce the problem to that of
hypothesis testing over particular instances of convex cones; we note
that cone testing problems have been extensively studied in past work
(see the paper~\cite{WeiGunWai17} and references therein). The
reduction takes the following form.  Consider the sub-classes of
$\Cpermr$ and $\Cperm$ in which the first $\dimone - 2$ rows are
identically zero. In this special case, obtaining a good estimate of
the row permutation $\pihat$ in the metric $\mathcal{R}(M^*, \pihat)$
corresponds to correctly ordering the last two rows of the matrix.
This ordering problem can be reduced to a compound hypothesis testing
problem.

Before diving into the details, let us introduce some useful notation. For two vectors $a$ and $b$ of equal dimension $n$, we use the notation $a \preceq b$ to mean that $a_i \leq b_i$ for all $i \in [n]$.
Recall that in our model, the number of observations at each entry has
distribution $\Poi(\lambda)$ where we set
$\lambda \defn \frac{N}{\dimone \dimtwo}$.  For a matrix $U \in [0, 1]^{d_1 \times d_2}$ (we often set this to be either one or two rows of $M^*$), we observe $\kappa_{i, j}$ noisy samples of $U_{i, j}$, where the entries of the $d_1 \times d_2$ matrix $\kappa$ are i.i.d. $\Poi(\lambda)$ random variables.  Under the Gaussian
observation model, the distribution $\mathbb{G}( U )$ of the
observations, when conditioned on a particular realization of $\kappa$, takes the form
\begin{align} \label{eq:gau-obs-dist}
\mathbb{G} ( U \cond \kappa) \defn \otimes_{i=1}^{d_1} \otimes_{j = 1}^{d_2} \otimes_{k
  = 1}^{\kappa_{i, j}} \NORMAL(U_{i, j}, 1)  , 
\end{align} 
where $\NORMAL(c, 1)$ denotes the standard Gaussian distribution with mean $c$.
Similarly, define the distribution $\Bdist(U)$ with Bernoulli observations replacing Gaussian observations.

We now set up a precise testing problem. For a set $\set \subseteq [0,
  1]^{2 \times \dimtwo}$ of pairs of row vectors\footnote{Since the setup of hypothesis testing is applied to two rows of the matrix $M^*$, we think of $u$ and $v$ as row vectors and write $[u; v]$ for the pair.}
$[u; v]$ and a radius of testing $\minradius$ to be chosen, let
$\mathbb{M}_{\set}(\minradius)$ denote a mixture distribution
supported on vector pairs $[u; v] \in \set$ that obey the conditions
\begin{align}
u &\preceq v, \text{ and} \label{eq:cond1}\\
\| u - v \|_2 &\geq \minradius. \label{eq:cond2}
\end{align}
Consider the testing problem (for the Gaussian observation setting) based on the pair of compound hypotheses
\begin{itemize}
\item
$H_0$: $X \sim \Gdist( [u; v] )$ where $[u; v] \sim \mathbb{M}_{\set}(\minradius)$, and
\item
$H_1$: $X \sim \Gdist( [v; u] )$ where $[u; v] \sim \mathbb{M}_{\set}(\minradius)$.
\end{itemize}
For Bernoulli observations, the distribution $\Gdist$ is replaced with the distribution $\Bdist$.
Denote by $\Pzero$ and $\Pone$ the distribution of $X$ under the hypotheses $H_0$ and $H_1$, respectively. Assume that the mixture distribution $\mathbb{M}_{\set}(\minradius)$ is constructed such
that
\begin{subequations}
\begin{align}
 \label{eq:cond3-a}  
\TV(\Pzero \,\|\, \Pone) \leq \frac{1}{2}.
\end{align}

We also define a related testing problem, given by the hypotheses
\begin{itemize}
\item
$\widetilde{H}_0$: $X = X_1 - X_2$ where $X_1 \sim \Gdist( u)$, $X_2 \sim \Gdist(v)$, and $[u; v] \sim \mathbb{M}_{\set}(\minradius)$;
\item
$\widetilde{H}_1$: $X = X_1 - X_2$ where $X_1 \sim \Gdist( v)$, $X_2 \sim \Gdist(u)$, and $[u; v] \sim \mathbb{M}_{\set}(\minradius)$.
\end{itemize}
Once again, for Bernoulli observations, the distribution $\Gdist$ is replaced with the distribution $\Bdist$.
Denote by $\Pzerotil$ and $\Ponetil$ the distribution of $X$ under the hypotheses $\widetilde{H}_0$ and $\widetilde{H}_1$, respectively, with
\begin{align}
 \label{eq:cond3-b}  
\TV(\Pzerotil \,\|\, \Ponetil) \leq \frac{1}{2}.
\end{align}
\end{subequations}
%
%

The following lemma shows that in order to establish a lower bound in the $\mathcal{R}$ metric, it is sufficient to lower bound the minimax testing radius of the hypothesis testing problem above, for a particular choice of the set $\set$.
\begin{lemma} \label{lem:testingreduction}
(a) Let $\set$ be the set of pairs of vectors $[u; v] \in [0,
    1]^{2 \times \dimtwo}$ such that there is a common permutation $\pi$ for
  which $\{u_{\pi(i)}\}_{i=1}^{\dimtwo}$ and
  $\{v_{\pi(i)}\}_{i=1}^{\dimtwo}$ are both non-decreasing.  Suppose
  that there exists a mixture $\mathbb{M}_{\set}(\minradius)$ and
  associated observation distributions $\Pzero$ and $\Pone$ that obey
  the conditions~\eqref{eq:cond1}, \eqref{eq:cond2}, and
  \eqref{eq:cond3-a}. Then we have
\begin{align*}
\inf_{\pihat} \sup_{M^* \in \Cperm} \EE[ \mathcal{R}(M^*, \pihat)]
\geq \frac{\minradius^2}{2 \dimtwo}.
\end{align*}
(b) Let $\set$ be the set 
of pairs of non-decreasing vectors in $[0, 1]^{\dimtwo}$, and suppose that there exists a mixture $\mathbb{M}_{\set}(\minradius)$ and
  associated observation distributions $\Pzerotil$ and $\Ponetil$ that obey
  the conditions~\eqref{eq:cond1}, \eqref{eq:cond2}, and
  \eqref{eq:cond3-b}  Then we have
\begin{align*}
\inf_{\pihat \in \mathcal{P}^{\mathsf{r}}_{\mathsf{pdd}}} \sup_{M^* \in \Cpermr} \EE[ \mathcal{R}(M^*, \pihat) ]
\geq \frac{\minradius^2}{4 \dimtwo}.
\end{align*}
\end{lemma}

We have thus reduced the problem to that of constructing a
mixture distribution $\mathbb{M}_{\set}$ satisfying particular conditions.
Establishing condition~\eqref{eq:cond3-a} (or~\eqref{eq:cond3-b}) is often the hardest part;
we use the following technical lemma to simplify the calculations.
\begin{lemma}
\label{lem:chi-sq-bd}
Consider positive integers $r, s$ and $d = rs$. 
Let $\theta^0$ be the all-half vector in $\real^d$. 
For each $k \in [s]$, let $I_k \defn \{(k-1) r + 1, \dots, k r\}$, and denote by $\theta^k \in \real^d$ the vector with entries
\begin{align*}
\theta^k_i = 
\begin{cases}
1/2 + \delta &\text{ if } i \in I_k \\
1/2 &\text{ otherwise.} 
\end{cases}
\end{align*}
Consider a random vector $w$ drawn uniformly at random from the set $\{\theta^k\}_{k \in [s]}$, and suppose that $\delta \le 2/5$. Then, we have 
\begin{align*}
\chi^2 \Big( \Gdist ( w ) \,\| \, \Gdist( \theta^0 ) \Big) &\le  \frac{2 \lambda \delta^2 r }{s}, \text{ provided } \lambda \delta^2 r \le 2/5, \text{ and } \\
\chi^2 \Big( \Bdist ( w ) \,\| \, \Bdist( \theta^0 ) \Big) &\le  \frac{8 \lambda \delta^2 r }{s}, \text{ provided } \lambda \delta^2 r \le 1/10.
\end{align*}
\end{lemma}

Lemmas~\ref{lem:testingreduction} and~\ref{lem:chi-sq-bd} are proved in Sections~\ref{sec:test-red} and~\ref{sec:pf-lem-chi}, respectively. Taking them as given, let us now proceed to a proof of the two parts.



\subsubsection{Proof of part (a)}

We require some notation to set up the mixture distribution. Let $s \in [\dimtwo]$ and to ease the notation, we assume that $\dimtwo/s$ is an integer; this only changes the constant factors in the proof. 
To define the mixture $[u; v] \sim \mathbb{M}_{\set} (\minradius)$, we set $u$ to be the vector with all entries identically equal to $1/2$, and construct a mixture for $v$. 
Let $e_i$ denote the $i$-th standard basis vector in $\real^s$, and for a positive scalar $\delta$ to be chosen, let $\delta e_i + 1/2$ denote the vector whose $i$-th coordinate is $\delta + 1/2$ and all other coordinates are $1/2$.  
Consider a vector $w$ chosen uniformly at random from the set $\{ \delta e_i + 1/2 \}_{i=1}^s$. 
Now define $v$ to be the Cartesian product of $\dimtwo/s$ independent copies of $w$, so that $v$ can be expressed as the all-half vector plus a vector of $\dimtwo/s$ ``bumps" of size $\delta$, with the locations of the bumps chosen at random within each sub-vector of size $s$.
It now remains to verify that the mixture induced by this construction satisfies conditions~\eqref{eq:cond1},~\eqref{eq:cond2}, and~\eqref{eq:cond3-a}. Condition~\eqref{eq:cond1} is satisfied by construction, as is condition~\eqref{eq:cond2} with $\minradius = \delta \sqrt{\dimtwo / s}$.

\paragraph{Verifying condition~\eqref{eq:cond3-a}}
Recall that we have $\mathbb{P}_0 = \Gdist([u; v])$ and $\mathbb{P}_1 = \Gdist([v; u])$. Also define $\mathbb{P} \defn \Gdist([u; u])$.
By triangle inequality, we have
\begin{align*}
\TV(\mathbb{P}_0 \,\|\, \mathbb{P}_1) &\le \TV(\mathbb{P}_0 \,\|\, \mathbb{P}) + \TV(\mathbb{P}_1 \,\|\, \mathbb{P}) \\
&= 2 \TV(\mathbb{P}_0 \,\|\, \mathbb{P}) \\
&\leq \sqrt{ 2 D(\mathbb{P}_0 \,\|\, \mathbb{P} ) },
\end{align*}
where we have used $D$ to denote the KL-divergence, and the final step holds by Pinsker's inequality.

Let $X_1$ and $X_2$ denote the observations on the two rows from the distribution $\Gdist([u; v])$. Then $X_1$ and $X_2$ are independent since $u$ is a fixed vector. 
An analogous statement is also true for the distribution $\Gdist([u; u])$. 
Therefore, for the Gaussian case, we have 
\begin{align*}
D (\mathbb{P}_0 \,\|\, \mathbb{P} ) &= D \Big( \Gdist([u; v]) \,\|\, \Gdist([u; u]) \Big) \\
&= D \Big( \Gdist(u) \,\|\, \Gdist(u) \Big) + D \Big( \Gdist(v) \,\|\, \Gdist(u) \Big) \\
&= D \Big( \Gdist(v) \,\|\, \Gdist(u) \Big),
\end{align*}
where we have abused notation slightly to use $\Gdist$ to denote both the matrix and vector distributions.

By construction, the sub-vector of $v$ on each of the $\dimtwo/s$ blocks is an independent copy of the vector $w$ defined above, so we further obtain 
\begin{align*}
D \Big( \mathbb{P}_0 \,\|\, \mathbb{P}  \Big) = \frac{\dimtwo}{s} D \Big( \Gdist(w) \,\|\, \Gdist(x) \Big) 
\le \frac{\dimtwo}{s} \chi^2 \Big( \Gdist(w) \,\|\, \Gdist(x) \Big) ,
\end{align*}
where $x$ is the vector in $\real^{s}$ with all entries equal to $1/2$. 

For the Gaussian case, let us take $\delta = 1/4$ and $s = \sqrt{\lambda \dimtwo} \vee 1$.
Since $\lambda \le 1$, applying Lemma~\ref{lem:chi-sq-bd} with $r = 1$ yields 
\[
D \Big( \mathbb{P}_0 \,\|\, \mathbb{P} \Big) \le \frac{2 \dimtwo \lambda \delta^2 }{ s^2 } \leq \frac{1}{8}. 
\]
Putting together the pieces shows the necessary bound on the $\TV$ distance.
An identical argument holds for Bernoulli observations by choosing a slightly smaller $\delta$, since Lemma~\ref{lem:chi-sq-bd} provides the same bound up to a constant factor.
\smallskip

Overall, we have shown that all three conditions~\eqref{eq:cond1}-\eqref{eq:cond3-a} hold with $\minradius^2 = \frac{ \dimtwo}{ 16} \sqrt{ \frac{ \dimone }{ N } } \wedge \frac{\dimtwo}{16}$, thereby proving the required result.
\qed

\subsubsection{Proof of part (b)} \label{sec:conjsupp}


Let the set $\mathcal{S}$ consist of pairs of non-decreasing vectors in $[0, 1]^{\dimtwo}$, and construct the mixture $\mathbb{M}_{\set}(\minradius)$ as follows.
%
Let $r, s \in [\dimtwo]$ and $t = \dimtwo / (r s)$. To ease the notation, we assume that $t$ is an integer without loss of generality. 
In order to build up to our final construction, first define random vectors $\mu$ and $\nu$ in
$\real^{s}$ by the assignment
\begin{align*}
\mu_j = \bfone\{j > J\} \quad \text{ and } \quad \nu_j = \bfone\{j \ge J\} ,
\end{align*} 
where $J$ is a uniform random index in $[s]$.

Next, define random vectors $\mathfrak{u}$ and $\mathfrak{v}$ on
$\real^{r s}$ by setting
\begin{align*}
\mathfrak{u}_i = \mu_{\lceil i/ r \rceil} \cdot \delta t^{-1} \quad
\text{ and } \quad \mathfrak{v}_i = \nu_{\lceil i/ r \rceil}
\cdot \delta t^{-1}
\end{align*}
for each $i \in [r s]$ and some fixed $\delta > 0$. 

Finally, we define random vectors $u$ and $v$ in $\real^{\dimtwo}$ as
follows. Split the $\dimtwo$ coordinates into $t$ consecutive
blocks, each of size $r s$. Let $\mathbf{b}$ denote the
all-ones vector in $\real^{r s}$.  For each even index $k \in
[t]$, define the sub-vectors of both $u$ and $v$ on the $k$-th
block to be $ \mathbf{b} \cdot (\frac{1}{4} + k\, t^{-1})$. On the
other hand, for each odd index $k$ in $[t]$, define the
sub-vectors of $u$ and $v$ on the $k$-th block respectively to be
\begin{align}
\label{eq:uv-def}
\mathbf{b} \cdot \Big(\frac{1}{4} + k\, t^{-1}\Big) +
\mathfrak{u}^{(k)} \quad \text{ and } \quad \mathbf{b} \cdot
\Big(\frac{1}{4} + k\, t^{-1}\Big) + \mathfrak{v}^{(k)} ,
\end{align}
where each $(\mathfrak{u}^{(k)}, \mathfrak{v}^{(k)})$ is an
independent copy of the pair of random vectors $(\mathfrak{u},
\mathfrak{v})$. Note that by construction, $u$ and $v$ are both
non-decreasing vectors in $[1/4, 3/4]^{\dimtwo}$, and we have that $u
\preceq v$.

Denote by $\mathbb{M}_{\set}(\minradius)$ the distribution of the
random pair $[u, v]$ constructed above.  
Clearly, this mixture satisfies conditions~\eqref{eq:cond1} and~\eqref{eq:cond2} by construction, with 
\[
\minradius^2
\gtrsim \frac{\delta^2 r }{ t } 
= \frac{\delta^2 r^2 s }{ \dimtwo } ,
\]
since the vectors $u$
and $v$ differ in the squared $\ell_2$ norm by $\delta^2 t^{-2} r$ on each block, and there are $t$ blocks in total. 
It remains to establish condition~\eqref{eq:cond3-b}.

%
\paragraph{Verifying condition~\eqref{eq:cond3-b}}

In order to verify condition~\eqref{eq:cond3-b}, we must bound the total variation distance between the distributions\footnote{Strictly speaking, for the difference of the two observed vectors, the variance of the Gaussian noise is $2$ instead of $1$, but the constants involved can be easily adjusted.} $\Gdist( u - v )$ and $\Gdist( v - u )$. Proceeding as in the proof of part (a) yields the bounds
\begin{align*}
\TV(\Gdist( u - v ) \,\|\, \Gdist( v - u )) &\leq \sqrt{2 D \left( \Gdist(u - v) \,\|\, \Gdist(0) \right) }, \text{ and }\\
D \left( \Gdist(u - v) \,\|\, \Gdist(0) \right) &\leq \frac{t}{2} D \left( \Gdist(\mathfrak{u} - \mathfrak{v}) \,\|\, \Gdist(0) \right),
\end{align*}
since the $t$ blocks in the mixture are independently generated, and the even ones are additionally i.i.d.

Furthermore, we have
\begin{align*}
D \left( \Gdist(\mathfrak{u} - \mathfrak{v}) \,\|\, \Gdist(0) \right) &\leq \chi^2 \left( \Gdist(\mathfrak{u} - \mathfrak{v}) \,\|\, \Gdist(0) \right) 
\leq \frac{2 \lambda \delta^2 r }{s t^2}
\end{align*}
where the final bound holds by Lemma~\ref{lem:chi-sq-bd} (note that $\delta$ in the statement of the lemma is now set to be $\delta t^{-1}$ in our packing construction, and all the vectors are shifted by the all-half vector) provided we have 
$
\frac{\lambda \delta^2 r}{ t^2 } = \frac{\lambda \delta^2 r^3 s^2}{ \dimtwo^2 } \le 2/5.
$ 
Putting together the pieces, the $\TV$ distance is bounded as
\begin{align*}
\TV(\Gdist( u - v ) \,\|\, \Gdist( v - u )) \leq \sqrt{ \frac{4 \lambda \delta^2 r }{s t} } = \sqrt{ \frac{4 \lambda \delta^2 r^2 }{\dimtwo} }.
\end{align*}
Finally, provided $\lambda \leq \dimtwo$, we may make the choice 
\begin{align*}
\delta &= 1/4, \\
r &= \sqrt{\dimtwo / \lambda} \wedge \dimtwo \text{ and} \\ 
s &=  (\frac 12 \lambda \dimtwo)^{1/4} \vee 1,
\end{align*}
to yield the final bounds $\TV(\Gdist( u - v ) \,\|\, \Gdist( v - u )) \leq 1/2$ and 
\[
\minradius^2 \gtrsim \dimtwo (\lambda \dimtwo)^{-3/4} \wedge \dimtwo \gtrsim \dimtwo \left( \frac{\dimone}{N} \right)^{3/4} \wedge \dimtwo. 
\]
Applying Lemma~\ref{lem:testingreduction}(b) then completes the proof.
\qed

\subsubsection{Proof of Lemma~\ref{lem:testingreduction}} \label{sec:test-red}

Let us prove the lemma for an arbitrary choice of the set $\set$ and
mixture distribution $\mathbb{M}_{\set}(\minradius)$ satisfying conditions
~\eqref{eq:cond1} and~\eqref{eq:cond2}; we specialize these choices shortly
to satisfy either condition~\eqref{eq:cond3-a} or condition~\eqref{eq:cond3-b}.

It
suffices to prove lower bounds over procedures provided with the
additional information that $\pi^*(i) = i$ for all $i \in [\dimone -
2]$.
Consider the class of matrices with the first $\dimone - 2$
rows set identically to zero.  Now choose a vector pair $[u; v] \sim
\mathbb{M}_{\set}(\minradius)$, and set
\begin{align} \label{eq:mchoice}
[M^*_{\dimone - 1}; M^*_{\dimone} ] = 
\begin{cases}
[u; v] \text{ with prob. } 1/2 \\
[v; u] \text{ with prob. } 1/2.
\end{cases}
\end{align}
The minimax rate is clearly lower bounded by the Bayes risk over this
random choice of $M^*$.

Our observations are now given by a noisy version of $M^*$, where each
entry of $M^*$ is observed independently a $\Poi(\frac{N}{\dimone
\dimtwo})$ number of times. Let the last two rows of our
observations be given by $y_1$ and $y_2$, respectively. 
For any pair of permutation estimates $(\pihat, \sigmahat)$, we
may thus write
\begin{align*}
\EE \, \mathcal{R}(M^*, \pihat) &\geq \frac{1}{\dimtwo} \EE \left[ \|
[M^*(\pi^*, \sigma^*)]_{\dimone} - [M^*(\pihat, \sigma^*)]_{\dimone}
\|_2^2 \right] \\ &\stackrel{\1}{\geq} \frac{1}{\dimtwo} \left(
\min_{[u; v] \sim \mathbb{M}_{\set}(\minradius)} \| u - v \|^2_2 \right) \cdot
\Pr\{[\pihat(\dimone - 1), \pihat(\dimone)] \neq [\pi^*(\dimone - 1),
\pi^*(\dimone)] \},
\end{align*}
where in step \1, we have used the fact that for any instance $[u; v]$
from the mixture $\mathbb{M}_{\set}(\minradius)$, it holds that
$\min\{ \|u\|_2^2, \|v\|_2^2\} \geq \|u - v \|_2^2$, so that the
returned permutation is always worse off if it swaps the last two rows
with any of the rows of $M^*$ that are identically zero. 

We now specialize the proof to the two cases.

\paragraph{Proof of part (a)}
Here, we are interested in distinguishing the hypotheses $H_0$ and $H_1$ above from observations of $y_1$ and $y_2$; the pair $[y_1; y_2]$ is drawn according to the distribution $\mathbb{P}_0$ and $\mathbb{P}_1$ in the null and alternative, respectively. Thus, we have the bound
\begin{align*}
\Pr\{[\pihat(\dimone - 1), \pihat(\dimone)] \neq [\pi^*(\dimone - 1),
\pi^*(\dimone)] \} \geq 1 - \TV(\mathbb{P}_0 \,\|\, \mathbb{P}_1) \geq 1/2,
\end{align*}
where the final step follows by condition~\eqref{eq:cond3-a}.
\qed

\paragraph{Proof of part (b)}

Now, we are interested in lower bounding the risk of any PDD estimator $\pihat$. 
Let $p$ denote the probability with which Step~1 of the algorithm producing $\pihat$ places a correctly directed edge between $\dimone - 1$ and $\dimone$. 
By the PDD property, this edge only depends on the observations through the difference $y_1 - y_2$. Thus, in order to lower bound the probability $1-p$ of not placing the correct edge, it suffices to lower bound the probability of error distinguishing the hypotheses $\widetilde{H}_0$ and $\widetilde{H}_1$ defined above.
Consequently, 
we have
\begin{align*}
1-p \geq 1 - \TV( \Pzerotil \,\|\, \Ponetil ) \geq 1/2,
\end{align*}
where the final step follows by condition~\eqref{eq:cond3-b}.

Moreover, when no edge is placed between $\dimone - 1$ and $\dimone$ in Step~1 of the algorithm (which happens with probability at most $1 - p$), we may consider two cases. 
If the (unordered) set $\{ \pihat(\dimone - 1), \pihat(\dimone) \}$ is not equal to $\{ \dimone - 1, \dimone \}$, then we immediately have $[\pihat(\dimone - 1), \pihat(\dimone)] \neq [\pi^*(\dimone - 1), \pi^*(\dimone)]$. 
Otherwise, both the choices $[\pihat(\dimone - 1), \pihat(\dimone)] = [\pi^*(\dimone - 1),
\pi^*(\dimone)]$ and $[\pihat(\dimone - 1), \pihat(\dimone)] = [\pi^*(\dimone),
\pi^*(\dimone - 1)]$ must be consistent with other edges given by Step~1, as $\dimone - 1$ and $\dimone$ are the last two nodes. Hence the random permutation $\pihat$ gives the correct relative order between $\dimone - 1$ and $\dimone$ with probability at most $1/2$. 
In conclusion, we have 
$$
\Pr \{ [\pihat(\dimone - 1), \pihat(\dimone)] \neq [\pi^*(\dimone - 1),
\pi^*(\dimone)] \} \ge (1-p) / 2 \ge 1/4 ,
$$
which completes the proof.  
\qed


\subsubsection{Proof of Lemma~\ref{lem:chi-sq-bd}} \label{sec:pf-lem-chi}

Note that there are two levels of randomness here: the first due to the fact that we have a random number of observations of each entry of the vector, and the second due to the mixture distribution according to which our observations are drawn. Let $\kappa$ denote the random vector of the numbers of observations, i.e., the random variable $\kappa_i$ denotes the number of observations at the $i$-th entry of the vector. Clearly, we have
\[
\chi^2( \Gdist ( w ) \,\|\, \Gdist ( \theta^0 ) ) = \E_{\kappa}
\big[ \chi^2( \Gdist ( w \cond \kappa) \,\|\, \Gdist (\theta^0 \cond \kappa) ) \big], 
\]
so that it suffices to bound the conditional $\chi^2$ divergence on the RHS above.
Now condition on $\kappa$ and use the shorthand $m = \|\kappa\|_1$ to denote the total number of observations. 
Let $\mu^k$ denote the vector in $\real^{m}$ that is the Cartesian product of $\kappa_i$ copies of $\theta^k_i$ over all $i \in [d]$. 
Let $p^k$ denote the density of $\NORMAL(\mu^k, I_m)$ in the Gaussian case.
Recall that $w$ is drawn uniformly at random from the underlying means; the definition of the $\chi^2$ divergence thus yields
\begin{align}
\chi^2( \Gdist ( w \cond \kappa) \,\|\, \Gdist (\theta^0 \cond \kappa) ) = \frac 1{s^2} \sum_{k, \ell = 1}^s \int \frac{ p^k p^{\ell} }{ p^0 } - 1 .
\label{eq:cond-chi}
\end{align}
In the case of Bernoulli observations, we use $p^k$ to denote independent Bernoulli entries with mean $\mu^k$, and replace the integral above by the appropriate sum. Let us now split the proof into these two cases.

\paragraph{Case 1, Gaussian observations} 
The conditional $\chi^2$ divergence in equation~\eqref{eq:cond-chi} is known~\cite{IngSus12, WeiGunWai17} to be equal to 
\begin{align}
\frac 1{s^2} \sum_{k, \ell = 1}^s \exp \langle \mu^k - \mu^0, \mu^\ell - \mu^0 \rangle - 1 
&=  \frac 1{s^2} \sum_{k, \ell = 1}^s  \exp \sum_{i=1}^d \kappa_i (\theta^k_i - 1/2) (\theta^\ell_i - 1/2) - 1 \notag \\
& = \frac 1{s^2} \sum_{k=1}^s \exp \Big( \delta^2 \sum_{i \in I_k} \kappa_i \Big) + \frac 1{s^2} \sum_{k \ne \ell} \exp(0) - 1 \notag \\
& = \frac 1{s^2} \sum_{k=1}^s \exp \Big( \delta^2 \sum_{i \in I_k} \kappa_i \Big)  - \frac 1s.  \label{eq:ps}
\end{align}
Since $\kappa$ is a Poisson random variable, we may explicitly evaluate its moment generating function to obtain 
\begin{align*}
\chi^2( \Gdist ( w ) \,\|\, \Gdist ( \theta^0 ) ) &= \frac 1{s^2} \sum_{i=1}^s \exp \Big( \lambda r (e^{\delta^2} - 1) \Big) - \frac 1s \\
&\le \frac 1s \exp ( \sqrt{2} \lambda \delta^2 r) - \frac 1s \\
&\le \frac{2 \lambda \delta^2 r}{s} , 
\end{align*}
where we have used twice the fact that $e^x - 1 \le \sqrt{2} x$ for $x \le 0.65$, and the assumptions $\delta^2 \le 0.65$ and $\lambda \delta^2 r \le 2/5$. 

\vspace{2mm}

\paragraph{Case 2, Bernoulli observations}

In the Bernoulli case, 
let us study one summand in equation~\eqref{eq:cond-chi}. 
For a vector $y \in \real^m$, we use the shorthand $y^k$ to denote the $k$-th block of $y$ of length $m_k = \sum_{i \in I_k} \kappa_i$. 
If $k \ne \ell$, then we have
\begin{align*}
\int \frac{ p^k p^{\ell} }{ p^0 } 
&= \sum_{x \in \{0, 1\}^m} \frac{ p^k (x) p^{\ell} (x) }{ p^0 (x) } \\
&= 2^m \sum_{x \in \{0, 1\}^m} \Big( \frac 12 + \delta \Big)^{\|x^k\|_1 + \|x^\ell\|_1} \Big( \frac 12 - \delta \Big)^{m_k - \|x^k\|_1 + m_\ell - \|x^\ell\|_1} \Big( \frac 12 \Big)^{2m - m_k - m_\ell} \\ 
&= \sum_{ \substack{x^k \in \{0, 1\}^{m_k} \\ x^\ell \in \{0, 1\}^{m_\ell}}} \Big( \frac 12 + \delta \Big)^{\|x^k\|_1 + \|x^\ell\|_1} \Big( \frac 12 - \delta \Big)^{m_k - \|x^k\|_1 + m_\ell - \|x^\ell\|_1} \\ 
&\stackrel{\1}{=} \Big( \frac 12 + \delta + \frac 12 - \delta \Big)^{m_k + m_\ell } = 1,
\end{align*}
where step $\1$ uses the binomial expansion.
If $k = \ell$, reasoning similarly yields the sequence of bounds 
\begin{align*}
\int \frac{ p^k p^{\ell} }{ p^0 } 
&= 2^m \sum_{x \in \{0, 1\}^m} \Big( \frac 12 + \delta \Big)^{2 \|x^k\|_1} \Big( \frac 12 - \delta \Big)^{2m_k - 2\|x^k\|_1} \Big( \frac 12 \Big)^{2m - 2m_k} \\ 
&= 2^{m_k} \sum_{x^k \in \{0, 1\}^{m_k}} \Big( \frac 12 + \delta \Big)^{2 \|x^k\|_1} \Big( \frac 12 - \delta \Big)^{2m_k - 2\|x^k\|_1}   \\ 
&= 2^{m_k} \Big[ \Big( \frac 12 + \delta \Big)^2 + \Big( \frac 12 - \delta \Big)^2 \Big]^{m_k}   
= ( 1 + 4 \delta^2 )^{m_k}.
\end{align*}

Combining the pieces, we see that 
\begin{align*}
\frac 1{s^2} \sum_{k, \ell = 1}^s \int \frac{ p^k p^{\ell} }{ p^0 } - 1 
&= \frac 1{s^2} \sum_{k=1}^s ( 1 + 4 \delta^2 )^{ \sum_{i \in I_k} \kappa_i } - \frac 1s \\
&\leq \frac 1{s^2} \sum_{k=1}^s \exp \Big( 4 \delta^2 \sum_{i \in I_k} \kappa_i \Big)  - \frac 1s.
\end{align*}
This inequality exactly resembles inequality~\eqref{eq:ps} (up to the constant factor $4$), and so proceeding as before but changing constant factors appropriately completes the proof.
\qed

\subsection{Proof of Proposition~\ref{prop:meta}}

Recall the definition of $\Mhat(\pihat, \sigmahat)$ in the
meta-algorithm, and additionally, define the projection of any matrix
$M \in \mathbb{R}^{\dimone \times \dimtwo}$ onto $\Cbiso(\pi,
\sigma)$ as
\begin{align*}
\mathcal{P}_{\pi, \sigma}(M) = \arg \min_{\Mtilde \in \Cbiso(\pi,
\sigma)} \| M - \Mtilde \|_F^2.
\end{align*}
Letting $W = Y^{(2)} - M^*$, we have
\begin{align}
\| \Mhat(\pihat, \sigmahat) - M^* \|_F^2 & \stackrel{\1}{\leq} 2 \|
\mathcal{P}_{\pihat, \sigmahat} (M^* + W) - \mathcal{P}_{\pihat,
\sigmahat} (M^*(\pihat, \sigmahat) + W) \|_F^2 \notag \\
& \qquad \quad \ + 2 \| \mathcal{P}_{\pihat, \sigmahat} (M^*(\pihat,
\sigmahat) + W) - M^* \|_F^2 \notag \\
& \stackrel{\2}{\leq} 2 \| M^*(\pihat, \sigmahat) - M^* \|_F^2 + 2 \|
\mathcal{P}_{\pihat, \sigmahat} (M^*(\pihat, \sigmahat) + W) - M^*
\|_F^2 \notag \\
&\stackrel{\3}{\leq} 4 \| \mathcal{P}_{\pihat, \sigmahat} (M^*(\pihat,
\sigmahat) + W) - M^*(\pihat, \sigmahat) \|_F^2 + 6 \| M^*(\pihat,
\sigmahat) - M^* \|_F^2, \label{eq:estapprox}
\end{align}
where step $\2$ follows from the non-expansiveness of a projection onto a convex set, and steps $\1$ and $\3$ from the triangle inequality.

The first term in equation~\eqref{eq:estapprox} is the estimation error of a
bivariate isotonic matrix with known permutations. 
Therefore, applying Corollary~\ref{cor:biso} with a union bound over all permutations $\pihat \in \mathfrak{S}_{\dimone}$ and $\sigmahat \in \mathfrak{S}_{\dimtwo}$ yields the bound
\begin{align*}
\| \mathcal{P}_{\pihat, \sigmahat} (M^*(\pihat,
\sigmahat) + W) - M^*(\pihat, \sigmahat) \|_F^2 \leq C \vartheta(N, \dimone \vee \dimtwo, \dimone \land \dimtwo, \vars)
\end{align*}
on the estimation error with probability at least\footnote{Choosing the constant $C$ to be twice the constant in Corollary~\ref{cor:biso}, we can boost the probability of the good event in Corollary~\ref{cor:biso} to $1 - \dimone^{-2\dimone}$, and applying a union bound over at most $2\dimone!$ permutations yields the claimed result.} $1 - \dimone^{-\dimone}$.

The approximation error can be split into two components: one along
the rows of the matrix, and the other along the columns.  More
explicitly, we have
\begin{align*}
\| M^* - M^*(\pihat, \sigmahat)\|_F^2 &\leq 2\| M^* - M^*(\pihat,
\id)\|_F^2 + 2\| M^*(\pihat, \id) - M^*(\pihat, \sigmahat)\|_F^2 \\ &=
2 \| M^* - M^*(\pihat, \id) \|_F^2 + 2 \| M^* - M^* (\id, \sigmahat)
\|_F^2.
\end{align*}
This completes the proof of the proposition. \qed


\subsection{Proof of Theorem~\ref{thm:ordered-col}}


In order to ease the notation, we adopt the shorthand
\begin{align*}
\eta \defn 16 \varplusone \Big( \sqrt{\frac{ \dimone \dimtwo^2}{N}
\log(\dimone \dimtwo) } + \frac{\dimone \dimtwo}{N} \log(\dimone
\dimtwo) \Big) ,
\end{align*}
and for each block $B \in \blref$ in Algorithm~2, we use the shorthand
\begin{align}
\label{eq:eta-b}
\eta_B \defn 16 \varplusone \Big( \sqrt{\frac{ |B| \dimone \dimtwo}{N}
\log(\dimone \dimtwo) } + \frac{\dimone \dimtwo}{N} \log(\dimone
\dimtwo) \Big)
\end{align}
throughout the proof.  Applying Lemma~\ref{lem:par-sum} with
$\mathcal{S} = \{i\} \times [\dimtwo]$ and then with $\mathcal{S} =
\{i\} \times B$ for each $i \in [\dimone]$ and $B \in \blref$, we
obtain
\begin{subequations}
\begin{align} \label{eq:uni-bd-1}
\Pr \left\{ \Big| S(i) - \sum_{\ell \in [\dimtwo]} M^*_{i,\ell} \Big|
\ge \frac{\eta}{2} \right\} \le 2(\dimone \dimtwo)^{-4} ,
\end{align}
and
\begin{align} \label{eq:uni-bd-2}
\Pr \left\{ \Big| S_{B}(i) - \sum_{\ell \in B} M^*_{i,\ell} \Big| \ge
\frac{\eta_B}{2} \right\} \le 2(\dimone \dimtwo)^{-4} .
\end{align}
\end{subequations}
A union bound over all rows and blocks yields that $\Pr \{\cE\} \ge 1
- 2(\dimone \dimtwo)^{-3}$, where we define the event
\begin{align*}
\cE \defn \left\{ \Big| S(i) - \sum_{\ell \in [\dimtwo]} M^*_{i,\ell}
\Big| \le \frac{\eta}{2} \text{ and } \Big| S_{B}(i) - \sum_{\ell \in
B} M^*_{i,\ell} \Big| \le \frac{\eta_B}{2} \text{ for all } i \in
  [\dimone], B \in \blref \right\} .
\end{align*}

We now condition on event $\cE$. Applying the triangle inequality, if
\begin{align*}
S(v) - S(u) > \eta \quad \text{ or } \quad
S_{B}(v) - S_{B}(u) > \eta_B,
\end{align*}
then we have
\begin{align*}
\sum_{\ell \in [\dimtwo]} M^*_{v,\ell} - \sum_{\ell \in [\dimtwo]}
M^*_{u,\ell} > 0 \quad \text{ or } \quad \sum_{\ell \in B}
M^*_{v,\ell} - \sum_{\ell \in B} M^*_{u,\ell} > 0 .
\end{align*}
It follows that $u < v$ since $M^*$ has non-decreasing columns. Thus,
by the choice of thresholds $\eta$ and $\eta_B$ in the algorithm, we
have guaranteed that every edge $u \to v$ in the graph $G$ is
consistent with the underlying permutation $\id$, so a topological
sort exists on event $\cE$.

We need the following lemma, whose proof is deferred to Section~\ref{sec:block-sum}. 

\begin{lemma}
\label{lem:block-sum}
Let $B$ be a subset of $[\dimtwo]$ and let $\eta_B$ be defined by \eqref{eq:eta-b}. 
Suppose that $\pihat$ is a topological sort of a graph $G$, where an edge $u \to v$ is present whenever 
$$
\sum_{\ell \in B}
M^*_{v,\ell} - \sum_{\ell \in B} M^*_{u,\ell} > 2 \eta_B . 
$$
Then for all $i \in [\dimone]$, we have 
$$
\sum_{j \in B}  \left| M^*_{\pihat (i), j} - M^*_{i,j} \right| \le 96 \varplusone \sqrt{\frac{\dimone
\dimtwo}{N} |B| \log (\dimone \dimtwo)} . 
$$
\end{lemma}

If we have
\begin{align*}
\sum_{\ell \in [\dimtwo]} M^*_{v,\ell} - \sum_{\ell \in [\dimtwo]}
M^*_{u,\ell} > 2 \eta \quad \text{ or } \quad \sum_{\ell \in B}
M^*_{v,\ell} - \sum_{\ell \in B} M^*_{u,\ell} > 2 \eta_B,
\end{align*}
then the triangle inequality implies that
\begin{align*}
S(v) - S(u) > \eta \quad \text{ or } \quad S_{B}(v) - S_{B}(u) >
\eta_B .
\end{align*}
Hence, the edge $u \to v$ is present in the graph $G$.  
As we defined $\pihatref$ as a topological sort of $G$, Lemma~\ref{lem:block-sum} implies that 
\begin{subequations}
\begin{align}
\sum_{j \in [\dimtwo]} \left| M^*_{\pihatref(i), j} - M^*_{i,j} \right|
&\le 96 \varplusone \sqrt{\frac{\dimone \dimtwo^2}{N} \log(\dimone
\dimtwo)} \quad \text{ for all } i \in [\dimone], \text{
and} \label{eq:row-bd-1} \\ 
\sum_{j \in B}  \left| M^*_{\pihatref(i), j} - M^*_{i,j} \right| &\le 96 \varplusone \sqrt{\frac{\dimone
\dimtwo}{N} |B| \log (\dimone \dimtwo)} \quad \text{ for all } i
\in [\dimone], B \in \blref. \label{eq:row-bd-2}
\end{align}
\end{subequations}

%
%
%
We are now tasked with bounding the max-row-norm error. Critical to our analysis is the following lemma:

\begin{lemma} \label{lem:l2-tv-l1}
For a vector $v \in \mathbb{R}^n$, define its variation as $\var(v) =
\max_i v_i - \min_i v_i$. Then we have
\begin{align*}
\|v\|_2^2 \le \var(v) \|v\|_1 + \|v\|_1^2/n.
\end{align*}
\end{lemma}
\noindent See Section~\ref{sec:l2-tv-l1} for the proof of this lemma.

For each $i \in [\dimone]$, define $\Delta^i$ to be the $i$-th row
difference $M^*_{\pihatref(i)} - M^*_{i}$, and for each block $B \in
\blref$, denote the restriction of $\Delta^i$ to $B$ by
$\Delta^i_B$. Lemma~\ref{lem:l2-tv-l1} applied with $v = \Delta^i_B$
yields
\begin{align}
\|\Delta^i\|_2^2 &= \sum_{B \in \blref} \| \Delta^i_B \|_2^2 \notag
\\ &\leq \sum_{B \in \blref} \var(\Delta^i_B) \|\Delta^i_B \|_1 +
\sum_{B \in \blref} \frac{\|\Delta^i_B\|_1^2}{|B|} \notag \\ &\le
\left( \max_{B \in \blref} \|\Delta^i_B \|_1 \right) \left( \sum_{B
\in \blref} \var\left(\Delta^i_B \right) \right) + \frac{\max_{B \in
\blref} \|\Delta^i_B\|_1}{\min_{B \in \blref} |B|} \sum_{B \in
\blref}\|\Delta^i_B\|_1 . \label{eq:row-wise-bd}
\end{align}
We now analyze the quantities in inequality~\eqref{eq:row-wise-bd}.
By the definition of the blocking $\BL$, we have 
\begin{align*}
\frac{1}{2} \blocksize \leq |B| \leq \blocksize .
\end{align*} 
Additionally, the bounds \eqref{eq:row-bd-1} and \eqref{eq:row-bd-2} imply that
\begin{align*}
\sum_{B \in \blref} \|\Delta^i_B\|_1 &= \|\Delta^i \|_1 \le 96
\varplusone \dimtwo \sqrt{\frac{\dimone}{N} \log (\dimone \dimtwo)} ,
\ \text{ and } \\ \|\Delta^i_B \|_1 &\leq 96 \varplusone \dimtwo \Big(
\frac{\dimone}{N} \log (\dimone \dimtwo) \Big)^{3/4} \ \text{ for all
} B \in \blref.
\end{align*}
Moreover, we have
\begin{align*}
\sum_{B \in \blref} \var\left(\Delta^i_B \right) & \le \sum_{B \in
\blref} \Big[ \var(M^*_{i, B} ) + \var(M^*_{\pihatref(i), B} ) \Big]
\\ & \le \var(M^*_i) + \var(M^*_{\pihatref(i)}) \le 2 ,
\end{align*}
because $M^*$ has monotone rows in $[0,1]^{\dimtwo}$.  Finally,
plugging all the pieces into equation~\eqref{eq:row-wise-bd} yields
\begin{align*}
\|\Delta^i\|_2^2 \lesssim (\vars \lor 1) \dimtwo \left( \frac{\dimone
\log \dimone }{N} \right)^{3/4} .
\end{align*}
Normalizing this bound by $1/\dimtwo$ establishes the bound~\eqref{eq:r-upper}.

Summing the max-row-norm bounds over the rows and applying
Proposition~\ref{prop:meta}, we obtain the bound~\eqref{eq:r-upper2} in the Frobenius error.
\qed


\subsection{Proof of Theorem~\ref{thm:fast-tds}}

As mentioned before, an equivalent max-row-norm bound to
equation~\eqref{eq:rowtds} was already obtained in previous
work~\cite{ShaBalWai16-2, ChaMuk16, PanMaoMutWaiCou17}. In fact, the
blocking procedure and comparisons of partial row sums in our
algorithm are irrelevant to achieving the
bound~\eqref{eq:rowtds}. Concentration of whole row sums $S(i)$
suffices to yield the desired rate.

The proof follows immediately from that of
Theorem~\ref{thm:ordered-col}. Using the same argument as in that
proof leading to the bound~\eqref{eq:row-bd-1}, we obtain
\begin{align*}
\Big| \sum_{j \in [\dimtwo]} (M^*_{\pihattds(i), j} - M^*_{i,j}) \Big|
\le 96 \varplusone \sqrt{\frac{\dimone \dimtwo^2}{N} \log(\dimone
\dimtwo)} \quad \text{ for all } i \in [\dimone] .
\end{align*}
Normalizing this bound by $1/\dimtwo$ establishes the bound~\eqref{eq:rowtds} corresponding to the row-approximation error.

For the rest of this section, we focus on showing the Frobenius error
bound~\eqref{eq:frobtds}.  The beginning of the proof proceeds in the
same way as the proof of Theorem~\ref{thm:ordered-col}, so that we
provide only a sketch.  We apply Lemma~\ref{lem:par-sum} with
$\mathcal{S} = \{i\} \times [\dimtwo]$ and $\mathcal{S} = \{i\} \times
\BL_k$ for each tuple $i \in [\dimone], k \in [K]$, and use the fact
that $K \leq \dimtwo / \beta \leq \dimtwo^{1/2}$, to obtain that with
probability at least $1-2(\dimone \dimtwo)^{-3}$, all the full row
sums of $Y^{(2)}$ and all the partial row sums over the column blocks
concentrate well around their means. By virtue of the conditions
\eqref{eq:full-sum} and \eqref{eq:block-sum}, we see that every edge
$u \to v$ in the graph $G$ is consistent with the underlying
permutation so that a topological sort exists with probability at
least $1-2(\dimone \dimtwo)^{-3}$. Additionally, it follows from
Lemma~\ref{lem:block-sum} and the same argument leading to
equations~\eqref{eq:row-bd-1} and~\eqref{eq:row-bd-2} that for all $i
\in [\dimone]$, we have
\begin{subequations}
\begin{align}
\Big| \sum_{j \in [\dimtwo]} (M^*_{\pihatftds(i), j} - M^*_{i,j})
\Big| &\le 96 \varplusone \sqrt{\frac{\dimone \dimtwo^2}{N}
\log(\dimone \dimtwo)} , \text{ and} \label{eq:per-bd-1} \\
\Big| \sum_{j \in \BL_k} (M^*_{\pihatftds(i), j} - M^*_{i,j}) \Big|
&\le 96 \varplusone \sqrt{\frac{\dimone \dimtwo}{N} |\BL_k| \log
(\dimone \dimtwo)} \quad \text{ for all } k \in
  [K], \label{eq:per-bd-2}
\end{align}
\end{subequations}
with probability at least $1-2(\dimone \dimtwo)^{-3}$.

On the other hand, we apply Lemma~\ref{lem:par-sum} with $\mathcal{S}
= [\dimone] \times \{j\}$ to obtain concentration for the column sums
of $Y^{(1)}$:
\begin{align}
\label{eq:col-con}
\Big| C(j) - \sum_{i =1}^{\dimone} M^*_{i,j} \Big| \le 8 \varplusone
\bigg( \sqrt{\frac{ \dimone^2 \dimtwo}{N} \log(\dimone \dimtwo) } +
\frac{\dimone \dimtwo}{N} \log(\dimone \dimtwo) \bigg)
\end{align}
for all $j \in [\dimtwo]$ with probability at least $1 - 2 (\dimone
\dimtwo)^{-3}$.  We carry out the remainder of the proof conditioned
on the event of probability at least $1-4(\dimone \dimtwo)^{-3}$ that
inequalities~\eqref{eq:per-bd-1}, \eqref{eq:per-bd-2} and~\eqref{eq:col-con} hold.

Having stated the necessary bounds, we now split the remainder of the
proof into two parts for convenience. In order to do so, we first
split the set $\BL$ into two disjoint sets of blocks, depending on
whether a block comes from an originally large block (of size larger
than $\beta = \blocksize$ as in Step~3 of Subroutine~1) or from an
aggregation of small blocks.  More formally, define the sets
\begin{subequations}
\begin{align*}
\bllarge & \defn \{B \in \BL: B \text{ was not obtained via
aggregation}\}, \text{ and} \\
\blsmall & \defn \BL \setminus \bllarge.
\end{align*}
\end{subequations}
For a set of blocks $\mathsf{B}$, define the shorthand $\cup
\mathsf{B} = \bigcup_{B \in \mathsf{B}} B$ for convenience.  We begin
by focusing on the blocks $\bllarge$.


\subsubsection{Error on columns indexed by $\cup \bllarge$}

Recall that when the columns of the matrix are ordered according to
$\sigmahatpre$, the blocks in $\bllarge$ are contiguous and thus have
an intrinsic ordering. We index the blocks according to this ordering
as $B_1, B_2, \ldots, B_\ell$ where $\ell = |\bllarge|$. Now define
the disjoint sets
\begin{align*}
\blone & \defn \{ B_k \in \bllarge: k = 0 \ (\modd 2) \}, \text{ and}
\\
\bltwo & \defn \{ B_k \in \bllarge: k = 1 \ (\modd 2) \}.
\end{align*}
Let $\ell_t = |\blt|$ for each $t = 1,2$.

Recall that each block $B_k$ in $\bllarge$ remains unchanged after
aggregation, and that the threshold we used to block the columns is \\
$\tau = 16 \varplusone \big( \sqrt{\frac{\dimone^2 \dimtwo}{N}
\log(\dimone \dimtwo) } + 2 \frac{\dimone \dimtwo}{N} \log(\dimone
\dimtwo) \big)$. Hence, applying the concentration bound
\eqref{eq:col-con} together with the definition of blocks in Step~2 of
Subroutine~1 yields
\begin{align}
\label{eq:col-sum-bd}
\Big| \sum_{i =1}^{\dimone} M^*_{i,j_1} - \sum_{i =1}^{\dimone}
M^*_{i,j_2} \Big| \le 96 \varplusone \sqrt{\frac{\dimone^2 \dimtwo}{N}
\log(\dimone \dimtwo)} \quad \text{ for all } j_1,j_2 \in B_k ,
\end{align}
where we again used the argument leading to the bounds~\eqref{eq:row-bd-1} and~\eqref{eq:row-bd-2} to combine the two terms. Moreover, since the threshold is twice the
concentration bound, it holds that under the true ordering $\id$,
every index in $B_k$ precedes every index in $B_{k+2}$ for any $k \in
[K-2]$.  By definition, we have thus ensured that the blocks in $\blt$
do not ``mix'' with each other.

The rest of the argument hinges on the following lemma, which is
proved in Section~\ref{sec:per-rc-bd}.

\begin{lemma}
\label{lem:per-rc-bd}
For $m \in \mathbb{Z}_+$, let $J_1 \sqcup \cdots \sqcup J_\ell$ be a
partition of $[m]$ such that each $J_k$ is contiguous and $J_k$
precedes $J_{k+1}$. Let $a_k = \min J_k$, $b_k = \max J_k$ and $m_k =
|J_k|$. Let $A$ be a matrix in $[0,1]^{n\times m}$ with non-decreasing
rows and non-decreasing columns. Suppose that
\begin{align*}
\sum_{i=1}^n (A_{i,b_k} - A_{i,a_k}) \le \taulem \ \text{ for each } k
\in [\ell] \text{ and some } \taulem \ge 0.
\end{align*} 
Additionally, suppose that there are positive reals $\rho, \rho_1,
\rho_2, \ldots, \rho_\ell$, and a permutation $\pi$ such that for any
$i \in [n]$, we have \1 $\sum_{j=1}^{m} |A_{\pi(i),j} - A_{i,j}| \le
\rho$, and \2 $\sum_{j \in J_k} |A_{\pi(i),j} - A_{i,j}| \le \rho_k$
for each $k \in [\ell]$.  Then it holds that
\begin{align*}
\sum_{i=1}^n \sum_{j=1}^m (A_{\pi(i),j} - A_{i,j})^2 \le 2\taulem
\sum_{k=1}^\ell \rho_k + n \rho \max_{k \in [\ell]}
\frac{\rho_k}{m_k}.
\end{align*}
\end{lemma}

We apply the lemma as follows. For $t = 1,2$, let the matrix $M^{(t)}$
be the sub-matrix of $M^*$ restricted to the columns indexed by the
indices in $\cup \blt$. The matrix $M^{(t)}$ has non-decreasing rows
and columns by assumption. We have shown that the blocks in $\blt$ do
not mix with each other, so they are contiguous and correctly ordered
in $M^{(t)}$. Moreover, the inequality assumptions of the lemma
correspond to the bounds~\eqref{eq:col-sum-bd}, \eqref{eq:per-bd-1} and
\eqref{eq:per-bd-2} respectively, by setting $J_1, \dots, J_\ell$ to be the blocks in $\blt$, and with the substitutions $A = M^{(t)}$, $n = \dimone$, $m = |\cup \blt|$, 
%
\begin{align*}
\taulem &= 96 \varplusone \sqrt{\frac{\dimone^2 \dimtwo}{N} \log(\dimone
\dimtwo)}, \\
\rho &= 96 \varplusone \sqrt{\frac{\dimone \dimtwo^2}{N} \log(\dimone
\dimtwo)}, \text{ and} \\
\rho_k &= 96 \varplusone \sqrt{\frac{\dimone
\dimtwo}{N} |J_k| \log(\dimone \dimtwo)}.
\end{align*}
Moreover, we have 
\begin{align}
\taulem \sum_{k=1}^\ell \rho_k &\lesssim \maxvarone
\frac{\dimone^{3/2} \dimtwo}{N} \log (\dimone \dimtwo) \sum_{B \in
\blt} \sqrt{|B|} \notag \\
&\stackrel{\1}{\leq} \maxvarone \frac{\dimone^{3/2} \dimtwo}{N}
\log (\dimone \dimtwo) \sqrt{\sum_{B \in \blt} |B|} \sqrt{\ell_t} \notag \\
&\stackrel{\2}{\leq} \frac{\maxvarone}{\sqrt{\beta}} \frac{\dimone^{3/2} \dimtwo^{2}}{N}  \log (\dimone \dimtwo), \label{eq:term1}
\end{align}
where step $\1$ follows from the Cauchy-Schwarz inequality, and step $\2$ from the fact that $\sum_{B \in \blt} |B| \leq \dimtwo$ and that by assumption of large blocks, we have $\min_{B \in \blt} |B| \ge \beta$ so that $\ell_t \leq \dimtwo / \beta$.

We also have
\begin{align}
n \rho \max_{k \in [\ell]} \frac{\rho_k}{m_k} &= \maxvarone \frac{\dimone^2 \dimtwo^{3/2}}{N}  \log (\dimone \dimtwo) \max_{B \in \blt} \frac{\sqrt{|B|}}{|B|} \notag \\
&\leq \frac{\maxvarone}{\sqrt{\beta}}
\frac{\dimone^2 \dimtwo^{3/2}}{N} \log (\dimone \dimtwo), \label{eq:term2}
\end{align}
where we have again used the fact that $\min_{B \in \blt} |B| \ge \beta$. Putting together the bounds~\eqref{eq:term1} and~\eqref{eq:term2} and applying Lemma~\ref{lem:per-rc-bd} yields
\begin{align}
\sum_{i \in [\dimone]} \sum_{j \in \cup \blt}
(M^*_{\pihatftds(i),j} - M^*_{i,j})^2 &\lesssim \frac{\maxvarone}{\sqrt{\beta}} \left(\dimone \dimtwo
\right)^{3/2} \maxdimonetwo^{1/2} \frac{\log (\dimone \dimtwo)}{N}. \label{eq:bigblockerror}
\end{align}
Substituting
$\beta = \blocksize$ and normalizing by $\dimone \dimtwo$ proves the result for the set of blocks
$\BL^{(t)}$. Summing over $t = 1, 2$ then yields a bound of twice the
size for columns of the matrix indexed by $\cup \bllarge$.


\subsubsection{Error on columns indexed by $\cup \blsmall$}

Next, we bound the approximation error of each row of the matrix with
column indices restricted to the union of all small blocks. Roughly speaking, all small (sub-)blocks are aggregated into those that have size of order $\beta$;
by definition of the blocking, this implies that the ambiguity in the column permutation for the aggregated block only exists within the small sub-blocks, and in that sense, this column permutation can be thought of as ``essentially known". 
Thus, the proof resembles that of Theorem~\ref{thm:ordered-col}: it is sufficient (for the eventual bound we target) to bound the Frobenius error by the max-row-norm error. 
Note that we must also make modifications that account for the fact that the column permutation is only approximately known. 
We split the proof into two cases.

\paragraph{Case 1} Let us first address the easy case where $\blsmall$ contains a single block of size less than
$\frac{\beta}{2} = \frac 12 \blocksize$. Here, we have
\begin{align*}
\sum_{i \in [\dimone]} \sum_{j \in \cup \blsmall} (M^*_{\pihatftds(i),j} - M^*_{i,j})^2  &\stackrel{\1}{\le} \sum_{i \in [\dimone]}  \sum_{j \in \cup \blsmall} \big| M^*_{\pihatftds(i),j} - M^*_{i,j} \big| \\
&\stackrel{\2}{=} \sum_{i \in [\dimone]}  \Big| \sum_{j \in \cup \blsmall} (M^*_{\pihatftds(i),j} - M^*_{i,j}) \Big| \\
&\stackrel{\3}{\le} \sum_{i \in [\dimone]} 96 \varplusone \sqrt{ \frac{\dimone \dimtwo}{2N} \beta \log (\dimone \dimtwo)  } \\
&= 48 \sqrt{2} \varplusone \frac{\dimone^{3/2} \dimtwo \maxdimonetwo^{1/4}}{N^{3/4}} \log^{3/4} (\dimone \dimtwo),
\end{align*}
where step \1 follows from the H\"older's inequality and the fact that $M^* \in [0,1]^{\dimone \times \dimtwo}$, step \2 from the monotonicity of the columns of $M^*$, and step \3 from equation~\eqref{eq:per-bd-2}. We have thus proved the theorem for this case.

\paragraph{Case 2} Let us now consider the case where $\blsmall$ contains multiple blocks. 
For each $i \in [\dimone]$, define $\Delta^i$ to be the restriction of
the $i$-th row difference $M^*_{\pihatftds(i)} - M^*_{i}$ to the union
of blocks $\cup \blsmall$. For each block $B \in \blsmall$, denote the
restriction of $\Delta^i$ to $B$ by
$\Delta^i_B$. Lemma~\ref{lem:l2-tv-l1} applied with $v = \Delta^i$
yields
\begin{align}
\|\Delta^i\|_2^2 &= \sum_{B \in \blsmall} \| \Delta^i_B \|_2^2 \notag \\
&\leq \sum_{B \in \blsmall} \var(\Delta^i_B) \|\Delta^i_B \|_1 + \sum_{B \in \blsmall} \frac{\|\Delta^i_B\|_1^2}{|B|} \notag \\
&\le \left( \max_{B \in \blsmall} \|\Delta^i_B \|_1 \right) \left( \sum_{B \in \blsmall} \var\left(\Delta^i_B \right) \right) +  \frac{\max_{B \in \blsmall} \|\Delta^i_B\|_1}{\min_{B \in \blsmall} |B|} \sum_{B \in \blsmall}\|\Delta^i_B\|_1 . \label{eq:row-l2-bd}
\end{align}
We now analyze the quantities in inequality~\eqref{eq:row-l2-bd}.
By the aggregation step of Subroutine~1, we have $\frac{1}{2} \beta \leq |B| \leq 2 \beta$, where $\beta = \blocksize$. Additionally, the bounds \eqref{eq:per-bd-1} and \eqref{eq:per-bd-2} imply that
\begin{align*}
\sum_{B \in \blsmall} \|\Delta^i_B\|_1 &= \|\Delta^i \|_1 \le 96 \varplusone \sqrt{\frac{\dimone \dimtwo^2}{N} \log (\dimone \dimtwo)} \lesssim \varplusone \beta, \ \text{ and } 
\\
\|\Delta^i_B \|_1 &\leq 96 \varplusone \sqrt{\frac{\dimone \dimtwo}{N} |B| \log (\dimone \dimtwo)} \\
&\leq 96 \sqrt{2} \varplusone \sqrt{\frac{\dimone \dimtwo}{N} \beta \log (\dimone \dimtwo)} \ \text{ for all } B \in \blsmall.
\end{align*}

Moreover, to bound the quantity $\sum_{B \in \blsmall} \var\left(\Delta^i_B \right)$, we proceed as in the proof for the large blocks in $\bllarge$. Recall that if we permute the columns by $\sigmahatpre$ according to the column sums, then the blocks in $\blsmall$ have an intrinsic ordering, even after adjacent small blocks are aggregated. Let us index the blocks in $\blsmall$ by $B_1, B_2, \ldots, B_m$ according to this ordering, where $m = |\blsmall|$. As before, the odd-indexed (or even-indexed) blocks do not mix with each other under the true ordering $\id$, because the threshold used to define the blocks is larger than twice the column sum perturbation.
We thus have
\begin{align*}
\sum_{B \in \blsmall} \var\left(\Delta^i_B \right) & = \sum_{\substack{k \in [m] \\ k \text{ odd}}} \var(\Delta^i_{B_k} ) + \sum_{\substack{k \in [m] \\ k \text{ even}}} \var(\Delta^i_{B_k} ) \\
&\le \sum_{\substack{k \in [m] \\ k \text{ odd}}} \big[ \var(M^*_{i, B_k} ) + \var(M^*_{\pihatftds(i), B_k} ) \big] \\
&\qquad \qquad \qquad + \sum_{\substack{k \in [m] \\ k \text{ even}}} \big[ \var(M^*_{i, B_k} ) + \var(M^*_{\pihatftds(i), B_k} ) \big] \\
& \stackrel{\1}{\le} 2\var(M^*_i) + 2\var(M^*_{\pihatftds(i)}) \stackrel{\2}{\le} 4 ,
\end{align*}
where inequality \1 holds because the odd (or even) blocks do not mix, and inequality \2 holds because $M^*$ has monotone rows in $[0,1]^{\dimtwo}$. 

Finally, putting together all the pieces, we can substitute for $\beta$, sum over the indices $i \in \dimone$, and normalize by $\dimone \dimtwo$ to obtain
\begin{align}
\frac{1}{\dimone \dimtwo} \sum_{i \in [\dimone]} \|\Delta^i\|_2^2 \lesssim \maxvarone \left( \frac{\dimone \log (\dimone \dimtwo)}{N} \right)^{3/4}, \label{eq:smallblockerror}
\end{align}
and so the error on columns indexed by the set $\cup \blsmall$ is bounded as desired.


\medskip

Combining the bounds~\eqref{eq:bigblockerror}
and~\eqref{eq:smallblockerror}, we conclude that
\begin{align*}
\frac{1}{\dimone \dimtwo} \|M^*(\pihatftds, \id) - M^*\|_F^2 \lesssim \maxvarone \dimone^{1/4} \maxdimonetwo^{1/2} \left( \frac{\log (\dimone \dimtwo)}{N} \right)^{3/4} 
\end{align*}
with probability at least $1 - 4(\dimone \dimtwo)^{-3}$.
The same proof works with the roles of $\dimone$ and $\dimtwo$ switched and all the matrices transposed, so we have
\begin{align*}
\frac{1}{\dimone \dimtwo} \|M^*(\id, \sigmahatftds) - M^*\|_F^2 \lesssim \maxvarone \dimtwo^{1/4} \maxdimonetwo^{1/2} \left( \frac{\log (\dimone \dimtwo)}{N} \right)^{3/4}
\end{align*}
with the same probability. Consequently,
\begin{align*}
\frac 1{\dimone \dimtwo} \left( \|M^*(\pihatftds, \id) - M^*\|_F^2 + \|M^*(\id, \sigmahatftds) - M^*\|_F^2 \right) \lesssim \maxvarone \Big(\frac{ \dimone \log \dimone}{N} \Big)^{3/4} 
\end{align*}
with probability at least $1 - 8(\dimone \dimtwo)^{-3}$, where we have
used the relation $\dimone \geq \dimtwo$. Applying
Proposition~\ref{prop:meta} completes the proof.
\qed


\subsection{Proof of lemmas for Theorems~\ref{thm:ordered-col} and~\ref{thm:fast-tds}}

We now proceed to the technical lemmas used to prove Theorems~\ref{thm:ordered-col} and~\ref{thm:fast-tds}.

\subsubsection{Proof of Lemma~\ref{lem:block-sum}}
\label{sec:block-sum}

Let us start with a simple lemma regarding permutations. 

\begin{lemma}
\label{lem:per-num}
Let $\{a_i\}_{i=1}^{n}$ be a non-decreasing sequence of real
numbers. If $\pi$ is a permutation in $\symgp_{n}$ such that $\pi(i) <
\pi(j)$ whenever $a_j - a_i > \tau$ where $\tau > 0$, then
$|a_{\pi(i)} - a_i| \le \tau$ for all $i \in [n]$.
\end{lemma}

\paragraph{Proof}
Suppose that $a_j - a_{\pi(j)} > \tau$ for some index $j \in
[n]$. Since $\pi$ is a bijection, there must exist an index $i \le
\pi(j)$ such that $\pi(i) > \pi(j)$. However, we then have $a_j - a_i
\ge a_j - a_{\pi(j)} > \tau$, which contradicts the assumption. A
similar argument shows that $a_{\pi(j)} - a_j > \tau$ also leads to a
contradiction. Therefore, we obtain that $|a_{\pi(j)} - a_j| \le \tau$
for every $j \in [n]$. 
\qed

\smallskip

We split the proof of Lemma~\ref{lem:block-sum} into two cases.
\paragraph{Case 1} First, suppose that $|B| \ge \frac{\dimone \dimtwo}{N}
\log(\dimone \dimtwo)$. In view of the condition 
\begin{align*}
\sum_{\ell \in B}
M^*_{v,\ell} - \sum_{\ell \in B} M^*_{u,\ell} > 2 \eta_B 
\end{align*} 
and the definition of a topological sort, Lemma~\ref{lem:per-num} with $a_i =
\sum_{\ell \in B} M^*_{i,\ell}$, $\pi = \pihat$ and $\tau=2\eta_B$ yields
\begin{align*}
\Big| \sum_{\ell \in B} (M^*_{\pihat(i),\ell} - M^*_{i,\ell} )
\Big| &\le 2 \eta_B \le 96 \varplusone \sqrt{\frac{\dimone \dimtwo}{N}
|B| \log(\dimone \dimtwo)},
\end{align*}
for all $i \in [\dimone]$.
\paragraph{Case 2} Otherwise,
we have $|B| \le \frac{\dimone \dimtwo}{N} \log(\dimone \dimtwo)$. It
then follows that
\begin{align*}
\Big| \sum_{\ell  \in B} (M^*_{\pihat(i),\ell} - M^*_{i,\ell} )
\Big| \le 2 |B| \le 2 \sqrt{\frac{\dimone \dimtwo}{N} |B|
\log(\dimone \dimtwo)},
\end{align*}
where we have used the fact that $M \in [0,1]^{\dimone \times
\dimtwo}$.

Since the columns of $M^*$ are all non-decreasing, we have
\begin{align*}
\sum_{j \in B} \left| M^*_{\pihat(i), j} - M^*_{i,j} \right| = \Big| \sum_{j \in B} (M^*_{\pihat(i), j} - M^*_{i,j}) \Big| ,
\end{align*}
so the proof is complete. 
\qed

\subsubsection{Proof of Lemma~\ref{lem:l2-tv-l1}}
\label{sec:l2-tv-l1}

Let $a = \min_{i\in [n]} v_i$ and $b = \max_{i\in [n]} v_i = a+
\var(v)$.  Since the quantities in the inequality remain the same if
we replace $v$ by $-v$, we assume without loss of generality that $b
\ge 0$. If $a \le 0$, then $\|v\|_\infty \le b-a = \var(v)$. If $a >
0$, then $a \le \|v\|_1/n$ and $\|v\|_\infty = b \le \|v\|_1/n +
\var(v)$. Hence, in any case we have $\|v\|_2^2 \le \|v\|_\infty
\|v\|_1 \le [\|v\|_1/n + \var(v)] \|v\|_1$.
\qed

\subsubsection{Proof of Lemma~\ref{lem:per-rc-bd}}
\label{sec:per-rc-bd}

Since $A$ has increasing rows, for any $i,i_2 \in [n]$ with $i \le
i_2$ and any $j, j_2 \in J_k$, we have
\begin{align*}
A_{i_2,j} - A_{i,j} &= (A_{i_2,j} - A_{i_2,a_k}) + (A_{i_2,a_k} - A_{i,b_k}) + (A_{i,b_k} - A_{i,j}) \\
&\le (A_{i_2,b_k} - A_{i_2,a_k}) + (A_{i_2,j_2} - A_{i,j_2}) + (A_{i,b_k} - A_{i,a_k}) .
\end{align*}
Choosing $j_2 = \arg \min_{r \in J_k} (A_{i_2,r} - A_{i,r})$, we obtain
\begin{align*}
A_{i_2,j} - A_{i,j} \le (A_{i_2,b_k} - A_{i_2,a_k}) + (A_{i,b_k} - A_{i,a_k}) + \frac 1{m_k} \sum_{r \in J_k} (A_{i_2,r} - A_{i,r}) .
\end{align*}
Together with the assumption on $\pi$, this implies that
\begin{align*}
|A_{\pi(i),j} - A_{i,j}| &\le \underbrace{A_{\pi(i),b_k} - A_{\pi(i),a_k}}_{x_{i,k}} + \underbrace{A_{i,b_k} - A_{i,a_k}}_{y_{i,k}} + \frac 1{m_k} \underbrace{\sum_{r \in J_k} |A_{\pi(i),r} - A_{i,r}|}_{z_{i,k}}.
\end{align*}
Hence, it follows that
\begin{align*}
\sum_{i=1}^n \sum_{j=1}^m (A_{i,j} - A_{\pi(i),j})^2 
&= \sum_{i=1}^n \sum_{k=1}^\ell \sum_{j \in J_k} (A_{i,j} - A_{\pi(i),j})^2  \\
&\le \sum_{i=1}^n \sum_{k=1}^\ell \sum_{j \in J_k} |A_{i,j} - A_{\pi(i),j}| ( x_{i,k} + y_{i,k} + z_{i,k}/m_k ) \\
&= \sum_{i=1}^n \sum_{k=1}^\ell z_{i,k} ( x_{i,k} + y_{i,k} + z_{i,k}/m_k ) .
\end{align*}
According to the assumptions, we have
\begin{enumerate}
\item $\sum_{k=1}^\ell x_{i,k} \le 1$ and $\sum_{i=1}^n x_{i,k} \le \taulem$ for any $i \in [n], k \in [\ell]$;
\item $\sum_{k=1}^\ell y_{i,k} \le 1$ and $\sum_{i=1}^n y_{i,k} \le \taulem$ for any $i \in [n], k \in [\ell]$;
\item $z_{i,k} \le \rho_k$ and $\sum_{k=1}^\ell z_{i,k} \le \rho$ for any $i \in [n], k \in [\ell]$.
\end{enumerate}
Consequently, the following bounds hold:
\begin{enumerate}
\item $\sum_{i=1}^n \sum_{k=1}^\ell z_{i,k} x_{i,k} \le \sum_{i=1}^n \sum_{k=1}^\ell \rho_k x_{i,k} \le \taulem \sum_{k=1}^\ell \rho_k$;
\item $\sum_{i=1}^n \sum_{k=1}^\ell z_{i,k} y_{i,k} \le \sum_{i=1}^n \sum_{k=1}^\ell \rho_k y_{i,k} \le \taulem \sum_{k=1}^\ell \rho_k$;
\item $\sum_{i=1}^n \sum_{k=1}^\ell z_{i,k}^2/m_k \le \sum_{i=1}^n \sum_{k = 1}^\ell z_{i,k} \cdot \max_{k \in [\ell]} (\rho_k/m_k) \le n \rho \max_{k \in [\ell]} (\rho_k/m_k)$.
\end{enumerate}
Combining these inequalities yields the claim.
\qed

\section{Poissonization reduction} \label{app:poi}

In this appendix, we show that Poissonization only affects the rates of estimation up to a constant factor. Note that we may assume that
$N \ge 4 \log(\dimone \dimtwo)$, since otherwise, all the bounds in
the theorems hold trivially.

Let us first show that an estimator designed for a Poisson number of samples may be employed for estimation with a fixed number of samples. Assume that $N$ is fixed, and we have an estimator $\Mhat_{\Poi}(N)$, which is designed under $N' = \Poi(N)$ observations $\{y_{\ell}\}_{\ell = 1}^{N'}$. 
Now, given exactly $N$ observations $\{y_{\ell} \}_{\ell = 1}^{N}$ from the model~\eqref{eq:model}, choose an integer $\widetilde{N} = \Poi(N/2)$, and output the estimator
\begin{align*}
\Mhat(N) = 
\begin{cases}
\Mhat_{\Poi} (N/2) & \text{ if } \widetilde{N} \leq N , \\
0 & \text{ otherwise.}
\end{cases}
\end{align*}

Recalling the assumption $N \geq 4 \log (\dimone \dimtwo)$, we have
\begin{align*}
\Pr \{\widetilde{N} \geq N \} \leq e^{-N/2} \leq (\dimone \dimtwo)^{-2}.
\end{align*}
Thus, the error of the estimator $\Mhat(N)$, which always uses at most $N$ samples, is bounded by $\frac{1}{\dimone \dimtwo} \| \Mhat_{\Poi}(N/2) - M^* \|_F^2$ with probability greater than $1 - (\dimone \dimtwo)^{-2}$, and moreover, we have
\begin{align*}
\EE \left[ \frac{1}{\dimone \dimtwo} \| \Mhat(N) - M^* \|_F^2 \right] \leq \EE \left[ \frac{1}{\dimone \dimtwo} \| \Mhat_{\Poi}(N/2) - M^* \|_F^2 \right] + (\dimone \dimtwo)^{-2}.
\end{align*}

We now show the reverse, that an estimator $\Mhat(N)$ designed using exactly $N$ samples may be used to estimate $M^*$ under a Poissonized observation model. 
Given $\widetilde{N} = \Poi(2N)$ samples, define the estimator
\begin{align*}
\Mhat_{\Poi}(2N) =
\begin{cases}
\Mhat(N) & \text{ if } \widetilde{N} \geq N, \\
0 & \text{ otherwise,}
\end{cases} 
\end{align*}
where in the former case, $\Mhat(N)$ is computed by discarding $\widetilde{N} - N$ samples at random.

Again, using the fact that $N \geq 4 \log(\dimone \dimtwo)$ yields 
\begin{align*}
\Pr \{ \widetilde{N} \ge N \} \leq e^{-N} \leq (\dimone \dimtwo)^{-4},
\end{align*}
and so once again, the error of the estimator $\Mhat_{\Poi}(2N)$ is bounded by $\frac{1}{\dimone \dimtwo} \| \Mhat(N) - M^* \|_F^2$ with probability greater than $1 - (\dimone \dimtwo)^{-4}$. A similar guarantee also holds in expectation.

\section{Consequences for cone testing}

Our proof of Theorem~\ref{thm:rowlb} has some consequences for compound hypothesis testing on cones, which may be of independent interest.
We begin by setting up the general framework of cone testing and the notion of the minimax testing radius associated with a compound decision problem~\cite{IngSus12}; our exposition is borrowed from Wei et al.~\cite{WeiGunWai17}.

We are interested in compound testing problems in which the null and alternative lie in nested $d$-dimensional convex cones $C_1$ and $C_2$, respectively, where $C_1 \subset C_2$. 
In order to specify such a problem precisely, we need more notation. 
For a given $\epsilon > 0$, we define the $\epsilon$-fattening of the cone $C_1$ as
\[
\mathbb{B}_2(C_1; \epsilon) \defn \{ \theta \in \real: \inf_{x \in C_1} \| \theta - x \|_2 \leq \epsilon \}.
\]
Given a noisy observation $y$ of a vector $\theta$, we are interested in distinguishing the two hypotheses
\begin{align*}
H_0 : \theta \in C_1 \quad \text{ and } \quad H_1 : \theta \in C_2 \setminus \mathbb{B}_2(C_1; \epsilon). 
\end{align*}
The parameter $\epsilon$ serves to titrate the difficulty of the testing problem. Letting $\psi: \real^d \to \{0, 1\}$ be any (measurable)
test function, we measure its performance in terms of the uniform error
\begin{align*}
\mathcal{E}(\psi; C_1, C_2, \epsilon) \defn \sup_{\theta \in C_1} \E_{\theta}[\psi(y)] + \sup_{\theta \in C_2 \setminus \mathbb{B}_2(C_1; \epsilon) } \E_{\theta}[1 - \psi(y)],
\end{align*}
which gives the worst-case error over both the null and alternative.
For a given error level $\rho \in (0, 1)$, the smallest setting of $\epsilon$ for which there exists a test $\psi$ with uniform error bounded by $\rho$ is called the minimax testing radius at level $\rho$. Often, we prove a lower bound on this minimax radius by considering two carefully constructed simple hypotheses, which induce the distributions $\mathbb{P}_0$ and $\mathbb{P}_1$ on the observations given the null and alternative respectively. For any such choice, we have the lower bound
\begin{align*}
\inf_{\psi} \mathcal{E}(\psi; C_1, C_2, \epsilon) \geq 1 - \TV(\mathbb{P}_0 \,\|\, \mathbb{P}_1).
\end{align*}

With this setup in place, let us also define some cones of interest. The first is  the positive orthant in $d$ dimensions, given by $\mathcal{R}^d_{\geq 0}$. The second is the $d$-dimensional monotone cone $\mathcal{M}^d \defn \left\{ \theta \in \real^d: \theta_1 \leq \theta_2 \leq \cdots \leq \theta_d \right\}$; also define the set
\begin{align*}
\mathcal{M}_{\mathsf{diff}} \defn \left\{ u - v: u, v \in \mathcal{M}^d \right\}.
\end{align*}
We are now ready to state our results for cone testing; for simplicity, we specialize to the standard Gaussian noise model with Poissonized observations. Recall our notation $\Gdist(U)$ defined in Section~\ref{sec:pf-rowlb}.
\begin{proposition} \label{prop:orthcone}
Suppose that $1/d \leq \lambda \leq 1$, and let $L = \{ 0 \}$ and $C = \mathcal{R}^d_{\geq 0}$. 
When each entry of the parameter is independently observed $\Poi(\lambda)$ times and each observation is corrupted by independent standard Gaussian noise, we have
\begin{align*}
\inf_{\psi} \; \mathcal{E}(\psi, L, C, \epsilon ) \geq \frac 12 \qquad \text{ whenever } \qquad \epsilon^2 \leq \frac{1}{16} \sqrt{ \frac{d}{\lambda} }.
\end{align*}
\end{proposition}
\begin{proof}
Consider the equivalent of our mixture $\mathbb{M} \defn \mathbb{M}_{\set}(\minradius)$ used to prove Theorem~\ref{thm:rowlb}(a) in $d$ dimensions, and consider the vector $v - u$ for $[u; v] \sim \mathbb{M}$. In Section~\ref{sec:pf-rowlb}, we proved---note that for $1/d \leq \lambda \leq 1$, we have $\minradius^2 = \frac 1{16} \sqrt{d / \lambda}$---the bound 
\begin{align*}
\TV( \Gdist(v - u) \,\|\, \Gdist(0) ) \leq 1/2.
\end{align*}
This immediately yields the claimed lower bound.
\end{proof}
As discussed before, Proposition~\ref{prop:orthcone} is an extension of Proposition 1 of Wei et al.~\cite{WeiGunWai17} to the missing data setting. Our second proposition covers a cone testing problem in which applying the general theory of Wei et al.~\cite{WeiGunWai17} does not lead to sharp bounds even in the fully observed setting.
\begin{proposition} \label{prop:moncone}
Suppose that $1/d \leq \lambda \leq d$, and let $L = \{ 0 \}$ and $C = \mathcal{M}_{\mathsf{diff}} \cap \real^d_{\geq 0}$. There is an absolute constant $c > 0$ such that when each entry of the parameter is independently observed $\Poi(\lambda)$ times and each observation is corrupted by independent standard Gaussian noise, we have
\begin{align*}
\inf_{\psi} \; \mathcal{E}(\psi, L, C , \epsilon) \geq \frac 12 \qquad \text{ whenever } \qquad \epsilon^2 \leq c \frac{ d^{1/4} }{ \lambda^{3/4} } . 
\end{align*}
\end{proposition}
\begin{proof}
The proof of this proposition follows as a corollary of our proof of Theorem~\ref{thm:rowlb}(b), since in proving the theorem, we established a lower bound on the minimax testing radius of this exact problem, with $\minradius^2 = c \, d \left( \lambda d \right)^{-3/4}$ valid for all $1/d \leq \lambda \leq d$.
\end{proof}

\begin{remark}
Note that we have set the probability of error $\rho$ to be equal to $1/2$ in stating the propositions above. Our technique also accommodates a generalization to arbitrary $\rho$ by a simple modification of the scalar $\delta$ in the respective packing constructions.
\end{remark}

\bibliographystyle{alpha}
\bibliography{fast_sst}

\end{document}